\documentclass{article}

    \PassOptionsToPackage{numbers, compress}{natbib}

\usepackage[final]{neurips_2021}

\usepackage[utf8]{inputenc} 
\usepackage[T1]{fontenc}    
\usepackage{hyperref}       
\usepackage{url}            
\usepackage{booktabs}       
\usepackage{amsfonts}       
\usepackage{nicefrac}       
\usepackage{microtype}      
\usepackage{xcolor}         

\usepackage{microtype}
\usepackage{pifont}
\usepackage{graphicx}
\usepackage{tikz}
\usepackage{subfigure}
\usepackage{booktabs} 

\usepackage{wrapfig}
\usepackage{color}
\usepackage{mathtools}

\usepackage{amsmath,amsfonts,bm}

\def\Figref#1{Figure~\ref{#1}}

\def\secref#1{Section~\ref{#1}}

\def\Apxref#1{Appendix~\ref{#1}}

\def\eqref#1{equation~(\ref{#1})}
\def\Eqref#1{Eq.~(\ref{#1})}

\def\Propref#1{Proposition~\ref{#1}}
\def\Defref#1{Definition~\ref{#1}}
\def\Tableref#1{Table~\ref{#1}}
\def\Lemref#1{Lemma~\ref{#1}}

\def\1{\bm{1}}

\DeclareMathAlphabet{\mathsfit}{\encodingdefault}{\sfdefault}{m}{sl}
\SetMathAlphabet{\mathsfit}{bold}{\encodingdefault}{\sfdefault}{bx}{n}

\def\gA{{\mathcal{A}}}
\def\gB{{\mathcal{B}}}

\def\gE{{\mathcal{E}}}

\def\gG{{\mathcal{G}}}
\def\gH{{\mathcal{H}}}

\def\gL{{\mathcal{L}}}

\def\gR{{\mathcal{R}}}
\def\gS{{\mathcal{S}}}
\def\gT{{\mathcal{T}}}

\def\gX{{\mathcal{X}}}

\def\gZ{{\mathcal{Z}}}

\newcommand{\E}{\mathbb{E}}

\newcommand{\R}{\mathbb{R}}

\DeclareMathOperator*{\argmin}{arg\,min}

\usepackage{subfigure}
\usepackage{multirow}
\usepackage{makecell}
\usepackage{array, booktabs, xcolor}
\usepackage{amssymb}
\usepackage{graphbox}
\usepackage{url}
\usepackage{algorithm} 
\usepackage{algorithmic}
\usepackage{amsthm,xparse}
\renewcommand{\algorithmicrequire}{\textbf{Require:}}  
\renewcommand{\algorithmicensure}{\textbf{Output:}}
\renewcommand{\algorithmiccomment}[1]{\unskip\hfill{\(\triangleright\)~#1}\par}

\usepackage{amsthm}
\usepackage{todonotes}

\newtheorem{lemma}{Lemma}

\newtheorem{definition}{Definition}

\newtheorem{prop}{Proposition}

\newtheorem{innercustomgeneric}{\customgenericname}
\providecommand{\customgenericname}{}
\newcommand{\newcustomtheorem}[2]{%
  \newenvironment{#1}[1]
  {%
   \renewcommand\customgenericname{#2}%
   \renewcommand\theinnercustomgeneric{##1}%
   \innercustomgeneric
  }
  {\endinnercustomgeneric}
}

\newcustomtheorem{customthm}{Theorem}
\newcustomtheorem{customlemma}{Lemma}
\newcustomtheorem{customprop}{Proposition}

\usepackage{thmtools}
\usepackage{thm-restate}
\usepackage{hyperref}
\usepackage{cleveref}

\allowdisplaybreaks

\usepackage{color}
\usepackage{mathtools}

\usepackage{blindtext}
\usepackage{subfigure}
\usepackage{multirow}
\usepackage{makecell}
\usepackage{array, booktabs, xcolor}
\usepackage{amssymb}
\usepackage{graphbox}
\usepackage{caption}
\usepackage{enumitem}
\usepackage{longtable}
\allowdisplaybreaks

\usepackage{tikz}
\usepackage{amsfonts}
\usepackage{amsmath,amssymb}
\usepackage{systeme,mathtools}
\usetikzlibrary{positioning,arrows.meta,quotes}
\usetikzlibrary{shapes,snakes}
\usetikzlibrary{bayesnet}
\tikzset{>=latex}
\tikzstyle{plate caption} = [caption, node distance=0, inner sep=0pt,
below left=5pt and 0pt of #1.south]

\newcommand{\mdp}{GBMDP}
    
\title{Learning Domain Invariant Representations in Goal-conditioned Block MDPs}

\author{%
  Beining Han \\
  IIIS, Tsinghua University \\
  \texttt{bouldinghan@gmail.com} \\
  \And
  Chongyi Zheng \\
  Carnegie Mellon University \\
  \texttt{chongyiz@andrew.cmu.edu} \\
  \AND
  Harris Chan ~~~~~~ Keiran Paster ~~~~~~ Michael R. Zhang ~~~~~~ Jimmy Ba \\
  University of Toronto \& Vector Institute \\
    \texttt{\{hchan, keirp, michael, jba\}@cs.toronto.edu} \\

}

\begin{document}

\maketitle

\begin{abstract}
Deep Reinforcement Learning (RL) is successful in solving many complex Markov Decision Processes (MDPs) problems. However, agents often face unanticipated environmental changes after deployment in the real world. These changes are often spurious and unrelated to the underlying problem, such as background shifts for visual input agents. Unfortunately, deep RL policies are usually sensitive to these changes and fail to act robustly against them. This resembles the problem of domain generalization in supervised learning. In this work, we study this problem for goal-conditioned RL agents. We propose a theoretical framework in the Block MDP setting that characterizes the generalizability of goal-conditioned policies to new environments. Under this framework, we develop a practical method \textit{PA-SkewFit} that enhances domain generalization. The empirical evaluation shows that our goal-conditioned RL agent can perform well in various unseen test environments, improving by 50\% over baselines.
\end{abstract}

\section{Introduction} \label{sec:intro}

Deep Reinforcement Learning (RL) has achieved remarkable success in solving high-dimensional Markov Decision Processes (MDPs) problems, e.g., Alpha Zero \cite{alphagozero2017} for Go, DQN \cite{dqn2015} for Atari games and SAC \cite{haarnoja2018sac}  for locomotion control. However, current RL algorithms requires massive amounts of trial and error to learn \cite{alphagozero2017, dqn2015, haarnoja2018sac}. They also tend to overfit to specific environments and often fail to generalize beyond the environment they were trained on \cite{packergao}. Unfortunately, this characteristic limits the applicability of RL algorithms for many real world applications. Deployed RL agents, e.g. robots in the field, will often face environment changes in their input such as different backgrounds, lighting conditions or object shapes \cite{julian2020efficient}. Many of these changes are often spurious and unrelated to the underlying task, e.g. control. However, RL agents trained without experiencing these changes are sensitive to the changes and often perform poorly in practice \cite{julian2020efficient, zhang2020inblock, zhang2020inrepr}.

In our work, we seek to tackle changing, diverse problems with goal-conditioned RL agents. Goal-conditioned Reinforcement Learning is a popular research topic as its formulation and method is practical for many robot learning problems \cite{andry2017her, eysenbach2020clearn}. In goal-conditioned MDPs, the agent has to achieve a desired goal state $g$ which is sampled from a prior distribution. The agent should be able to achieve not only the training goals but also new test-time goals. Moreover, in practice, goal-conditioned RL agents often receive high-dimensional inputs for both observations and goals \cite{paster2020glamour, pere2018unsupervisedgoal}. Thus, it is important to ensure that the behaviour of goal-conditioned RL agents is invariant to any irrelevant environmental changes in the input at test time. 
Previous work \cite{zhang2020inblock} tries to address these problems via model bisimulation metric \cite{ferns2011bisimulation}. 
These methods aim to acquire a minimal representation which is invariant to irrelevant environment factors. However, as goal-conditioned MDPs are a family of MDPs indexed by the goals, it is inefficient for these methods to acquire the model bisimulation representation for every possible goal, especially in high-dimensional continuous goal spaces (such as images).

In our work, we instead choose to optimize a surrogate objective to learn the invariant policy. Our main contributions are:
\begin{enumerate}[leftmargin=0.2in, itemsep=-0.5pt, topsep=-0.5pt, partopsep=-1.5pt]
    \item We formulate the Goal-conditioned Block MDPs (\mdp s) to study domain generalization in the goal-conditioned reinforcement learning setting (Section \ref{sec:problem}), and propose a general theory characterizing how well a policy generalizes to unseen environments (Section \ref{sec:dgtheory}).
    \item We propose a theoretically-motivated algorithm based on optimizing a surrogate objective, \textit{perfect alignment}, with \textit{aligned data} (Section \ref{sec:alignsample}). We then describe a practical implementation based on  Skew-Fit \cite{pong2020skew} to achieve the objective (Section \ref{sec:alignsampleskewfit}).
    \item Empirically, our experiments for a sawyer arm robot simulation with visual observations and goals demonstrates that our proposed method achieves state-of-the-art performance compared to data augmentation and bisimulation baselines at generalizing to unseen test environments in goal-conditioned tasks (Section \ref{sec:experiment}). 
\end{enumerate}


\section{Problem Formulation} \label{sec:problem}
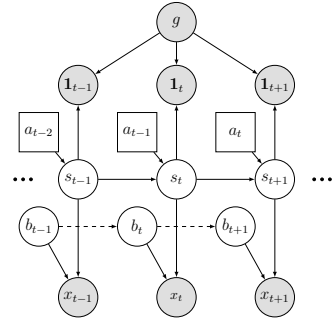
\begin{wrapfigure}{r}{0.3\textwidth} 
    \vspace{-12pt}
    \resizebox{!}{0.3 \textwidth} {
        \begin{tikzpicture}
        \node [circle,draw=black,fill=white,inner sep=0pt,minimum size=1.0cm] (s_{t - 1}) at (0,0) {\large $s_{t - 1}$};
        \node [circle,draw=black,fill=white,inner sep=0pt,minimum size=1.0cm] (b_{t - 1}) at (-1.0,-1.2) {\large $b_{t - 1}$};
        \node [rectangle,draw=black,fill=white,inner sep=0pt,minimum size=1.0cm] (a_{t - 2}) at (-1.0,1.2) {\large $a_{t - 2}$};
        \node [obs,minimum size=1.0cm] (x_{t - 1}) at (0,-3.0) {\large $x_{t - 1}$};
        \node [obs,minimum size=1.0cm] (r_{t - 1}) at (0,2.4) {\large $\mathbf{1}_{t - 1}$};
    
        \node [circle,draw=black,fill=white,inner sep=0pt,minimum size=1.0cm] (s_t) at (2.5,0) {\large $s_t$};
        \node [circle,draw=black,fill=white,inner sep=0pt,minimum size=1.0cm] (b_t) at (1.5,-1.2) {\large $b_t$};
        \node [rectangle,draw=black,fill=white,inner sep=0pt,minimum size=1.0cm] (a_{t - 1}) at (1.5,1.2) {\large $a_{t - 1}$};
        \node [obs,minimum size=1.0cm] (x_t) at (2.5,-3.0) {$x_t$};
        \node [obs,minimum size=1.0cm] (r_t) at (2.5,2.4) {\large $\mathbf{1}_t$};
        \node [obs,minimum size=1.0cm] (g) at (2.5,4) {\large $g$};
    
        \node [circle,draw=black,fill=white,inner sep=0pt,minimum size=1.0cm] (s_{t + 1}) at (5,0) {\large $s_{t + 1}$};
        \node [circle,draw=black,fill=white,inner sep=0pt,minimum size=1.0cm] (b_{t + 1}) at (4,-1.2) {\large $b_{t + 1}$};
        \node [rectangle,draw=black,fill=white,inner sep=0pt,minimum size=1.0cm] (a_t) at (4,1.2) {\large $a_t$};
        \node [obs,minimum size=1.0cm] (x_{t + 1}) at (5,-3.0) {\large $x_{t + 1}$};
        \node [obs,minimum size=1.0cm] (r_{t + 1}) at (5,2.4) {\large $\mathbf{1}_{t+1}$};
        
        \path [draw,->] (a_{t - 2}) edge (s_{t - 1});
        \path [draw,->] (s_{t - 1}) edge (x_{t - 1});
        \path [draw,->] (b_{t - 1}) edge (x_{t - 1});
        \path [draw,->] (s_{t - 1}) edge (r_{t - 1});
        
        \path [draw,->] (a_{t - 1}) edge (s_t);
        \path [draw,->] (s_t) edge (x_t);
        \path [draw,->] (b_t) edge (x_t);
        \path [draw,->] (s_t) edge (r_t);
        
        \path [draw,->] (a_t) edge (s_{t + 1});
        \path [draw,->] (s_{t + 1}) edge (x_{t + 1});
        \path [draw,->] (b_{t + 1}) edge (x_{t + 1});
        \path [draw,->] (s_{t + 1}) edge (r_{t + 1});
        
        \path [draw,->] (s_{t - 1}) edge (s_t);
        \path [dashed,->] (b_{t - 1}) edge (b_t);
        
        \path [draw,->] (s_t) edge (s_{t + 1});
        \path [dashed,->] (b_t) edge (b_{t + 1});
        
        \path [draw,->] (g) edge (r_{t - 1});
        \path [draw,->] (g) edge (r_t);
        \path [draw,->] (g) edge (r_{t + 1});
        
        \node[mark size=1pt,color=black] at (-1.2,0) {\pgfuseplotmark{*}};
        \node[mark size=1pt,color=black] at (-1.4,0) {\pgfuseplotmark{*}};
        \node[mark size=1pt,color=black] at (-1.6,0) {\pgfuseplotmark{*}};
        
        \node[mark size=1pt,color=black] at (6.0,0) {\pgfuseplotmark{*}};
        \node[mark size=1pt,color=black] at (6.2,0) {\pgfuseplotmark{*}};
        \node[mark size=1pt,color=black] at (6.4,0) {\pgfuseplotmark{*}};
        \end{tikzpicture}
    }
    \caption{Graphical model for Goal-conditioned Block MDPs (\mdp s) setting. The agent takes in the goal $g$ and observation $x_t$, which is produced by the domain invariant state $s_t$ and environmental state $b_t$, and acts with action $a_t$. Note that $b_t$ may have temporal dependence indicated by the dashed edge. } 
    \label{fig:gbmdp}
    \vspace{-12pt}
    \vskip -0.2in
\end{wrapfigure}

In this section, we formulate the domain invariant learning problem as solving Goal-conditioned Block MDPs (\mdp s). This extends previous work on learning invariances \cite{zhang2020inblock, du2019block} to the goal-conditioned setting \cite{kaelbling1993learning,schaul2015universal,andry2017her}. 

We consider a family of Goal-conditioned Block MDP environments $M^{\gE} = \{(\gS, \gA, \gX^e, \gT^e, \gG, \gamma)|  e \in \gE\}$ where $e$ stands for the environment index. Each environment consists of shared state space $\gS$, shared action space $\gA$, observation space $\gX^e$, transition dynamic $\gT^e$, shared goal space $\gG \subset \gS$ and the discount factor $\gamma$.\

Moreover, we assume that $M^{\gE}$ follows the generalized Block structure \cite{zhang2020inblock}.
The observation $x^e \in \gX^e$ is determined by state $s \in \gS$ and the environmental factor $b^e \in \gB^e$, i.e., $x^e(s, b^e)$ (\Figref{fig:gbmdp}). For brevity, we use $x^e_t(s)$ to denote the observation for domain $e$ at state $s$ and step $t$. We may also omit $t$ as $x^e(s)$ if we do not emphasize on the step $t$ or the exact environmental factor $b^e_t$. The transition function is thus consists of state transition $p(s_{t+1}|s_t, a_t)$ (also $p(s_0)$), environmental factor transition $q^e(b^e_{t+1}|b^e_t)$. In our work, we assume the state transition is nearly deterministic, i.e., $\forall s, a,$ entropy $\gH(p(s_{t+1}|s_t, a_t)), \gH(p(s_0)) \ll 1$, which is quite common in most RL benchmarks and applications \cite{dqn2015, greg2016openai, vitchyr2018multiworld}. Most importantly, $\gX^{\gE} = \cup_{e \in \gE} \gX^e$ satisfies the disjoint property \cite{du2019block}, i.e., each observation $x \in \gX^{\gE}$ uniquely determines its underlying state $s$. Thus, the observation space $\gX^{\gE}$ can be partitioned into disjoint blocks $\gX(s), s \in \gS$. This assumption prevents the partial observation problem.

The objective function in \mdp{} is to learn a goal-conditioned policy $\pi(a|x^e, g)$ that maximizes the discounted state density function $J(\pi)$ \cite{eysenbach2020clearn} across all domains $e \in \gE$. In our theoretical analysis, we do not assume the exact form of $g$ to the policy. One can regard $\pi(\cdot|x^e, g)$ as a group of RL policies indexed by the goal state $g$.
\begin{align} \label{equ:mdpobj}
    J(\pi) = \E_{e \sim \gE, g \sim \gG, \pi}\left[(1 - \gamma) \sum_{t=0}^{\infty} \gamma^t p_{\pi}^e(s_t = g|g) \right] = \E_{e \sim \gE}[J^e(\pi)]
\end{align}
$p_{\pi}^e(s_t=g|g)$ denotes the probability of achieving goal $g$ under policy $\pi(\cdot|x^e, g)$ at step $t$ in domain $e$. Besides, $e \sim \gE$ and $g \sim \gG$ refers to uniform samples from each set. As $p_{\pi}^e$ is defined over state space, it may differs among environments since policy $\pi$ takes $x^{e}$ as input. Fortunately, in a \mdp, there exist optimal policies $\pi_G(\cdot|x^e, g)$ which are invariant over all environments, i.e., $\pi_G(a|x^e(s), g) = \pi_G(a|x^{e'}(s), g), \forall a \in \gA, s \in \gS, e, e' \in \gE$. 

During training, the agent has access to training environments $\{e_i\}_{i=1}^N = \gE_{\text{train}} \subset \gE$ with their environment indices. However, we do not assume that $\gE_{\text{train}}$ is i.i.d sampled from $\gE$. Thus, we want the goal-conditioned RL agent to acquire the ability to neglect the spurious and unrelated environmental factor $b^e$ and capture the underlying invariant state information. This setup is adopted in many recent works such as in \cite{zhang2020inblock} and in domain generalization \cite{koh2020wilds, arjovsky2019irm} for supervised learning. 

\section{Method} \label{sec:method}

In this section, we propose a novel learning algorithm to solve \mdp s. First, we propose a general theory to characterize how well a policy $\pi$ generalizes to unseen test environments after training on $\gE_{\text{train}}$. Then, we introduce \emph{perfect alignment} as a surrogate objective for learning. This objective is supported by the generalization theory. Finally, we propose a practical method to acquire perfect alignment. 

\subsection{Domain Generalization Theory for \mdp} \label{sec:dgtheory}
In a seminal work, Ben-David et al. \cite{ben2010datheory} shows it is possible to bound the error of a classifier trained on a source domain on a target domain with a different data distribution. Follow-up work extends the theory to the domain generalization setting \cite{sicilia2021dgtheory, albuquerque2019dgtheory}. In \mdp, we can also derive similar theory to characterize the generalization from training environments $\gE_{\text{train}}$ to target test environment $t$. The theory relies on the Total Variation Distance $D_\text{TV}$ \cite{wiki:Total_variation_distance_of_probability_measures} of two policies $\pi_1, \pi_2$ with input $(x^e, g)$, which is defined as follows.
\begin{align*}
    D_\text{TV}(\pi_1(\cdot|x^e, g) \parallel \pi_2(\cdot|x^e, g)) = \sup_{A' \in \sigma(\gA)} |\pi_1(A'|x^e, g) - \pi_2(A'|x^e, g)|
\end{align*}
In the following statements, we denote $\rho(x, g)$ as some joint distributions of goals and observations that $g \sim \gG$ and $x$ is determined by $\rho(x|g)$. Additionally, we use $\rho_{\pi}^e(x^e|g)$ to denote the discounted occupancy measure of $x^e$ in environment $e$ under policy $\pi(\cdot|x^e, g)$ and refer $\rho_{\pi}^e(x^e)$ as the marginal distribution. Furthermore, we denote $\epsilon^{\rho(x,g)}(\pi_1 \parallel \pi_2)$ as the average $D_{\text{TV}}$ between $\pi_1$ and $\pi_2$, i.e., $\epsilon^{\rho(x, g)}(\pi_1 \parallel \pi_2) = \E_{\rho(x, g)}[D_\text{TV}(\pi_1(\cdot|x, g) \parallel \pi_2(\cdot|x, g))]$. This quantity is crucial in our theory as it can characterize the performance gap between two policies (see \Apxref{apx:proof}).

Then, similar to the famous $\gH \Delta \gH$-divergence \cite{ben2010datheory, sicilia2021dgtheory} in domain adaptation theory, we define $\Pi \Delta \Pi$-divergence of two joint distributions $\rho(x, g)$ and $\rho(x, g)'$ in terms of the policy class $\Pi$:
\begin{align*}
    d_{\Pi \Delta \Pi}(\rho(x, g), \rho(x, g)') = \sup_{\pi, \pi' \in \Pi}|\epsilon^{\rho(x, g)}(\pi \parallel \pi') - \epsilon^{\rho(x, g)'}(\pi \parallel \pi')|
\end{align*}
On one hand, $d_{\Pi \Delta \Pi}$ is a distance metric which reflects the distance between two distributions w.r.t function class $\Pi$. On the other hand, if we fix these two distributions, it also reveals the quality of the function class $\Pi$, i.e., smaller $d_{\Pi \Delta \Pi}$ means more invariance to the distribution change. Finally, we state the following Proposition in which $\pi_G$ is some optimal and invariant policy.
\begin{prop}[Informal] \label{prop:gbdg}
For any $\pi \in \Pi$, we consider the occupancy measure $\{\rho_{\pi}^{e_i}(x^{e_i}, g) \}_{i=1}^N$ for training environments and $\rho_{\pi_G}^t(x^t, g)$ for the target environment. For simplicity, we use $\epsilon^{e_i}$ as the abbreviation of $\epsilon^{\rho_{\pi}^{e_i}(x^{e_i}, g)}$, $\epsilon^t$ as $\epsilon^{\rho_{\pi_G}^t(x^t, g)}$ and $\delta =\max_{e_i, e_i' \in \gE_{\text{train}}} d_{\Pi \Delta \Pi}(\rho_{\pi}^{e_i}(x^{e_i}, g), \rho_{\pi}^{e_i'}(x^{e_i'}, g))$. Let
\begin{align*}
    \lambda = \frac{1}{N} \sum_{i=1}^N \epsilon^{e_i}(\pi^* \parallel \pi_G) + \epsilon^{t}(\pi^* \parallel \pi_G), ~~~~ \pi^* = \argmin_{\pi' \in \Pi} \sum_{i=1}^N \epsilon^{e_i}(\pi' \parallel \pi_G) 
\end{align*}
Then, we have
\begin{align}
     J^t(\pi_G) - J^t(\pi)\leq & \frac{1}{N} \sum_{i=1}^N  \epsilon^{e_i}(\pi \parallel \pi_G) + \lambda + \delta + \min_{\rho(x, g) \in B} d_{\Pi \Delta \Pi}(\rho(x, g), \rho_{\pi_G}^t(x^t, g)) \label{equ:gbdg}
\end{align}
where $B$ is a characteristic set of joint distributions determined by $\gE_{\text{train}}$ and policy class $\Pi$.
\end{prop}
The formal statement and the proof are shown in \Apxref{apx:gbdg_proof}. Generally speaking, the first term of the right hand side in \Eqref{equ:gbdg} quantifies the performance of $\pi$ in the $N$ training environments. $\lambda$ quantifies the optimality of the policy class $\Pi$ over all environments. $\delta$ reflects how the policy class $\Pi$ can reflect the difference among $\{\rho^{e_i}_\pi(x^{e_i}, g), e_i \in \gE_{\text{train}} \}$, which should be small if the policy class is invariant. The last term characterizes the distance between training environment and target environment and will be small if the training environments are diversely distributed.

Many works on domain generalization of supervised learning \cite{ben2010datheory, liu2019transferable, sicilia2021dgtheory, albuquerque2019dgtheory, akuzawa2019advaccconstraint} spend much effort in discussing the trade-offs among different terms similar to the ones in \Eqref{equ:gbdg}, e.g., minimizing $\delta $ may increase $\lambda$ \cite{akuzawa2019advaccconstraint}, and in developing sophisticated techniques to optimize the bound, e.g. distribution matching \cite{louizos2016fairvae, li2018condalign, jin2020alignrestore} or adversarial learning \cite{liu2019transferable}. 

Different from their perspectives, in \mdp s, we propose a simple but effective criteria to minimize the bound. From now on, we only consider the policy class $\Pi = \Pi_{\Phi} = \{w (\Phi(x), g), \forall w\}$. Usually, $\Phi$ will be referred as an encoder which maps $x \in \gX^{\gE}$ to some latent representation $z = \Phi(x)$. We will also use the notation $z(s) = \Phi(x(s))$ if we do not emphasize on the specific environment.
\begin{definition}[\textbf{Perfect Alignment}] \label{def:pa}
An encoder is called a \emph{perfect alignment} encoder $\Phi$ w.r.t environment set $E$ if $\forall e, e' \in E$ and $\forall s, s' \in \gS$, $\Phi(x^e(s)) = \Phi(x^{e'}(s'))$ if and only if $s = s'$.
\end{definition}
As illustrated in \Figref{fig:tsne}, an encoder is in perfect alignment if it maps two observations of the same underlying state $s$ to the same latent encoding $z(s)$ while also preventing meaningless embedding, i.e., mapping observations of different states to the same $z$. We believe perfect alignment plays an important role in domain generalization for goal-conditioned RL agents. Specifically, it can minimize the bound of \Eqref{equ:gbdg} as follows.
\begin{prop}[Informal] \label{prop:padg}
If the encoder $\Phi$ is a perfect alignment over $\gE_{\text{train}}$, then
\begin{align} \label{equ:padg}
    J^t(\pi_G) - J^t(\pi) \leq \underbrace{\frac{1}{N} \sum_{i=1}^N \epsilon^{e_i}(\pi \parallel \pi_G)}_{(E)} +  \underbrace{\epsilon^t(\pi^* \parallel \pi_G) + d_{\Pi_{\Phi}\Delta\Pi_{\Phi}}(\tilde{\rho}(x, g), \rho^t_{\pi_G}(x^t, g))}_{(t)}
\end{align}
where $\tilde{\rho}(x, g)$ and $\pi^*$ are defined in \Propref{prop:gbdg} (also \Apxref{apx:proof}). 
\end{prop}

In \Apxref{apx:padg_proof}, we formally prove \Propref{prop:padg} when $\Phi$ is a $(\eta, \psi)$-perfect alignment, i.e., $\Phi$ is only near perfect alignment. The proof shows that the generalization error bound is minimized on the R.H.S of \Eqref{equ:padg} when $\Phi$ asymptotically becomes an exact perfect alignment encoder. Therefore, in our following method, we aim to learn a perfect alignment encoder via aligned sampling (\secref{sec:alignsample}).

For the remaining terms in the R.H.S of \Eqref{equ:padg}, we find it hard to quantify them task agnostically, as similar difficulties also exist in the domain generalization theory of supervised learning \cite{sicilia2021dgtheory}. Fortunately, we can derive upper bounds for the remaining terms under certain assumptions and we observe that these upper bounds are significantly reduced via our method in the experiments (\secref{sec:experiment}). The $(E)$ term represents how well the learnt policy $\pi$ approximates the optimal invariant policy on the training environments and is reduced to almost zero via RL (\Tableref{tab:comparative_evaluation}). For the $(t)$ term, we show that an upper bound of $(t)$ is proportion to the invariant quality of $\Phi$ on the target environment. Moreover, we find that learning a perfect alignment encoder over $\gE_{\text{train}}$ empirically improves the invariant quality over other unseen environments ($t$) (\Figref{fig:ablation}). Thus, this $(t)$ term upperbound is reduced by learning perfect alignment. Please refer to \Apxref{apx:remain_discuss} for more details.

Based on the theory we derived in this subsection, we adopt perfect alignment as the heuristic to address \mdp s in our work. In the following subsections, we propose a practical method to acquire a perfect alignment encoder over the training environments.

\subsection{Learning Domain Invariant via Aligned Sampling} \label{sec:alignsample}
First, we discuss about the \emph{if} condition on perfect alignment encoder $\Phi$, i.e., $\forall s, \Phi(x^e(s)) = \Phi(x^{e'}(s))$. The proposed method is based on \emph{aligned sampling}. In contrast, most RL algorithms use observation-dependent sampling from the environment, e.g., $\epsilon$-greedy or Gaussian distribution policies \cite{haarnoja2018sac, fujimoto2018td3, pong2020skew, dqn2015}. However, with observation-dependent sampling, occupancy measures  $\rho_{\pi}^e(s), \forall e \in \gE_{\text{train}}$ will be different. Thus, simply aligning the latent representation of these observations will fail to produce a perfect alignment encoder $\Phi$.

Thus, we propose a novel strategy for data collection called \emph{aligned sampling}. First, we randomly select a trajectory (e.g., from replay buffer etc.), denoted as $\{x_0^e, a_0, x_1^e, a_1, \ldots, x^e_T\}$ from environment $e$. The set of corresponding states along this trajectory are denoted as $\{s^e_t(a_{0:t})\}_{t=0}^T$. Second, we take the same action sequence $a_{0:T}$ in another domain $e'$ to get another trajectory $\{x_0^{e'}, a_0, x_1^{e'}, a_1, \ldots, x_T^{e'}\}$ (so as $\{s^{e'}_t(a_{0:t})\}_{t=0}^T$). We refer to the data collected by aligned sampling from all training environments as \emph{aligned data}. These aligned observations $\{x^{e_i}_t(a_{0:t})\}, \forall e_i \in \gE_{\text{train}}$ are stored in an aligned buffer $\gR_{\text{align}}$ corresponding to the aligned action sequence $a_{0:t}$.

Under the definition of \mdp, we have $\forall t \in [0:T], s \in \gS$, $\rho(s^e_t(a_{0:t})) = \rho(s^{e'}_t(a_{0:t}))$, i.e., the same state distribution. Therefore, we can use MMD loss \cite{gretton2008mmd} to match distribution of $\Phi(x^e(s))$ for the aligned data. More specifically, in each iteration, we sample a mini-batch of $B$ aligned observations of every training environment $e_i \in \gE_{\text{train}}$ from $\gR_{\text{align}}$, i.e., $\gB_{\text{align}} = \{x^{e_i}(s^{e_i}_t(a_{0:t}^b)), \forall e_i \in \gE_{\text{train}}\}_{b=1}^B$. Then we use the following loss as a computationally efficient approximation of the MMD metric \cite{zhao2015fastmmd, louizos2016fairvae}.
\begin{align*} 
    L^{\text{MMD}}(\Phi) =  \E_{e, e' \sim \gE_{\text{train}}, \gB_{\text{align}} \sim \gR_{\text{align}}} [ \parallel \frac{1}{B}\sum_{b=1}^B \psi(\Phi(x^e(s^e_t(a_{0:t}^b))) - \frac{1}{B} \sum_{b=1}^B \psi(\Phi(x^{e'}(s^{e'}_t(a_{0:t}^b)))) \parallel_2^2 ]
\end{align*}
where $\psi$ is a random expansion function. 

\begin{figure}
    \centering
    \subfigure[Illustration of Aligned Sampling]{ 
    \label{fig:illus_align}
    \includegraphics[width=2.625in, height=1.4765625in]{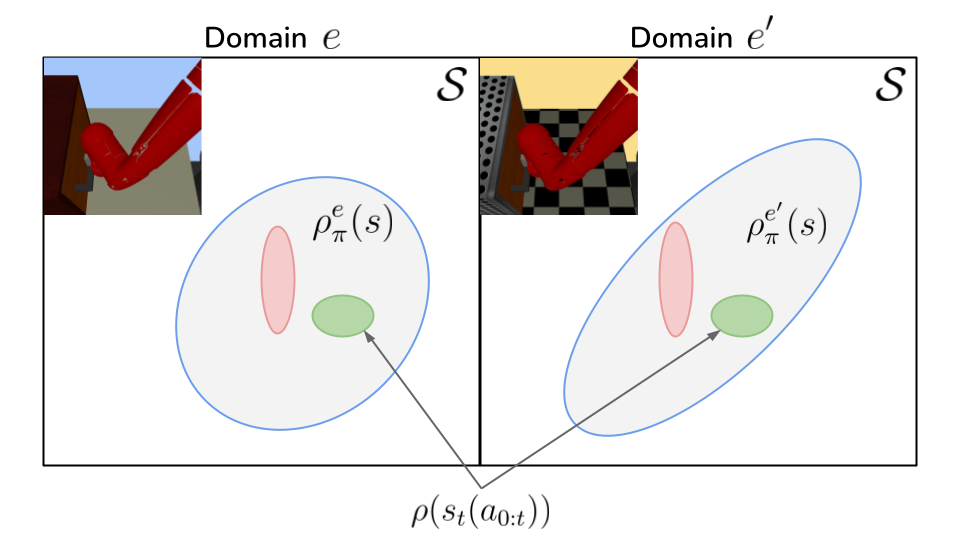} 
    }
    \subfigure[Overall structure]{ 
    \label{fig:vae_struct}
    \includegraphics[width=2.625in, height=1.4765625in]{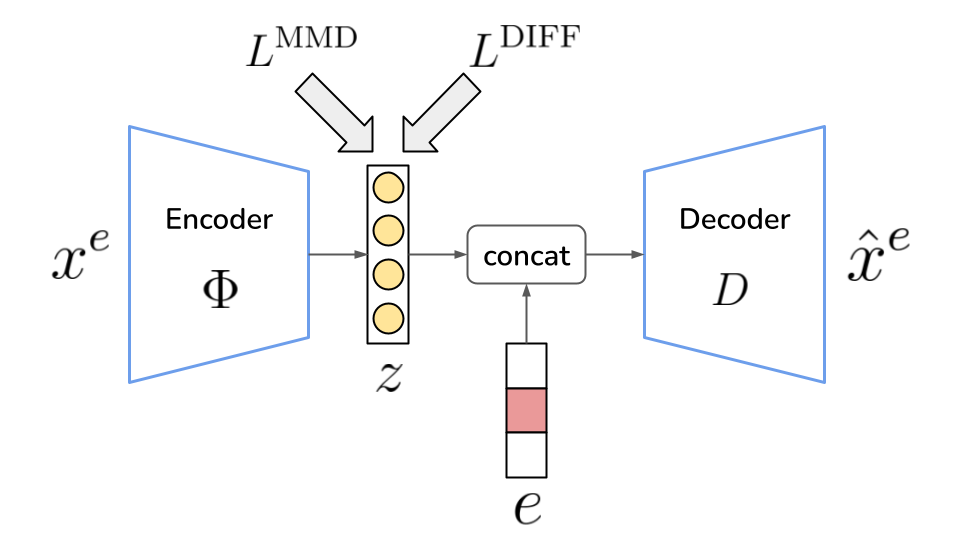}
    }
    \caption{\textbf{(a)}: Illustration of Aligned Sampling. Square represents the whole state space $\gS$, gray area represents the distribution $\rho_{\pi}^e(s)$ in two different environments. Small colored areas are the aligned state distribution generated by aligned sampling in \secref{sec:alignsample}. \textbf{(b)}: Overall VAE structure in our PA-SF. Encoder maps $x^e$ to the latent embedding $z$ and decoder $D$ reconstructs the observations with $z$ and index $e$. $L^{\text{MMD}}$ and $L^{\text{DIFF}}$ denote the two losses in \secref{sec:alignsample}.}
    \vskip -0.2in
\end{figure}

In \Figref{fig:illus_align}, we illustrate the intuition of the above approach. When the transition is nearly deterministic, the entropy for $\rho(s^e_t(a_{0:t}))$ is much smaller, i.e.,  $\gH(\rho(s^e_t(a_{0:t}))) \ll \gH(\rho^e_{\pi}(s_t))$. Thus, $\rho(s^e_t(a_{0:t}))$ can be regarded as small patches in $\gS$. We use the MMD loss $L^{\text{MMD}}$ to match the latent representation $\{\Phi(x^e(s)),  s \sim \rho(s^e_t(a_{0:t})) \}, \forall e \in \gE_{\text{train}}$ together. As a consequence, we should achieve an encoder $\Phi$ that is more aligned. We discuss the theoretical property of $L^{\text{MMD}}$ in detail in \Apxref{apx:mmd}.

However, simply minimizing $L^{\text{MMD}}$ may violate the \emph{only if} condition for perfect alignment. For example, a trivial solution for $L^{\text{MMD}} = 0$ is mapping all observations to some constant latent. To ensure that $\Phi(x^e(s)) = \Phi(x^{e'}(s'))$ \emph{only if} $s = s'$, we additionally use the difference loss $L^{\text{DIFF}}$ as follows.
\begin{align*}
    L^{\text{DIFF}}(\Phi) = - \E_{e \sim \gE_{\text{train}}, x^e, \tilde{x}^e \in \gR^e} \parallel \Phi(x^e) - \Phi(\tilde{x}^e) \parallel_2^2
\end{align*}
where $\gR^e$ refers to the replay buffer of environment $e$. Clearly, minimizing $L^{\text{DIFF}}$ encourages dispersed latent representations over all states $s \in \gS$.

We refer to the combination $\alpha_{\text{MMD}} L^{\text{MMD}} + \alpha_{\text{DIFF}} L^{\text{DIFF}}$ as our perfect alignment loss $L^{\text{PA}}$. Note that $L^{\text{PA}}$ resembles contrastive learning \cite{chen2020simple, laskin2020curl}. Namely, observations of aligned data from $\gR_{\text{align}}$ are positive pairs while observations sampled randomly from a big replay buffer are negative pairs. We match the latent embedding of positive pairs via the MMD loss while separating negative pairs via the difference loss. As discussed in \secref{sec:dgtheory}, we believe this latent representation will improve generalization to unseen target environments.

\subsection{Perfect Alignment for Skew-Fit}
\label{sec:alignsampleskewfit}

In \secref{sec:experiment}, we will train goal-conditioned RL agents with perfect alignment encoder using the Skew-Fit algorithm \cite{pong2020skew}. Skew-Fit is typically designed for visual-input agents which learn a goal-conditioned policy via purely self-supervised learning.

First, Skew-Fit trains a $\beta$-VAE with observations collected online to acquire a compact and meaningful latent representation for each state, i.e., $z(s)$ from the image observations $x(s)$. Then, Skew-Fit optimizes a SAC \cite{haarnoja2018sac} agent in the goal-conditioned setting over the latent embedding of the image observation and goal, $\pi(a|z,g)$. The reward function is the negative of $l_2$ distance between the two latent representation $z(s)$ and $z(g)$, i.e., $r(s,g) = -\parallel z(s) - z(g) \parallel_2$. Furthermore, to improve sample efficiency, Skew-Fit proposes skewed sampling for goal-conditioned exploration.

In our algorithm, \emph{Perfect Alignment for Skew-Fit} (PA-SF), the encoder $\Phi$ is optimized via both $\beta$-VAE losses  as \cite{pong2020skew, nair2018rig} and $L^{\text{PA}}$ loss to ensure meaningful and perfectly aligned latent representation.
\begin{align} \label{equ:pasfobj}
    L(\Phi, D) = L^{\text{RECON}}(x^e, \hat{x^e}) + \beta D_\text{KL}(q_{\Phi}(z|x^e) \parallel p(z)) + \alpha_{\text{MMD}} L^{\text{MMD}} + \alpha_{\text{DIFF}} L^{\text{DIFF}}
\end{align}
In addition, we use both aligned sampling and observation-dependent sampling. Aligned sampling provides aligned data but hurts sample-efficiency while observation-dependent sampling is exploration-efficient but fails to ensure alignment. In practice, we find that collecting a small portion ($15\%$ of all data collected) of aligned data in $\gR_{\text{align}}$ is enough for perfect alignment via $L^{\text{PA}}$.

Additionally, inspired by \cite{louizos2016fairvae}, we also change the $\beta$-VAE structure to what is shown in \Figref{fig:vae_struct}, since in \mdp~data are collected from $N$ training environments and thus, the identity Gaussian distribution is no longer a proper fit for prior. The encoder $\Phi$ maps $x^e(s)$ to some latent representation $z(s)$ while the decoder $D$ takes both $z(s)$ and the environment index $e$ as input to reconstruct $\hat{x^e}(s)$. Note that by using both $L^{\text{PA}}$ and $L^{\text{RECON}}$, we require static environmental factors in $\gE_{\text{train}}$ (unnecessary for testing environments) for a stable optimization. In future work, we will address the limit from $\beta$-VAE by training two latent representations simultaneously to stabilize the optimization for generality.

\section{Experiments} \label{sec:experiment}

In this section, we conduct experiments to evaluate our PA-SF algorithms. The experiments are based on multiworld \cite{vitchyr2018multiworld}. Our empirical analysis tries to answer the following questions: (1) How well does PA-SF perform in solving \mdp{} problems? (2) How does each component proposed in \secref{sec:method} contribute to the performance? 

\subsection{Comparative Evaluation} \label{sec:compare_eval}

In this subsection, we aim to answer the question (1) by comparing our proposed PA-SF method with vanilla Skew-Fit and several other baselines that attempt to acquire invariant policies for RL agents. 

\paragraph{Baselines} Current methods for obtaining robust policies can be characterized into two categories: (1) data augmentation and (2) model bisimulation.

\begin{enumerate}[leftmargin=0.2in, itemsep=-0.5pt, topsep=-0pt, partopsep=-3pt]
    \item \textbf{Data Augmentation}. Recent work \cite{stone2021distractingbench} tries to use data augmentation to prevent the RL agents from distractions. We implement the most widely accepted data augmentation methods RAD \cite{laskin2020rad} upon Skew-Fit (Skew-Fit + RAD) as a baseline. Note that our PA-SF method does not use any data augmentation and is parallel with this kind of techniques.
    \item \textbf{Model Bisimulation} \cite{ferns2011bisimulation}. These methods utilize bisimulation metrics to learn a minimal but sufficient representation which will neglect irrelevant features of Block MDPs. We include MISA \cite{zhang2020inblock} and DBC \cite{zhang2020inrepr} in our comparison as they are the most successful implementations for high-dimensional tasks. Moreover, in the goal-conditioned setting, we use an oracle state-goal distance $-\parallel s - g \parallel_2$ as rewards for these two algorithms in \mdp. In contrast, our PA-SF method does not have such information.
\end{enumerate}

\begin{table}
  \caption{Evaluation of PA-SF and baselines on four control tasks. We report the mean and one standard deviation on each task (lower metric is better).}
  \label{tab:comparative_evaluation}
  \centering
  {\setlength{\tabcolsep}{0.32em} 
  \begin{tabular}{llllll}
    \toprule
    & \multirow{3}{*}{Algorithm} & Reach & Door & Push & Pickup \\
    \cmidrule(r){3-6}
    & & Hand distance & Angle difference & Puck distance & Object distance \\
    & & (35K) & (150K) & (400K) & (280K) \\
    \midrule
    \multirow{5}{*}{Test Avg} & Skew-Fit & $0.111 \pm 0.010$ & $0.194 \pm 0.018$ & $0.086 \pm 0.004$ & $0.037 \pm 0.006$ \\
    & Skew-Fit + RAD & $0.105 \pm 0.010$ & $0.162 \pm 0.030$ & $0.082 \pm 0.008$ & $0.040 \pm 0.004$ \\
    & MISA & $0.239 \pm 0.0142$ & $0.255 \pm 0.027$ & $0.099 \pm 0.006$ & $0.043 \pm 0.004$ \\
    & DBC & $0.185 \pm 0.037$ & $0.320 \pm 0.033$ & $0.095 \pm 0.006$ & $0.045 \pm 0.002$ \\
    & PA-SF(\textbf{Ours}) & $\mathbf{0.076} \pm 0.005$ & $\mathbf{0.106} \pm 0.015 $& $\mathbf{0.069} \pm 0.005$ & $\mathbf{0.028} \pm 0.004$ \\
    \midrule
    \multirow{2}{*}{Train Avg}  & PA-SF(\textbf{Ours}) & $0.067 \pm 0.005$ & $0.058 \pm 0.074$ & $0.060 \pm 0.005$ & $0.020 \pm 0.008$ \\
    & Oracle Skew-Fit & $0.055 \pm 0.010$ & $0.057 \pm 0.012$& $0.054 \pm 0.006$ & $0.020 \pm 0.006$ \\
    \bottomrule
  \end{tabular}}
\end{table}

\begin{figure}[t]
  \centering
  \includegraphics[width=\linewidth]{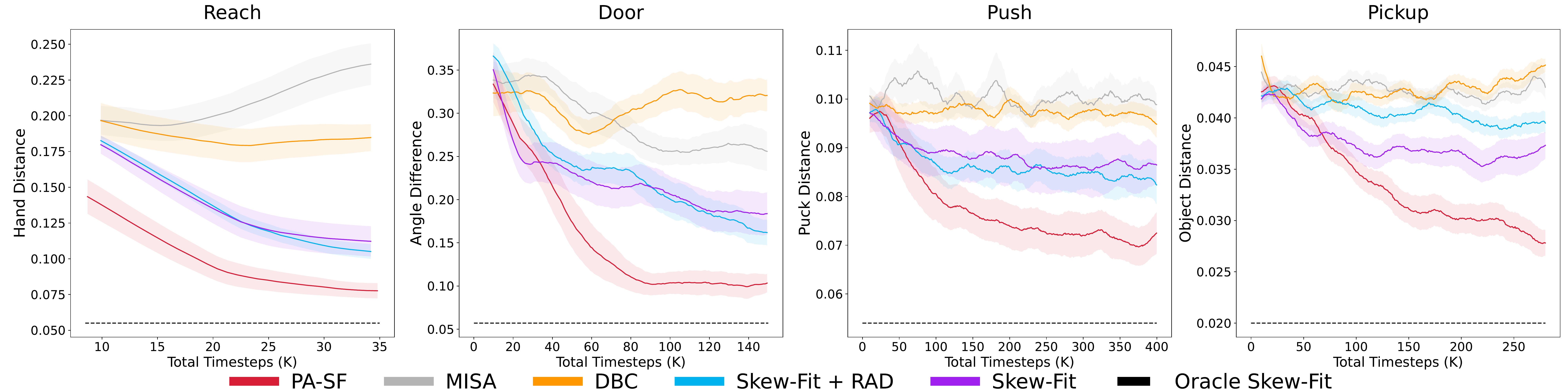}
  \caption{Learning curve of all algorithms on average across test environments for each task. All curves show the mean and one standard deviation (a half for Pickup to show clearly) of 7 seeds.}
  \label{fig:learning_curve}
  \vskip -0.2in
\end{figure}

\paragraph{Environments} We evaluate PA-SF and all baselines on a set of \mdp{} tasks based on multiworld benchmark \cite{vitchyr2018multiworld}, which is widely used to evaluate the performance of visual input goal-conditioned algorithms. We use the following four basic tasks \cite{nair2018rig, pong2020skew}: \textit{Reach}, \textit{Door}, \textit{Pickup} and \textit{Push}. In \mdp, we create different environments with various backgrounds, desk surfaces, and object appearances. During testing, we also create environments with unseen video backgrounds to mimic environmental factor transitions $q^e(b_{t+1}^e|b_{t}^e)$. This makes policy generalization more challenging. 
Please refer to \Apxref{apx:imp} for a full description of our experiment setup and implementation details of the baselines and our algorithm.

\paragraph{Results} In \Tableref{tab:comparative_evaluation}, we show the final average performance of each algorithm on \emph{unseen} test environments $\gE_{\text{test}}$. The corresponding learning curves are shown in \Figref{fig:learning_curve}. This metric shows the generalizability of each RL agent. All these agents are allowed to collect data from $\gE_{\text{train}}$ ($N=3$) with static environmental factors. Our PA-SF achieves SOTA performance on all tasks. On testing environments, we achieve a relative reduction around 40\% to 65\% of the corresponding metrics over vanilla Skew-Fit w.r.t the optimal metric possible (Oracle Skew-Fit). Oracle Skew-Fit refers to the performance of a Skew-Fit algorithm trained directly on the single environment (and \emph{not} simultaneously on all $\gE_{\text{train}}$).

Other invariant policy learning methods perform sluggishly on all tasks. For DBC and MISA, we hypothesize that they struggle for goal-conditioned problems since the model bisimulation metric is defined for a single MDP. In \mdp s, this means acquiring a set of encoders $\Phi_g$ that achieves model bisimulation for every possible $g$ and is thus inefficient for learning. By design, our method is not susceptible to this issue as the perfect alignment is a universal invariant representation for all goals. Data augmentation via RAD provides marginal improvement over the vanilla Skew-Fit. Nevertheless, we believe developing adequate data augmentation techniques for GBMDPs is an important research problem and is orthogonal with our method.

Additionally, we also show the performance of PA-SF on the training environments in \Tableref{tab:comparative_evaluation}. PA-SF is still comparable and as sample-efficient as Skew-Fit in the training environments. This supports the claim that the $(E)$ term in the R.H.S of \Eqref{equ:padg} is reduced to almost zero via RL training in practice.
\subsection{Design Evaluation} \label{sec:ablate}

In this subsection, we conduct comprehensive analysis on the design of  PA-SF to interpret how well it carries out the theoretical framework discussed in \secref{sec:dgtheory} and \secref{sec:alignsample}.

To begin with, we show the learning curves in \Figref{fig:ablation} of different ablations of PA-SF in the \textit{Door} environment during both training and testing. To understand the roles of $L^{\text{DIFF}}$ and $L^{\text{MMD}}$, PA-SF (w/o D) excludes $L^{\text{DIFF}}$ and PA-SF (w/o MD) excludes both losses\footnote{A single $L^{\text{DIFF}}$ is not useful here.}.  Noticing that PA-SF (w/o MD) is equivalent to the Skew-Fit algorithm with our proposed VAE structure (\Figref{fig:vae_struct}).  We also add PA-SF (w/o AS) which excludes aligned sampling.

\begin{figure}[h]
  \centering
  \includegraphics[width=\linewidth]{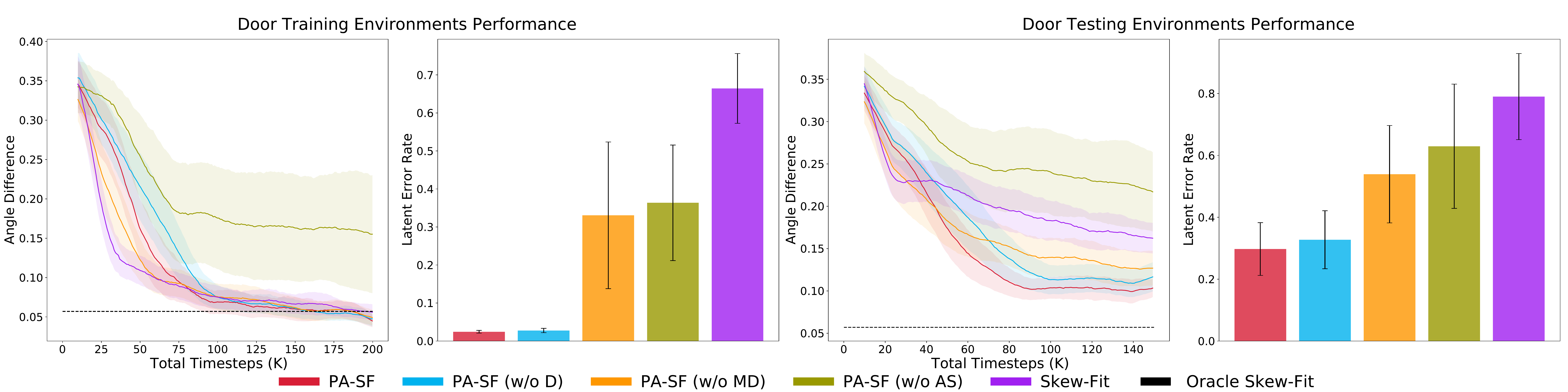}
  \caption{Ablation of PA-SF and visualization of the latent representation via LER metric. All curves represent the mean and one standard deviation across 7 seeds.}
  \label{fig:ablation}
  \vskip -0.2in
\end{figure}

Additionally, we also quantify the quality of the latent representation $\Phi(x^e(s))$ in \Figref{fig:ablation} via the metric \emph{Latent Error Rate} (LER). LER is defined as the average over environment set $E \in \{\gE_{\text{train}}, \gE_{\text{test}}\}$ as follows:
\begin{align*}
    \text{Err}(\Phi) = \E_{e \sim E, s \sim \gS} \left[ \frac{\parallel \Phi(x^e(s)) - \Phi(x^{e_0}(s)) \parallel_2}{\parallel \Phi(x^e(s)) \parallel_2} \right]
\end{align*}

In general, the smaller $\text{Err}(\Phi)$ is, the closer the encoder $\Phi$ is to perfect alignment over the environments $E$. We first focus on the discussion about training performance.
\begin{enumerate}[leftmargin=0.2in, itemsep=-0.5pt, topsep=-0.5pt, partopsep=-1.5pt]
    \item $\Phi$ achieves the \emph{if} condition of perfect alignment over $\gE_{\text{train}}$ via $L^{\text{MMD}}$ as the LER value of PA-SF and PA-SF (w/o D) is almost $0$. While without MMD loss, PA-SF (w/o MD) and Skew-Fit struggle with large LER value despite achieving good training performance. Furthermore, the comparison between PA-SF and PA-SF (w/o AS) demonstrates the importance of using aligned data in the MMD loss (Otherwise, the matching is inherently erroneous). 
    \item The \emph{only if} condition, i.e., $\Phi(x^e(s)) = \Phi(x^{e'}(s'))$ only if $s=s'$, is also achieved empirically by visualizing the reconstruction of the VAE (\Figref{fig:additional_recon} in \Apxref{apx:additional}) and we believe this is satisfied by both the difference loss and the reconstruction loss. Under the \emph{only if} condition, the SAC \cite{haarnoja2018sac} trained on the latent space achieves the optimal performance. In contrast, PA-SF (w/o AS) fails to learn well on the training environments as its latent representation is mixed over different states.
\end{enumerate}

Second, we focus on the generalization performance on target domains $t$, i.e., term $(t)$ in \Eqref{equ:padg}. We observe the following:

\begin{enumerate}[leftmargin=0.2in, itemsep=-0.5pt, topsep=-0.5pt, partopsep=-1.5pt]
    \item As shown by the learning curve of test environments, the target domain performance of different ablations match that of the LER metric: SkewFit, PA-SF (w/o AS) > PA-SF (w/o MD) > PA-SF (w/o D) > PA-SF. During training, these ablations have almost the same performance, except PA-SF (w/o AS). This indicates that the increased test performance indeed comes from the improved representation quality of the encoder $\Phi$, i.e., more aligned. This supports our claim at the end of \secref{sec:dgtheory} and the upper bound analysis on the $(t)$ term in \Apxref{apx:remain_discuss}, that the increased invariant property of $\Phi$ produces better domain generalization performance.
    \item In test environment ablations, the LER is reduced significantly on methods with $L^{\text{MMD}}$. This supports our claim that a perfect alignment encoder on training environments also improves the encoder's invariant property on unseen environments. In addition, by encouraging dispersed latent representation, the difference loss $L^{\text{DIFF}}$ also plays a role in reducing LER during testing. This supports the necessity of both losses for generalization.
\end{enumerate}

We observe the similar results in other tasks as well (\Apxref{apx:additional}). Here, we also visualize the latent space by t-SNE plot to illustrate the perfect alignment on task \textit{Push}. Dots in training environments are matched perfectly and the corresponding test environment dot is approximately near as expected.

\begin{figure}
  \centering
  \includegraphics[scale=0.35]{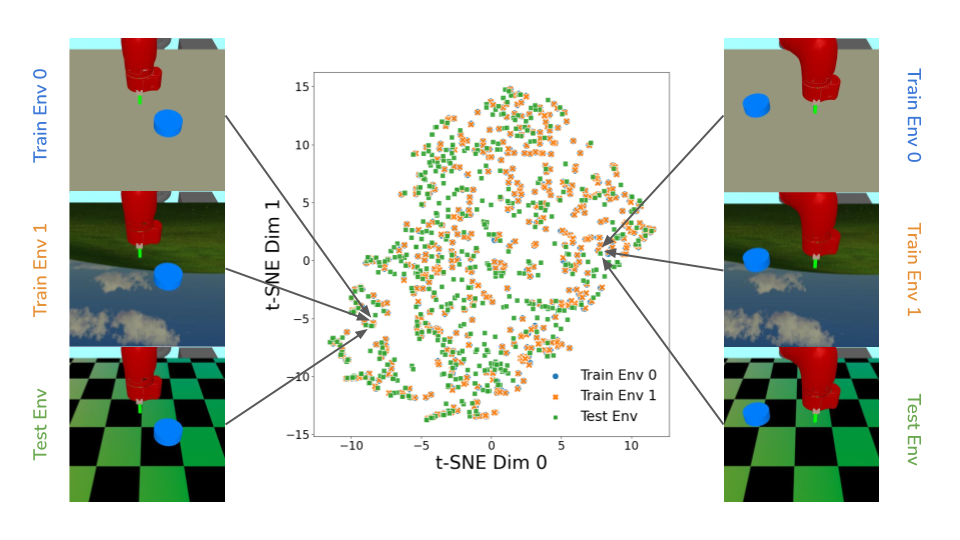}
  \caption{t-SNE visualization of the latent space $\Phi(x^e)$ trained with PA-SF for three environments: 2 training and 1 testing of \textit{Push} as well as instances visualization.}
  \label{fig:tsne}
  \vskip -0.2in
\end{figure}

\section{Related Work} \label{sec:related}
\textbf{Goal-conditioned RL:} Goal-conditioned RL \cite{kaelbling1993learning,schaul2015universal} removes the need for complicated reward shaping by only rewarding agents for reaching a desired set of goal states. 
RIG \cite{nair2018rig} is the seminal work for visual-input, reward-free learning in goal-conditioned MDPs. Skew-Fit \cite{pong2020skew} improves over RIG \cite{nair2018rig} in training efficiency by ensuring the behavioural goals used to explore are diverse and have wide state coverage. However, Skew-Fit has its own limitation in understanding the semantic meaning of the goal-conditioned task. To acquire more meaningful goal's and observation's latent representation, several approaches apply inductive biases or seek human feedback. ROLL \cite{wang2020roll} applies object extraction methods under strong assumptions, while WSC \cite{lee2020weakly} uses weak binary labeled data as the reward function. 
Others explore the same goal-conditioned RL problem via hindsight experience replay \cite{andry2017her, ren2019exploration, ghosh2019isl}, unsupervised reward learning \cite{pere2018unsupervisedgoal}, inverse dynamics models \cite{paster2020glamour}, C-learning \cite{eysenbach2020clearn}, goal generation \cite{florensa2018goalgan, nair2019hierforsignt, pitis2020maximumentro}, goal-conditioned forward models \cite{nair2020goalmodel}, and hierarchical RL \cite{li2021slow, nachum2018nor, zhang2020world, hou2020hac}. 
Our study focus on learning goal-conditioned policies that is invariant of spurious environmental factors. We aim to learn a policy that can generalize to visual goals in unseen test environments.

\textbf{Learning Invariants in RL:} Robustness to domain shifts is crucial for real-world applications of RL. \cite{zhang2020inblock, zhang2020inrepr, gelada2019deepmdp} implement the model-bisimulation framework \cite{ferns2011bisimulation} to acquire a minimal but sufficient representation for solving the MDP problem. However, model-bisimulation for high-dimension problems typically requires domain-invariant and dense rewards. These assumptions do not hold in \mdp s. Contrastive Metric Embeddings (CME) \cite{agarwal2021contrastivemdp} instead uses $\pi^*$-bisimulation metric but it also requires extra information of the optimal policy. Another line of work tries to address these issues via self-supervised learning. \cite{stone2021distractingbench} tests multiple data augmentation methods including RAD \cite{laskin2020rad} and DrQ \cite{kostrikov2020image} to boost the robustness of the representation as well as the policy. Our work can also apply data augmentation in practice. However, we find that RAD is not very helpful in the Skew-Fit framework. Additionally, \cite{hansen2020selfdeploy, bodnar2020geometric} use self-supervised correction during real-world adaptation like sim2real transfer but these methods are incompatible for domain generalization.
\section{Conclusion}

In this paper, we study the problem of learning invariant policies in Goal-conditioned RL agents. The problem is formulated as a \mdp{}, which is an extension of Goal-conditioned MDPs and Block MDPs where we want the agent's policy to generalize to unseen test environments after training on several training environments.

As supported by the generalization bound for \mdp, we propose a simple but effective heuristic, i.e., perfect alignment which we can minimize the bound asymptotically and benefit the generalization. To learn a perfect alignment encoder, we propose a practical method based on aligned sampling. The method resembles contrastive learning: matching latent representation of aligned data via MMD loss and dispersing the whole latent representations via the DIFF loss. Finally, we propose a practical implementation Perfect Alignment for Skew-Fit (PA-SF) by adding the perfect alignment loss to Skew-Fit and changing the VAE structure to handle \mdp s.

The empirical evaluation shows that our method is the SOTA algorithm and achieves a remarkable increase in test environments' performance over other methods. We also compare our algorithm with several ablations and analyze the representation quantitatively. The results support our claims in the theoretical analysis that perfect alignment criteria is effective and that we can effectively optimize the criteria with our proposed method.
We believe the perfect alignment criteria will enable applications in diverse problem settings and offers interesting directions for future work, such as extensions to other goal-conditioned learning frameworks \cite{eysenbach2020clearn, paster2020glamour}. 


\section*{Acknowledgements}
We are grateful for the feedback from anonymous reviewers. Resources used in preparing this research were provided, in part, by the Province of Ontario, the Government of Canada through CIFAR, and companies sponsoring the Vector Institute <\href{http://www.vectorinstitute.ai/partners}{http://www.vectorinstitute.ai/partners}>.



\bibliographystyle{unsrtnat}  
\bibliography{neurips_2021}  

\begin{thebibliography}{56}
\providecommand{\natexlab}[1]{#1}
\providecommand{\url}[1]{\texttt{#1}}
\expandafter\ifx\csname urlstyle\endcsname\relax
  \providecommand{\doi}[1]{doi: #1}\else
  \providecommand{\doi}{doi: \begingroup \urlstyle{rm}\Url}\fi

\bibitem[Silver et~al.(2017)Silver, Schrittwieser, Simonyan, Antonoglou, Guez,
  Hubert, Baker, Lai, Bolton, Chen, Lillicrap, Hui, Sifre, van~den Driessche,
  Graepel, and Hassabis]{alphagozero2017}
David Silver, Julian Schrittwieser, Karen Simonyan, Aj~Antonoglou, Ioannis
  abd~Huang, Arthur Guez, Thomas Hubert, Lucas Baker, Matthew Lai, Adrian
  Bolton, Yutian Chen, Timothy Lillicrap, Fan Hui, Laurent Sifre, George
  van~den Driessche, Thore Graepel, and Demis Hassabis.
\newblock Mastering the game of go without human knowledge.
\newblock \emph{Nature}, 550, 2017.

\bibitem[Mnih et~al.(2015)Mnih, Kavukcuoglu, Silver, Rusu, Veness, Bellemare,
  Graves, Riedmiller, Fidjeland, strovski, Petersen, Beattie, Sadik,
  Antonoglou, King, umaran, Wierstra, Legg, and Hassabis]{dqn2015}
Volodymyr Mnih, Koray Kavukcuoglu, David Silver, Andrei~A Rusu, Joel Veness,
  Marc Bellemare, Alex Graves, Martin Riedmiller, Andreas Fidjeland, Georg
  strovski, Stig Petersen, Charles Beattie, Amir Sadik, Ioannis Antonoglou,
  Helen King, Dharshan umaran, Daan Wierstra, Shane Legg, and Demis Hassabis.
\newblock Human-level control through deep reinforcement learning.
\newblock \emph{Nature}, 518, 2015.

\bibitem[Haarnoja et~al.(2018)Haarnoja, Zhou, Abbeel, and
  Levine]{haarnoja2018sac}
Tuomas Haarnoja, Aurick Zhou, Pieter Abbeel, and Sergey Levine.
\newblock Soft actor-critic: Off-policy maximum entropy deep reinforcement
  learning with a stochastic actor.
\newblock In \emph{Proceedings of the 35th International Conference on Machine
  Learning}, volume~80 of \emph{Proceedings of Machine Learning Research},
  pages 1861--1870, Stockholmsmässan, Stockholm Sweden, 10--15 Jul 2018. PMLR.

\bibitem[Packer et~al.(2018)Packer, Gao, Kos, Kr{\"{a}}henb{\"{u}}hl, Koltun,
  and Song]{packergao}
Charles Packer, Katelyn Gao, Jernej Kos, Philipp Kr{\"{a}}henb{\"{u}}hl,
  Vladlen Koltun, and Dawn Song.
\newblock Assessing generalization in deep reinforcement learning.
\newblock \emph{CoRR}, abs/1810.12282, 2018.
\newblock URL \url{http://arxiv.org/abs/1810.12282}.

\bibitem[Julian et~al.(2020)Julian, Swanson, Sukhatme, Levine, Finn, and
  Hausman]{julian2020efficient}
Ryan Julian, Benjamin Swanson, Gaurav~S Sukhatme, Sergey Levine, Chelsea Finn,
  and Karol Hausman.
\newblock Efficient adaptation for end-to-end vision-based robotic
  manipulation.
\newblock \emph{arXiv preprint arXiv:2004.10190}, 2020.

\bibitem[Zhang et~al.(2020{\natexlab{a}})Zhang, Lyle, Sodhani, Filos,
  Kwiatkowska, Pineau, Gal, and Precup]{zhang2020inblock}
Amy Zhang, Clare Lyle, Shagun Sodhani, Angelos Filos, Marta Kwiatkowska, Joelle
  Pineau, Yarin Gal, and Doina Precup.
\newblock Invariant causal prediction for block {MDP}s.
\newblock In \emph{International Conference on Machine Learning}, pages
  11214--11224. PMLR, 2020{\natexlab{a}}.

\bibitem[Zhang et~al.(2020{\natexlab{b}})Zhang, McAllister, Calandra, Gal, and
  Levine]{zhang2020inrepr}
Amy Zhang, Rowan McAllister, Roberto Calandra, Yarin Gal, and Sergey Levine.
\newblock Learning invariant representations for reinforcement learning without
  reconstruction.
\newblock \emph{arXiv preprint arXiv:2006.10742}, 2020{\natexlab{b}}.

\bibitem[Marcin et~al.(2017)Marcin, Filip, Alex, Jonas, Rachel, Peter, Bob,
  Josh, Pieter, and Wojciech]{andry2017her}
Andrychowicz Marcin, Wolski Filip, Ray Alex, Schneider Jonas, Fong Rachel,
  Welinder Peter, McGrew Bob, Tobin Josh, Abbeel Pieter, and Zaremba Wojciech.
\newblock Hindsight experience replay.
\newblock In \emph{Advances in Neural Information Processing Systems}, 2017.

\bibitem[Eysenbach et~al.(2020)Eysenbach, Salakhutdinov, and
  Levine]{eysenbach2020clearn}
Benjamin Eysenbach, Ruslan Salakhutdinov, and Sergey Levine.
\newblock C-learning: Learning to achieve goals via recursive classification.
\newblock \emph{arXiv preprint arXiv:2011.08909}, 2020.

\bibitem[Paster et~al.(2020)Paster, McIlraith, and Ba]{paster2020glamour}
Keiran Paster, Sheila~A McIlraith, and Jimmy Ba.
\newblock Planning from pixels using inverse dynamics models.
\newblock \emph{arXiv preprint arXiv:2012.02419}, 2020.

\bibitem[P{\'e}r{\'e} et~al.(2018)P{\'e}r{\'e}, Forestier, Sigaud, and
  Oudeyer]{pere2018unsupervisedgoal}
Alexandre P{\'e}r{\'e}, S{\'e}bastien Forestier, Olivier Sigaud, and
  Pierre-Yves Oudeyer.
\newblock Unsupervised learning of goal spaces for intrinsically motivated goal
  exploration.
\newblock In \emph{International Conference on Learning Representations}, 2018.

\bibitem[Ferns et~al.(2011)Ferns, Panangaden, and
  Precup]{ferns2011bisimulation}
Norm Ferns, Prakash Panangaden, and Doina Precup.
\newblock Bisimulation metrics for continuous markov decision processes.
\newblock \emph{SIAM Journal on Computing}, 40\penalty0 (6):\penalty0
  1662--1714, 2011.

\bibitem[Pong et~al.(2020)Pong, Dalal, Lin, Nair, Bahl, and
  Levine]{pong2020skew}
Vitchyr Pong, Murtaza Dalal, Steven Lin, Ashvin Nair, Shikhar Bahl, and Sergey
  Levine.
\newblock Skew-fit: State-covering self-supervised reinforcement learning.
\newblock In \emph{International Conference on Machine Learning}, pages
  7783--7792. PMLR, 2020.

\bibitem[Du et~al.(2019)Du, Krishnamurthy, Jiang, Agarwal, Dudik, and
  Langford]{du2019block}
Simon Du, Akshay Krishnamurthy, Nan Jiang, Alekh Agarwal, Miroslav Dudik, and
  John Langford.
\newblock Provably efficient {RL} with rich observations via latent state
  decoding.
\newblock In \emph{International Conference on Machine Learning}, pages
  1665--1674. PMLR, 2019.

\bibitem[Kaelbling(1993)]{kaelbling1993learning}
Leslie~Pack Kaelbling.
\newblock Learning to achieve goals.
\newblock In \emph{IJCAI}, pages 1094--1099. Citeseer, 1993.

\bibitem[Schaul et~al.(2015)Schaul, Horgan, Gregor, and
  Silver]{schaul2015universal}
Tom Schaul, Daniel Horgan, Karol Gregor, and David Silver.
\newblock Universal value function approximators.
\newblock In \emph{International conference on machine learning}, pages
  1312--1320. PMLR, 2015.

\bibitem[Greg et~al.(2016)Greg, Vicki, Ludwig, Jonas, John, Jie, and
  Wojciech]{greg2016openai}
Brockman Greg, Cheung Vicki, Pettersson Ludwig, Schneider Jonas, Schulman John,
  Tang Jie, and Zaremba Wojciech.
\newblock Openai gym, 2016.

\bibitem[Pong et~al.(2018)Pong, Dalal, Lin, and Nair]{vitchyr2018multiworld}
Vitchyr Pong, Murtaza Dalal, Steven Lin, and Ashvin Nair.
\newblock multiworld.
\newblock \url{https://github.com/vitchyr/multiworld}, 2018.

\bibitem[Koh et~al.(2020)Koh, Sagawa, Marklund, Xie, Zhang, Balsubramani, Hu,
  Yasunaga, Phillips, Beery, et~al.]{koh2020wilds}
Pang~Wei Koh, Shiori Sagawa, Henrik Marklund, Sang~Michael Xie, Marvin Zhang,
  Akshay Balsubramani, Weihua Hu, Michihiro Yasunaga, Richard~Lanas Phillips,
  Sara Beery, et~al.
\newblock Wilds: A benchmark of in-the-wild distribution shifts.
\newblock \emph{arXiv preprint arXiv:2012.07421}, 2020.

\bibitem[Arjovsky et~al.(2019)Arjovsky, Bottou, Gulrajani, and
  Lopez-Paz]{arjovsky2019irm}
Martin Arjovsky, L{\'e}on Bottou, Ishaan Gulrajani, and David Lopez-Paz.
\newblock Invariant risk minimization.
\newblock \emph{arXiv preprint arXiv:1907.02893}, 2019.

\bibitem[Ben-David et~al.(2010)Ben-David, Blitzer, Crammer, Kulesza, Pereira,
  and Vaughan]{ben2010datheory}
Shai Ben-David, John Blitzer, Koby Crammer, Alex Kulesza, Fernando Pereira, and
  Jennifer~Wortman Vaughan.
\newblock A theory of learning from different domains.
\newblock \emph{Machine learning}, 79\penalty0 (1):\penalty0 151--175, 2010.

\bibitem[Sicilia et~al.(2021)Sicilia, Zhao, and Hwang]{sicilia2021dgtheory}
Anthony Sicilia, Xingchen Zhao, and Seong~Jae Hwang.
\newblock Domain adversarial neural networks for domain generalization: When it
  works and how to improve.
\newblock \emph{arXiv preprint arXiv:2102.03924}, 2021.

\bibitem[Albuquerque et~al.(2019)Albuquerque, Monteiro, Darvishi, Falk, and
  Mitliagkas]{albuquerque2019dgtheory}
Isabela Albuquerque, Jo{\~a}o Monteiro, Mohammad Darvishi, Tiago~H Falk, and
  Ioannis Mitliagkas.
\newblock Generalizing to unseen domains via distribution matching.
\newblock \emph{arXiv preprint arXiv:1911.00804}, 2019.

\bibitem[Wikipedia(2021)]{wiki:Total_variation_distance_of_probability_measures}
Wikipedia.
\newblock {Total variation distance of probability measures} --- {W}ikipedia{,}
  the free encyclopedia.
\newblock
  \url{http://en.wikipedia.org/w/index.php?title=Total\%20variation\%20distance\%20of\%20probability\%20measures&oldid=1024556616},
  2021.
\newblock [Online; accessed 24-May-2021].

\bibitem[Liu et~al.(2019)Liu, Long, Wang, and Jordan]{liu2019transferable}
Hong Liu, Mingsheng Long, Jianmin Wang, and Michael Jordan.
\newblock Transferable adversarial training: A general approach to adapting
  deep classifiers.
\newblock In \emph{International Conference on Machine Learning}, pages
  4013--4022. PMLR, 2019.

\bibitem[Akuzawa et~al.(2019)Akuzawa, Iwasawa, and
  Matsuo]{akuzawa2019advaccconstraint}
Kei Akuzawa, Yusuke Iwasawa, and Yutaka Matsuo.
\newblock Adversarial invariant feature learning with accuracy constraint for
  domain generalization.
\newblock In \emph{Joint European Conference on Machine Learning and Knowledge
  Discovery in Databases}, pages 315--331. Springer, 2019.

\bibitem[Louizos et~al.(2016)Louizos, Swersky, Li, Welling, and
  Zemel]{louizos2016fairvae}
Christos Louizos, Kevin Swersky, Yujia Li, Max Welling, and Richard~S Zemel.
\newblock The variational fair autoencoder.
\newblock In \emph{ICLR}, 2016.

\bibitem[Li et~al.(2018)Li, Tian, Gong, Liu, Liu, Zhang, and
  Tao]{li2018condalign}
Ya~Li, Xinmei Tian, Mingming Gong, Yajing Liu, Tongliang Liu, Kun Zhang, and
  Dacheng Tao.
\newblock Deep domain generalization via conditional invariant adversarial
  networks.
\newblock In \emph{Proceedings of the European Conference on Computer Vision
  (ECCV)}, pages 624--639, 2018.

\bibitem[Jin et~al.(2020)Jin, Lan, Zeng, and Chen]{jin2020alignrestore}
Xin Jin, Cuiling Lan, Wenjun Zeng, and Zhibo Chen.
\newblock Feature alignment and restoration for domain generalization and
  adaptation.
\newblock \emph{arXiv preprint arXiv:2006.12009}, 2020.

\bibitem[Fujimoto et~al.(2018)Fujimoto, Hoof, and Meger]{fujimoto2018td3}
Scott Fujimoto, Herke Hoof, and David Meger.
\newblock Addressing function approximation error in actor-critic methods.
\newblock In \emph{International Conference on Machine Learning}, pages
  1587--1596, 2018.

\bibitem[Gretton et~al.(2008)Gretton, Borgwardt, Rasch, Sch{\"o}lkopf, and
  Smola]{gretton2008mmd}
Arthur Gretton, Karsten~M Borgwardt, Malte~J Rasch, Bernhard Sch{\"o}lkopf, and
  Alexander Smola.
\newblock A kernel method for the two-sample problem.
\newblock \emph{Journal of Machine Learning Research}, 1:\penalty0 1--10, 2008.

\bibitem[Zhao and Meng(2015)]{zhao2015fastmmd}
Ji~Zhao and Deyu Meng.
\newblock Fastmmd: Ensemble of circular discrepancy for efficient two-sample
  test.
\newblock \emph{Neural computation}, 27\penalty0 (6):\penalty0 1345--1372,
  2015.

\bibitem[Chen et~al.(2020)Chen, Kornblith, Norouzi, and Hinton]{chen2020simple}
Ting Chen, Simon Kornblith, Mohammad Norouzi, and Geoffrey Hinton.
\newblock A simple framework for contrastive learning of visual
  representations.
\newblock In \emph{International conference on machine learning}, pages
  1597--1607. PMLR, 2020.

\bibitem[Laskin et~al.(2020{\natexlab{a}})Laskin, Srinivas, and
  Abbeel]{laskin2020curl}
Michael Laskin, Aravind Srinivas, and Pieter Abbeel.
\newblock Curl: Contrastive unsupervised representations for reinforcement
  learning.
\newblock In \emph{International Conference on Machine Learning}, pages
  5639--5650. PMLR, 2020{\natexlab{a}}.

\bibitem[Nair et~al.(2018)Nair, Pong, Dalal, Bahl, Lin, and
  Levine]{nair2018rig}
Ashvin Nair, Vitchyr Pong, Murtaza Dalal, Shikhar Bahl, Steven Lin, and Sergey
  Levine.
\newblock Visual reinforcement learning with imagined goals.
\newblock In \emph{Proceedings of the 32nd International Conference on Neural
  Information Processing Systems}, pages 9209--9220, 2018.

\bibitem[Stone et~al.(2021)Stone, Ramirez, Konolige, and
  Jonschkowski]{stone2021distractingbench}
Austin Stone, Oscar Ramirez, Kurt Konolige, and Rico Jonschkowski.
\newblock The distracting control suite--a challenging benchmark for
  reinforcement learning from pixels.
\newblock \emph{arXiv preprint arXiv:2101.02722}, 2021.

\bibitem[Laskin et~al.(2020{\natexlab{b}})Laskin, Lee, Stooke, Pinto, Abbeel,
  and Srinivas]{laskin2020rad}
Michael Laskin, Kimin Lee, Adam Stooke, Lerrel Pinto, Pieter Abbeel, and
  Aravind Srinivas.
\newblock Reinforcement learning with augmented data.
\newblock \emph{arXiv preprint arXiv:2004.14990}, 2020{\natexlab{b}}.

\bibitem[Wang et~al.(2020)Wang, Narasimhan, Lin, Okorn, and Held]{wang2020roll}
Yufei Wang, Gautham~Narayan Narasimhan, Xingyu Lin, Brian Okorn, and David
  Held.
\newblock Roll: Visual self-supervised reinforcement learning with object
  reasoning.
\newblock \emph{arXiv preprint arXiv:2011.06777}, 2020.

\bibitem[Lee et~al.(2020)Lee, Eysenbach, Salakhutdinov, Finn,
  et~al.]{lee2020weakly}
Lisa Lee, Benjamin Eysenbach, Ruslan Salakhutdinov, Chelsea Finn, et~al.
\newblock Weakly-supervised reinforcement learning for controllable behavior.
\newblock \emph{arXiv preprint arXiv:2004.02860}, 2020.

\bibitem[Ren et~al.(2019)Ren, Dong, Zhou, Liu, and Peng]{ren2019exploration}
Zhizhou Ren, Kefan Dong, Yuan Zhou, Qiang Liu, and Jian Peng.
\newblock Exploration via hindsight goal generation.
\newblock \emph{Advances in Neural Information Processing Systems},
  32:\penalty0 13485--13496, 2019.

\bibitem[Ghosh et~al.(2019)Ghosh, Gupta, Reddy, Fu, Devin, Eysenbach, and
  Levine]{ghosh2019isl}
Dibya Ghosh, Abhishek Gupta, Ashwin Reddy, Justin Fu, Coline Devin, Benjamin
  Eysenbach, and Sergey Levine.
\newblock Learning to reach goals via iterated supervised learning.
\newblock \emph{arXiv e-prints}, pages arXiv--1912, 2019.

\bibitem[Florensa et~al.(2018)Florensa, Held, Geng, and
  Abbeel]{florensa2018goalgan}
Carlos Florensa, David Held, Xinyang Geng, and Pieter Abbeel.
\newblock Automatic goal generation for reinforcement learning agents.
\newblock In \emph{International conference on machine learning}, pages
  1515--1528. PMLR, 2018.

\bibitem[Nair and Finn(2019)]{nair2019hierforsignt}
Suraj Nair and Chelsea Finn.
\newblock Hierarchical foresight: Self-supervised learning of long-horizon
  tasks via visual subgoal generation.
\newblock In \emph{International Conference on Learning Representations}, 2019.

\bibitem[Pitis et~al.(2020)Pitis, Chan, Zhao, Stadie, and
  Ba]{pitis2020maximumentro}
Silviu Pitis, Harris Chan, Stephen Zhao, Bradly Stadie, and Jimmy Ba.
\newblock Maximum entropy gain exploration for long horizon multi-goal
  reinforcement learning.
\newblock In \emph{International Conference on Machine Learning}, pages
  7750--7761. PMLR, 2020.

\bibitem[Nair et~al.(2020)Nair, Savarese, and Finn]{nair2020goalmodel}
Suraj Nair, Silvio Savarese, and Chelsea Finn.
\newblock Goal-aware prediction: Learning to model what matters.
\newblock In \emph{International Conference on Machine Learning}, pages
  7207--7219. PMLR, 2020.

\bibitem[Li et~al.()Li, Zheng, Wang, and Zhang]{li2021slow}
Siyuan Li, Lulu Zheng, Jianhao Wang, and Chongjie Zhang.
\newblock Learning subgoal representation with slow dynamics.

\bibitem[Nachum et~al.(2018)Nachum, Gu, Lee, and Levine]{nachum2018nor}
Ofir Nachum, Shixiang Gu, Honglak Lee, and Sergey Levine.
\newblock Near-optimal representation learning for hierarchical reinforcement
  learning.
\newblock In \emph{International Conference on Learning Representations}, 2018.

\bibitem[Zhang et~al.(2020{\natexlab{c}})Zhang, Yang, and
  Stadie]{zhang2020world}
Lunjun Zhang, Ge~Yang, and Bradly~C Stadie.
\newblock World model as a graph: Learning latent landmarks for planning.
\newblock \emph{arXiv preprint arXiv:2011.12491}, 2020{\natexlab{c}}.

\bibitem[Hou et~al.(2020)Hou, Fei, Deng, and Xu]{hou2020hac}
Zhimin Hou, Jiajun Fei, Yuelin Deng, and Jing Xu.
\newblock Data-efficient hierarchical reinforcement learning for robotic
  assembly control applications.
\newblock \emph{IEEE Transactions on Industrial Electronics}, 2020.

\bibitem[Gelada et~al.(2019)Gelada, Kumar, Buckman, Nachum, and
  Bellemare]{gelada2019deepmdp}
Carles Gelada, Saurabh Kumar, Jacob Buckman, Ofir Nachum, and Marc~G Bellemare.
\newblock Deepmdp: Learning continuous latent space models for representation
  learning.
\newblock In \emph{International Conference on Machine Learning}, pages
  2170--2179. PMLR, 2019.

\bibitem[Agarwal et~al.(2021)Agarwal, Machado, Castro, and
  Bellemare]{agarwal2021contrastivemdp}
Rishabh Agarwal, Marlos~C Machado, Pablo~Samuel Castro, and Marc~G Bellemare.
\newblock Contrastive behavioral similarity embeddings for generalization in
  reinforcement learning.
\newblock \emph{arXiv preprint arXiv:2101.05265}, 2021.

\bibitem[Kostrikov et~al.(2020)Kostrikov, Yarats, and
  Fergus]{kostrikov2020image}
Ilya Kostrikov, Denis Yarats, and Rob Fergus.
\newblock Image augmentation is all you need: Regularizing deep reinforcement
  learning from pixels.
\newblock \emph{arXiv preprint arXiv:2004.13649}, 2020.

\bibitem[Hansen et~al.(2020)Hansen, Sun, Abbeel, Efros, Pinto, and
  Wang]{hansen2020selfdeploy}
Nicklas Hansen, Yu~Sun, Pieter Abbeel, Alexei~A Efros, Lerrel Pinto, and
  Xiaolong Wang.
\newblock Self-supervised policy adaptation during deployment.
\newblock \emph{arXiv preprint arXiv:2007.04309}, 2020.

\bibitem[Bodnar et~al.(2020)Bodnar, Hausman, Dulac-Arnold, and
  Jonschkowski]{bodnar2020geometric}
Cristian Bodnar, Karol Hausman, Gabriel Dulac-Arnold, and Rico Jonschkowski.
\newblock A geometric perspective on self-supervised policy adaptation.
\newblock \emph{arXiv preprint arXiv:2011.07318}, 2020.

\bibitem[Schulman et~al.(2015)Schulman, Levine, Abbeel, Jordan, and
  Moritz]{schulman2015trpo}
John Schulman, Sergey Levine, Pieter Abbeel, Michael Jordan, and Philipp
  Moritz.
\newblock Trust region policy optimization.
\newblock In \emph{International conference on machine learning}, pages
  1889--1897. PMLR, 2015.

\bibitem[Watkins and Dayan(1992)]{watkins1992qlearning}
Christopher~JCH Watkins and Peter Dayan.
\newblock Q-learning.
\newblock \emph{Machine learning}, 8\penalty0 (3-4):\penalty0 279--292, 1992.

\end{thebibliography}

\clearpage

\appendix

\section{Notation} \label{apx:notation}

\renewcommand*{\arraystretch}{1.22}{ 
{\setlength{\tabcolsep}{0.25em} 
\begin{longtable}{ll}
    \caption{Description for symbols} 
    \label{tab:symbol_description} \\
    
    \toprule
    Symbol & Description \\
    \midrule 
    \endfirsthead
    
    \multicolumn{2}{c}{\tablename\ \thetable{} -- continued from previous page} \\
    \toprule
    Symbol & Description \\
    \midrule
    \endhead
    
    \toprule
    \multicolumn{2}{c}{{Continued on next page}} \\ 
    \midrule
    \endfoot
    
    \bottomrule
    \endlastfoot
    
    $M^{\gE}$ & A family of Goal-conditioned Block MDPs \\
    $e$ & Environment index \\
    $\gS$ & Shared state space among environments $e \in \gE$ \\
    $\gA$ & Shared action space among environments $e \in \gE$ \\
    $\gX^e$ & Specific observation space for environment $e$ \\
    $\gT^e$ & Specific transition dynamic for environment $e$ \\
    $\gG$ & Shared goal space among environments $e \in \gE$ \\
    $\gamma$ & Shared discount factor among environments $e \in \gE$ \\
    \multirow{2}{*}{$b^e$} & Environmental factor for environment $e$ \\
    & (e.g. background) \\
    \multirow{2}{*}{$\gB^e$} & Specific environmental factor space for environment $e$ \\
    & (i.e. video backgrounds are allowed) \\
    \multirow{2}{*}{$x^e_t = x^e(s_t, b^e_t)$} & Observation determined by state $s$ and environmental \\
    & factor $b^e$ for environment $e$ at timestep $t$ \\
    $p(s_{t + 1} | s_t, a_t)$ & State transition shared among environments \\
    $q^e(b^e_{t + 1} | b^e_t)$ & Environmental factor transition for environment $e$ \\
    $\gX^{\gE} = \cup_{e \in \gE} \gX^e$ & Joint set of observation spaces \\
    $\pi(a | x^e, g)$ & Goal-conditioned policy shared among environments \\
    $J(\pi)$ & Objective function for policy $\pi$ \\
    $J^e(\pi)$ & Objective function for policy $\pi$ in environment $e$ \\
    \multirow{2}{*}{$p_{\pi}^e(s_t=g|g)$} & Probability of achieving goal $g$ under policy $\pi(\cdot|x^e, g)$ \\
    & at timestep $t$ in environment $e$ \\
    \multirow{2}{*}{$\pi_G(\cdot|x^e, g)$} & Optimal policies which are invariant over all \\
    & environments \\
    $\{e_i\}_{i=1}^N = \gE_{\text{train}}$ & Training environments \\
    $\rho(x, g) = \rho(g) \rho(x | g)$ & Joint distributions of goals and observations \\
    \multirow{2}{*}{$\rho_{\pi}^e(x^e|g)$} & Occupancy measure of $x^e$ in environment $e$ under \\
    & policy $\pi(\cdot|x^e, g)$ \\
    $\rho_{\pi}^e(x^e)$ & Marginal distribution of $\rho_{\pi}^e(x^e, g)$ over goals \\
    $\epsilon^{\rho(x,g)}(\pi_1 \parallel \pi_2)$ & Averaged Total Variation between policy $\pi_1$ and $\pi_2$ \\
    $\Pi$ & Policy class (i.e. space for all possible policies) \\
    \multirow{2}{*}{$d_{\Pi \Delta \Pi}(\rho(x, g), \rho(x, g)')$} & $\Pi \Delta \Pi$-divergence of two joint distributions $\rho(x, g)$ \\
    & and $\rho(x, g)'$ in terms of the policy class $\Pi$ \\
    \multirow{3}{*}{$\epsilon^{e_i}(\pi_1 \parallel \pi_2) = \epsilon^{\rho_{\pi}^{e_i}(x^{e_i}, g)}(\pi_1 \parallel \pi_2)$} & Total Variation between policy $\pi_1$ and $\pi_2$ averaged \\
    & over joint occupancy measure under policy $\pi$ in \\
    & training environment $e_i$ \\
    \multirow{3}{*}{$\epsilon^t(\pi_1 \parallel \pi_2) = \epsilon^{\rho_{\pi_G}^t(x^t, g)}(\pi_1 \parallel \pi_2)$} & Total Variation between policy $\pi_1$ and $\pi_2$ averaged \\
    & over joint occupancy measure under policy $\pi_G$ in \\
    & testing environment $t$ \\
    \multirow{3}{*}{$\pi^*$} & The closest policy for training environments in policy \\
    & class $\Pi$ w.r.t.optimal invariant policy $\pi_G$ measured \\
    & by averaged Total Variation \\
    \multirow{3}{*}{$\delta$} & Maximum $\Pi \Delta \Pi$-divergence between occupancy \\ 
    & measure for two different training environments under \\
    & given policy $\pi$ \\
    \multirow{2}{*}{$\lambda$} & Performance of $\pi^*$ in both training and testing \\ 
    & environments in terms of average TV distance \\
    \multirow{2}{*}{$B$} & Characteristic set of joint distributions determined by \\
    & $\gE_{\text{train}}$ and policy class $\Pi$ \\
    $\Phi$ & Observation encoder \\
    $z^e(s) = \Phi(x^e(s))$ & Latent representation of observation $x^e$ with state $s$ \\
    $\Pi_{\Phi} = \{w \circ (\Phi(x), g)\}, \forall w\}$ & Policy class induced by encoder $\Phi$ with any function $w$ \\
    \multirow{3}{*}{$\tilde{\rho}(x, g)$} & The closest occupancy measure in characteristic set $B$ \\
    & w.r.t. occupancy measure in testing environment under \\
    & $\Pi \Delta \Pi$-divergence \\
    $\{s^e_t(a_{0:t})\}_{t=0}^T$ & Set of states along trajectory $\{x_0^e, a_0, x_1^e, a_1, \ldots, x^e_T\}$ \\
    \multirow{2}{*}{$\{x^{e_i}_t(a_{0:t})\}$} & Aligned observations in environment $e_i$ for action \\
    & sequence $\{a_0, \ldots, a_{t - 1}\}$ (\texttt{numpy} style indexing)\\
    $\gR_{\text{align}}$ & Replay buffer for aligned transitions \\
    \multirow{2}{*}{$\gB_{\text{align}} = \{x^{e_i}(s^{e_i}_t(a_{0:t}^b)), \forall e_i \in \gE_{\text{train}}\}_{b=1}^B$} & Batch of aligned observations from all the training \\
    & environments \\
    $L^{\text{MMD}}(\Phi)$ & MMD loss for encoder $\Phi$ \\
    $\psi(z)$ & Random expansion function for latent representation $z$ \\
    $L^{\text{DIFF}}(\Phi)$ & Difference loss for encoder $\Phi$ \\
    $\gR^e$ & Replay buffer for transitions from environment $e$ \\
    $L^{\text{PA}}$ & Perfect alignement loss \\
    $L^{\text{RECON}}$ & Reconstruction loss \\
    $\beta$ & KL divergence coefficient \\
    $\alpha_{\text{MMD}}$ & MMD loss coefficient \\
    $\alpha_{\text{DIFF}}$ & Difference loss coefficient \\
    $\text{Err}(\Phi)$ & Latent error rate for encoder $\Phi$ \\
\end{longtable}
}
}

\clearpage

\section{Algorithm} \label{apx:algo}

The main difference between the PA-SF and Skew-Fit are (i) separate replay buffer for each training environments $\gR = \{ \gR^e, e \in \gE_{\text{train}} \}$, (ii) an additional aligned buffer for the aligned data $\gR_{\text{align}} = \{ \gR_{\text{align}}^e, e \in \gE_{\text{train}} \}$ and a corresponding aligned sampling procedure, (iii) VAE training uses \Eqref{equ:pasfobj} with mini-batches from both replay buffer and aligned buffer. The overall algorithm is described in Algorithm \ref{algo:general} and implementation details are listed in \Apxref{apx:imp}.

\begin{algorithm}[h]
    \begin{algorithmic}[1]
        \REQUIRE $\beta$-VAE decoder, encoder $q_{\phi}$, goal-conditioned policy $\pi_{\theta}$, goal-conditioned value function $Q_w$, skew parameter $\alpha$, VAE training schedule, training environments $\mathcal{E}_{\text{train}}$, replay buffer $\gR = \{ \gR^e, e \in \gE_{\text{train}} \}$, aligned buffer $\gR_{\text{align}} = \{ \gR_{\text{align}}^e, e \in \gE_{\text{train}} \}$, coefficients in \Eqref{equ:pasfobj}.
        \FOR{$m = 0, \ldots, M - 1$ episodes}
            \FOR[Exploration Rollout]{$e = 0, \dots, N - 1$ training environments}
                \STATE Sample goal observation $x^e(g) \sim p_{\text{skewed}}^{e, m}$ and encode as $z^e(g) = q_{\phi}(x^e(g))$.
                \STATE Sample initial observation $x_0^e$ from the environment $e$.
                \FOR{$t = 0, \ldots, H - 1$ steps}
                    \STATE Get action $a_t \sim \pi_{\theta}(q_{\phi}(x_t^e), g)$.
                    \STATE Get next state $x_{t+1}^e \sim p(\cdot \mid x_t^e, a_t)$.
                    \STATE Store $(x_t^e, a_t, x_{t+1}^e, x^e(g))$ into replay buffer $\mathcal{R}^e$.
                \ENDFOR
            \ENDFOR
            \STATE Sample action sequences $\{a_{0:T}\}$ by executing the policy \\ in a random training environment. \COMMENT{Aligned Sampling}
            \FOR{$e = 0, \dots, N - 1$ training environments}
                \STATE Sample initial state $x^e_0 (s_0^e)$ from the environment $e$.
                \STATE Rollout action sequence $\{a_{0:T}\}$ to get $\{x^e_t(s^e_t(a_{0:t})) \}_{t=0}^T$.
            \ENDFOR
            \STATE Store $\{x^e_t(s^e_t(a_{0:t})) \}_{t=0}^T$ in aligned buffer $\gR_{\text{align}}^e$ indexed by $a_{0:t}$ for $e \in \gE_{\text{train}}$.
            \FOR[Policy Gradient]{$i = 0, \dots, I - 1$ training iterations}
                \STATE Sample transition $(x^e_{t'}, a_{t'}, x^e_{t' + 1}, z^e(g)) \sim \mathcal{R}^e$ for all $e \in \gE_{\text{train}}$.
                \STATE Encode $z^e_{t'} = q_\phi(x^e_{t'}), z^e_{t' + 1} = q_\phi(x^e_{t' + 1})$.
                \STATE (Probability $0.5$) replace $z^e(g)$ with $q_{\phi}(x'(g))$ where $x'(g) \sim p_{\text{skewed}}^{e, m}$.
                \STATE Compute new reward $r = -||z_{t' + 1}^e - z^e(g)||_2$.
                \STATE Update $\pi_{\theta}$ and $Q_{w}$ via SAC on $(z_{t'}^e, a_{t'}, z_{t' + 1}^e, z^e(g), r)$.
            \ENDFOR
            \FOR[Hindsight Relabeling]{$t=0,...,H -1$ steps}
                \FOR{$j = 0, ..., J - 1$ steps}
                    \STATE Sample future state $x_{h_j}^e$, $t < h_j \leq H-1$ for all $e \in \gE_{\text{train}}$.
                    \STATE Store $(x_t^e, a_t, x_{t+1}^e, q_\phi \left( x_{h_j}^e) \right)$ into $\mathcal{R}^e$.
                \ENDFOR
            \ENDFOR
            \STATE Construct skewed replay buffer distribution $p_{\text{skewed}}^{e, m + 1}$ using data \\ from $\mathcal{R}^e$ for all $e \in \mathcal{E}_{\text{train}}$. \COMMENT{Skewing Replay Buffers}
            \STATE Fine-tune $\beta$-VAE on $ \{ x^{e} \}_{b=1}^B \sim p_{\text{skewed}}^{e, m + 1}$ and $\{ x^{e}(s^{e}_t(a^b_{0:t})) \}_{b=1}^B \sim \mathcal{R}_{\text{aligned}}^e$ \\ for all $e \in \gE_{\text{train}}$ according to the VAE training schedule and via \Eqref{equ:pasfobj}. \COMMENT{VAE Training}
        \ENDFOR
    \end{algorithmic}
    \caption{Perfect Alignment for Skew-Fit (PA-SF).}
    \label{algo:general}
\end{algorithm}

\clearpage

\newtheorem{remark}{Remark}

\section{Proofs and Discussions} \label{apx:proof}

In this section, we provide detailed proofs and statements omitted in the main text. In addition, we also discuss the assumptions we make in the analysis in detail.

\subsection{Illustration of Different MDP Problems}

Here, we illustrate different graphical models of related MDPs including Block MDPs (Figure \ref{fig:blockmdp}), Goal-conditioned MDPs (Figure \ref{fig:goalmdp}), and ours Goal-conditioned Block MDPs (GBMDP) (Figure \ref{fig:gbmdp}). We use the indicator funtion in the Goal-conditioned and GBMDP settings to emphasize that the reward is sparse. In practice, the goal $g$ may only be indirectly observed as $x^e(g)$, such as future state in pixel space for a particular domain. 

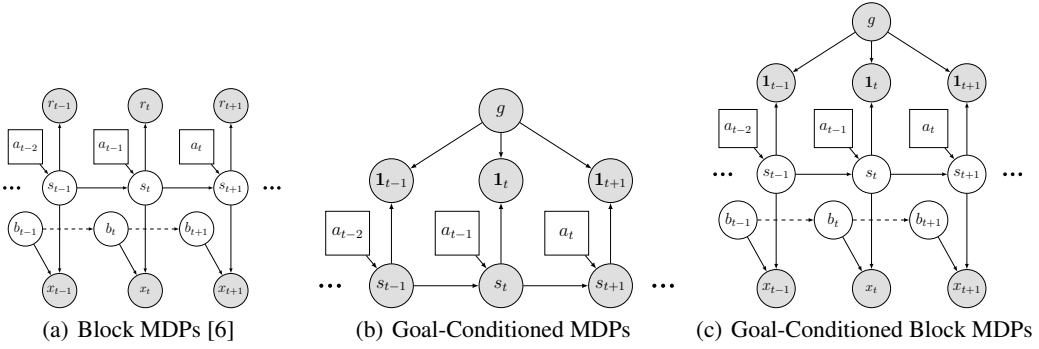
\begin{figure}[h]
    \subfigure[Block MDPs \cite{zhang2020inblock}]{ 
    \label{fig:blockmdp} 
    \resizebox{!}{1.15in}{
        \begin{tikzpicture}
            \node [circle,draw=black,fill=white,inner sep=0pt,minimum size=1.0cm] (s_{t - 1}) at (0,0) {\large $s_{t - 1}$};
            \node [circle,draw=black,fill=white,inner sep=0pt,minimum size=1.0cm] (b_{t - 1}) at (-1.0,-1.2) {\large $b_{t - 1}$};
            \node [rectangle,draw=black,fill=white,inner sep=0pt,minimum size=1.0cm] (a_{t - 2}) at (-1.0,1.2) {\large $a_{t - 2}$};
            \node [obs,minimum size=1.0cm] (x_{t - 1}) at (0,-3.0) {\large $x_{t - 1}$};
            \node [obs,minimum size=1.0cm] (r_{t - 1}) at (0,2.4) {\large $r_{t - 1}$};
        
            \node [circle,draw=black,fill=white,inner sep=0pt,minimum size=1.0cm] (s_t) at (2.5,0) {\large $s_t$};
            \node [circle,draw=black,fill=white,inner sep=0pt,minimum size=1.0cm] (b_t) at (1.5,-1.2) {\large $b_t$};
            \node [rectangle,draw=black,fill=white,inner sep=0pt,minimum size=1.0cm] (a_{t - 1}) at (1.5,1.2) {\large $a_{t - 1}$};
            \node [obs,minimum size=1.0cm] (x_t) at (2.5,-3.0) {$x_t$};
            \node [obs,minimum size=1.0cm] (r_{t}) at (2.5,2.4) {\large $r_{t}$};
        
            \node [circle,draw=black,fill=white,inner sep=0pt,minimum size=1.0cm] (s_{t + 1}) at (5,0) {\large $s_{t + 1}$};
            \node [circle,draw=black,fill=white,inner sep=0pt,minimum size=1.0cm] (b_{t + 1}) at (4,-1.2) {\large $b_{t + 1}$};
            \node [rectangle,draw=black,fill=white,inner sep=0pt,minimum size=1.0cm] (a_t) at (4,1.2) {\large $a_t$};
            \node [obs,minimum size=1.0cm] (x_{t + 1}) at (5,-3.0) {\large $x_{t + 1}$};
            \node [obs,minimum size=1.0cm] (r_{t + 1}) at (5,2.4) {\large $r_{t + 1}$};
            
            \path [draw,->] (a_{t - 2}) edge (s_{t - 1});
            \path [draw,->] (s_{t - 1}) edge (x_{t - 1});
            \path [draw,->] (b_{t - 1}) edge (x_{t - 1});
            \path [draw,->] (s_{t - 1}) edge (r_{t - 1});
            
            \path [draw,->] (a_{t - 1}) edge (s_t);
            \path [draw,->] (s_t) edge (x_t);
            \path [draw,->] (b_t) edge (x_t);
            \path [draw,->] (s_t) edge (r_t);
            
            \path [draw,->] (a_t) edge (s_{t + 1});
            \path [draw,->] (s_{t + 1}) edge (x_{t + 1});
            \path [draw,->] (b_{t + 1}) edge (x_{t + 1});
            \path [draw,->] (s_{t + 1}) edge (r_{t + 1});
            
            \path [draw,->] (s_{t - 1}) edge (s_t);
            \path [dashed,->] (b_{t - 1}) edge (b_t);
            
            \path [draw,->] (s_t) edge (s_{t + 1});
            \path [dashed,->] (b_t) edge (b_{t + 1});
            
            \node[mark size=1pt,color=black] at (-1.2,0) {\pgfuseplotmark{*}};
            \node[mark size=1pt,color=black] at (-1.4,0) {\pgfuseplotmark{*}};
            \node[mark size=1pt,color=black] at (-1.6,0) {\pgfuseplotmark{*}};
            
            \node[mark size=1pt,color=black] at (6.0,0) {\pgfuseplotmark{*}};
            \node[mark size=1pt,color=black] at (6.2,0) {\pgfuseplotmark{*}};
            \node[mark size=1pt,color=black] at (6.4,0) {\pgfuseplotmark{*}};
        \end{tikzpicture}
        }
    }
    \subfigure[Goal-Conditioned MDPs]{ 
        \label{fig:goalmdp} 
        \resizebox{!}{1.15in}{
        \begin{tikzpicture}
            \node [obs,minimum size=1.0cm] (s_{t - 1}) at (0,0) {\large $s_{t - 1}$};
            \node [rectangle,draw=black,fill=white,inner sep=0pt,minimum size=1.0cm] (a_{t - 2}) at (-1.0,1.2) {\large $a_{t - 2}$};
            \node [obs,minimum size=1.0cm] (r_{t - 1}) at (0,2.4) {\large $\mathbf{1}_{t - 1}$};
        
            \node [obs,minimum size=1.0cm] (s_t) at (2.5,0) {\large $s_t$};
            \node [rectangle,draw=black,fill=white,inner sep=0pt,minimum size=1.0cm] (a_{t - 1}) at (1.5,1.2) {\large $a_{t - 1}$};
            \node [obs,minimum size=1.0cm] (r_t) at (2.5,2.4) {\large $\mathbf{1}_{t}$};
            \node [obs,minimum size=1.0cm] (g) at (2.5,4) {\large $g$};
        
            \node [obs,minimum size=1.0cm] (s_{t + 1}) at (5,0) {\large $s_{t + 1}$};
            \node [rectangle,draw=black,fill=white,inner sep=0pt,minimum size=1.0cm] (a_t) at (4,1.2) {\large $a_t$};
            \node [obs,minimum size=1.0cm] (r_{t + 1}) at (5,2.4) {\large $\mathbf{1}_{t + 1}$};
            
            \path [draw,->] (a_{t - 2}) edge (s_{t - 1});
            \path [draw,->] (s_{t - 1}) edge (r_{t - 1});
            
            \path [draw,->] (a_{t - 1}) edge (s_t);
            \path [draw,->] (s_t) edge (r_t);
            
            \path [draw,->] (a_t) edge (s_{t + 1});
            \path [draw,->] (s_{t + 1}) edge (r_{t + 1});
            
            \path [draw,->] (s_{t - 1}) edge (s_t);
            
            \path [draw,->] (s_t) edge (s_{t + 1});
            
            \path [draw,->] (g) edge (r_{t - 1});
            \path [draw,->] (g) edge (r_t);
            \path [draw,->] (g) edge (r_{t + 1});
            
            \node[mark size=1pt,color=black] at (-1.2,0) {\pgfuseplotmark{*}};
            \node[mark size=1pt,color=black] at (-1.4,0) {\pgfuseplotmark{*}};
            \node[mark size=1pt,color=black] at (-1.6,0) {\pgfuseplotmark{*}};
            
            \node[mark size=1pt,color=black] at (6.0,0) {\pgfuseplotmark{*}};
            \node[mark size=1pt,color=black] at (6.2,0) {\pgfuseplotmark{*}};
            \node[mark size=1pt,color=black] at (6.4,0) {\pgfuseplotmark{*}};
        \end{tikzpicture}
        }
    }
    \subfigure[Goal-Conditioned Block MDPs]{ 
        \label{fig:gbmdp} 
        \resizebox{!}{1.6in}{
        \begin{tikzpicture}
            \node [circle,draw=black,fill=white,inner sep=0pt,minimum size=1.0cm] (s_{t - 1}) at (0,0) {\large $s_{t - 1}$};
            \node [circle,draw=black,fill=white,inner sep=0pt,minimum size=1.0cm] (b_{t - 1}) at (-1.0,-1.2) {\large $b_{t - 1}$};
            \node [rectangle,draw=black,fill=white,inner sep=0pt,minimum size=1.0cm] (a_{t - 2}) at (-1.0,1.2) {\large $a_{t - 2}$};
            \node [obs,minimum size=1.0cm] (x_{t - 1}) at (0,-3.0) {\large $x_{t - 1}$};
            \node [obs,minimum size=1.0cm] (r_{t - 1}) at (0,2.4) {\large $\mathbf{1}_{t - 1}$};
        
            \node [circle,draw=black,fill=white,inner sep=0pt,minimum size=1.0cm] (s_t) at (2.5,0) {\large $s_t$};
            \node [circle,draw=black,fill=white,inner sep=0pt,minimum size=1.0cm] (b_t) at (1.5,-1.2) {\large $b_t$};
            \node [rectangle,draw=black,fill=white,inner sep=0pt,minimum size=1.0cm] (a_{t - 1}) at (1.5,1.2) {\large $a_{t - 1}$};
            \node [obs,minimum size=1.0cm] (x_t) at (2.5,-3.0) {$x_t$};
            \node [obs,minimum size=1.0cm] (r_t) at (2.5,2.4) {\large $\mathbf{1}_t$};
            \node [obs,minimum size=1.0cm] (g) at (2.5,4) {\large $g$};
        
            \node [circle,draw=black,fill=white,inner sep=0pt,minimum size=1.0cm] (s_{t + 1}) at (5,0) {\large $s_{t + 1}$};
            \node [circle,draw=black,fill=white,inner sep=0pt,minimum size=1.0cm] (b_{t + 1}) at (4,-1.2) {\large $b_{t + 1}$};
            \node [rectangle,draw=black,fill=white,inner sep=0pt,minimum size=1.0cm] (a_t) at (4,1.2) {\large $a_t$};
            \node [obs,minimum size=1.0cm] (x_{t + 1}) at (5,-3.0) {\large $x_{t + 1}$};
            \node [obs,minimum size=1.0cm] (r_{t + 1}) at (5,2.4) {\large $\mathbf{1}_{t+1}$};
            
            \path [draw,->] (a_{t - 2}) edge (s_{t - 1});
            \path [draw,->] (s_{t - 1}) edge (x_{t - 1});
            \path [draw,->] (b_{t - 1}) edge (x_{t - 1});
            \path [draw,->] (s_{t - 1}) edge (r_{t - 1});
            
            \path [draw,->] (a_{t - 1}) edge (s_t);
            \path [draw,->] (s_t) edge (x_t);
            \path [draw,->] (b_t) edge (x_t);
            \path [draw,->] (s_t) edge (r_t);
            
            \path [draw,->] (a_t) edge (s_{t + 1});
            \path [draw,->] (s_{t + 1}) edge (x_{t + 1});
            \path [draw,->] (b_{t + 1}) edge (x_{t + 1});
            \path [draw,->] (s_{t + 1}) edge (r_{t + 1});
            
            \path [draw,->] (s_{t - 1}) edge (s_t);
            \path [dashed,->] (b_{t - 1}) edge (b_t);
            
            \path [draw,->] (s_t) edge (s_{t + 1});
            \path [dashed,->] (b_t) edge (b_{t + 1});
            
            \path [draw,->] (g) edge (r_{t - 1});
            \path [draw,->] (g) edge (r_t);
            \path [draw,->] (g) edge (r_{t + 1});
            
            \node[mark size=1pt,color=black] at (-1.2,0) {\pgfuseplotmark{*}};
            \node[mark size=1pt,color=black] at (-1.4,0) {\pgfuseplotmark{*}};
            \node[mark size=1pt,color=black] at (-1.6,0) {\pgfuseplotmark{*}};
            
            \node[mark size=1pt,color=black] at (6.0,0) {\pgfuseplotmark{*}};
            \node[mark size=1pt,color=black] at (6.2,0) {\pgfuseplotmark{*}};
            \node[mark size=1pt,color=black] at (6.4,0) {\pgfuseplotmark{*}};
        \end{tikzpicture}
        }
    }
    \caption{Comparison of graphical models of \textbf{(a)} a Block MDP \cite{du2019block,zhang2020inblock}, \textbf{(b)} a Goal-Conditioned MDP \cite{kaelbling1993learning,schaul2015universal,andry2017her}, and \textbf{(c)} our proposed Goal-Conditioned Block MDP. The agent takes in the goal $g$ and observation $x_t$, which is produced by the domain invariant state $s_t$ and environmental state $b_t$, and acts with action $a_t$. $\mathbf{1}_t$ denotes the indicator function on whether the inputs are the same state. Note that $b_t$ may have temporal dependence indicated by the dashed edge. } 
\end{figure}

\subsection{Proof of \Propref{prop:gbdg}} \label{apx:gbdg_proof}

Recall that \Propref{prop:gbdg} bounds the generalization performance by 4 terms: (1) average training environments' performance, (2) optimality of the policy class, (3) $d_{\Pi \Delta \Pi}$ over all training environments, and (4) the discrepancy measure between training environments and the target environment.

We begin our analysis by proving the following two Lemmas. For simplicity, we denote $p_{\pi}^{\Delta,e}(s|g)$ as the \textit{discounted state density} as follows.
\begin{align*}
    p_{\pi}^{\Delta, e}(s|g) = (1 - \gamma) \sum_{t=0}^{\infty} \gamma^t p^e_{\pi}(s_t=s|g)
\end{align*}
where $p^e_{\pi}(s_t=s|g)$ is the probability of state $s$ under goal-conditioned policy $\pi(\cdot|x^e, g)$ at step $t$ in domain $e$ (marginalized over the initial state $s_0 \sim p(s_0)$, previous actions $a_i \sim \pi(a_i | x^e_i, g), i = 0, \dots, t - 1$, and previous states $s_i \sim p(s_i | s_{i - 1}, a_{i - 1}), i = 0, \dots, t - 1$).

\begin{lemma} \label{lemma:obj_bound}
$\forall e \in \gE_{\text{all}}$, let $\rho_{\pi}^e(x^e, g)$ denote joint distributions of $g \sim \gG$ and $x^e$ under policy $\pi(\cdot |x^e, g)$, then $\forall \pi_1, \pi_2$, we have
\begin{align*}
    |J^e(\pi_1) - J^e(\pi_2)| \leq \frac{2 \gamma}{1 - \gamma} \E_{\rho_{\pi_1}^e(x^e, g)} \left[D_{TV}(\pi_1(\cdot|x^e, g) \parallel \pi_2(\cdot|x^e, g)) \right]
\end{align*}
\end{lemma}

\begin{proof}
By the definition of $J^e(\pi)$, we have
\begin{align*}
    |J^e(\pi_1) -J^e(\pi_2)| & = |\E_{g \sim \gG}[p_{\pi_1}^{\Delta, e}(g|g) - p_{\pi_2}^{\Delta, e}(g|g)]| \\
    & \leq \E_{g \sim \gG}[|p_{\pi_1}^{\Delta, e}(g|g) - p_{\pi_2}^{\Delta, e}(g|g)|]
\end{align*}
Thus, it suffices to prove $\forall g \in \gG, $
\begin{align*}
    |p_{\pi_1}^{\Delta, e}(g|g) - p_{\pi_2}^{\Delta, e}(g|g)| \leq \frac{2 \gamma}{1 - \gamma} \E_{\rho_{\pi_1}^e(x^e|g)} \left[D_{TV}(\pi_1(\cdot|x^e, g) \parallel \pi_2(\cdot|x^e, g)) \right]
\end{align*}
First, we consider $|p^e_{\pi_1}(s_T=s|g) - p^e_{\pi_2}(s_T=s|g)|$ for some fixed step $T$. Denote $\{\pi_1 < t, \pi_2 \geq t \}$ as another policy which imitates policy $\pi_1$ for first $t$ steps and then imitates $\pi_2$ for the rest. By the telescoping operation, we have $\forall s \in \gS, g \in \gG$
\begin{align*}
|p^e_{\pi_1}(s_T=s|g) - p^e_{\pi_2}(s_T=s|g)| & \leq \sum_{t=0}^{T-1} |p^e_{\pi_1 < t, \pi_2 \geq t}(s_T=s|g) - p^e_{\pi_1 < t+1, \pi_2 \geq t+1}(s_T=s|g)| \\
& = \sum_{t=0}^{T-1} P^e_t(s|\pi_1, \pi_2, g, T)
\end{align*}
where we use $P^e_t(s|\pi_1, \pi_2, g, T)$ to denote each term for brevity.
\begin{align*}
    P^e_t(s|\pi_1, \pi_2, g, T) & = | \int_{s_t, a} p^e_{\pi_1}(s_t|g)\pi_2(a|x^e(s_t), g)p^e_{\pi_2}(s_T=s|s_t, a, g) d s_t d a -\\
    & ~~~~~~ - \int_{s_t, a} p^e_{\pi_1}(s_t|g)\pi_1(a|x^e_t(s_t), g)p^e_{\pi_2}(s_T=s|s_t, a, g) d s_t d a | \\
    & \leq \int_{s_t, a} p^e_{\pi_1}(s_t|g) |\pi_1(a|x^e(s_t), g) - \pi_2(a|x^e(s_t), g)| p^e_{\pi_2}(s_T=s|s_t, a, g) d s_t d a \\
    & \leq \int_{s_t} p^e_{\pi_1}(s_t|g) \left( \int_a |\pi_1(a|x^e(s_t), g) - \pi_2(a|x^e(s_t), g)| d a \right) d s_t \\
    & \leq 2 \int_{s_t} p^e_{\pi_1}(s_t|g) D_{TV}(\pi_1(\cdot|x^e(s_t), g) \parallel \pi_2(\cdot|x^e(s_t), g)) d s_t 
\end{align*}
Here, $p^e_{\pi_2}(s_T=s|s_t, a, g)$ is the probability of achieving state $s$ at step $T$ under policy $\pi_2(\cdot|x^e, g)$ when it takes action $a$ at $s_t$. Noticing that, the upper bound for $P^e_t(s|\pi_1, \pi_2, g, T)$ is not dependent on $T$. Thus, we have $\forall s \in \gS, g \in \gG$,
\begin{align*}
    |p_{\pi_1}^{\Delta, e}(s|g) - p_{\pi_2}^{\Delta, e}(s | g)| \leq & (1 - \gamma) \sum_{T=0}^{\infty} \gamma^T \sum_{t=0}^{T-1} P^e_t(s|\pi_1, \pi_2, g, T) \\
    = & (1 - \gamma) \sum_{t=0}^{\infty} \frac{\gamma^{t+1}}{1 - \gamma} P^e_t(s|\pi_1, \pi_2, g, t+1) \\
    \leq & 2 \sum_{t=0}^{\infty} \gamma^{t+1} \int_{s_t} p^e_{\pi_1}(s_t|g) D_{TV}(\pi_1(\cdot|x^e(s_t), g) \parallel \pi_2(\cdot|x^e(s_t), g)) d s_t \\
    = & 2\gamma \int_{s_t} \sum_{t=0}^{\infty} \gamma^t p^e_{\pi_1}(s_t|g) D_{TV}(\pi_1(\cdot|x^e(s_t), g) \parallel \pi_2(\cdot|x^e(s_t), g)) d s_t \\
    = & \frac{2\gamma}{1 - \gamma} \E_{\rho_{\pi_1}^e(x^e|g)} \left[D_{TV}(\pi_1(\cdot|x^e, g) \parallel \pi_2(\cdot|x^e, g) \right]
\end{align*}
The Lemma holds by averaging over $g \sim \gG$.
\end{proof}

\Lemref{lemma:obj_bound} bounds the objective function between two policies $\pi_1$ and $\pi_2$ with the Total Variation distance. Recall that we use $\epsilon^{\rho(x, g)}(\pi_1 \parallel \pi_2)$ to denote the average $D_{\text{TV}}$ between the two policies under the joint distribution. We refer to $\pi_G$ as some optimal and invariant policy. Then, we have the following Lemma.

\begin{lemma} \label{lemma:da_boud}
For any policy class $\Pi$ and two joint distributions $\rho^s(x, g)$ and $\rho^t(x, g)$, suppose 
\begin{align*}
    \pi^*_{s, t} = \argmin_{\pi' \in \Pi} \epsilon^{\rho^s(x, g)}(\pi' \parallel \pi_G) + \epsilon^{\rho^t(x, g)}(\pi' \parallel \pi_G)
\end{align*}
then we have for any $\pi \in \Pi$
\begin{align*}
    \epsilon^{\rho^t(x, g)}(\pi \parallel \pi_G) \leq \epsilon^{\rho^s(x, g)}(\pi \parallel \pi_G) + \underbrace{\sup_{\pi, \pi' \in \Pi} \left| \epsilon^{\rho^s(x, g)} (\pi \parallel \pi') - \epsilon^{\rho^t(x, g)}(\pi \parallel \pi') \right|}_{d_{\Pi \Delta \Pi}(\rho^s(x, g), \rho^t(x, g))} + \lambda_{s, t}
\end{align*}
where $ \lambda_{s, t} = \epsilon^{\rho^s(x, g)}(\pi^*_{s, t} \parallel \pi_G) + \epsilon^{\rho^t(x, g)}(\pi^*_{s, t} \parallel \pi_G)$.
\end{lemma}

\begin{proof}
Noticing that $D_{\text{TV}}$ is a distance metric that satisfies the triangular inequality and is symmetric. Thus, we have $\forall \pi_1, \pi_2, \pi_3$ and any joint distribution $\rho(x, g)$.
\begin{align*}
    \epsilon^{\rho(x, g)}(\pi_1 \parallel \pi_2) & = \E_{\rho(x, g)}[D_{\text{TV}}(\pi_1(\cdot|x, g) \parallel \pi_2(\cdot|x, g))] \\
    & \leq \E_{\rho(x, g)}[D_{\text{TV}}(\pi_1(\cdot|x, g) \parallel \pi_3(\cdot|x, g)) + D_{\text{TV}}(\pi_3(\cdot|x, g) \parallel \pi_2(\cdot|x, g))] \\
    & =  \epsilon^{\rho(x, g)}(\pi_1 \parallel \pi_3) + \epsilon^{\rho(x, g)}(\pi_3 \parallel \pi_2)
\end{align*}
Based on this property, we have
\begin{align*}
    \epsilon^{\rho^t(x, g)}(\pi \parallel \pi_G) \leq &~\epsilon^{\rho^t(x, g)}(\pi^*_{s, t} \parallel \pi_G) + \epsilon^{\rho^t(x, g)}(\pi \parallel \pi^*_{s, t}) \\
    \leq &~\epsilon^{\rho^t(x, g)}(\pi^*_{s, t} \parallel \pi_G) + \epsilon^{\rho^s(x, g)}(\pi \parallel \pi^*_{s, t}) + \left|\epsilon^{\rho^t(x, g)}(\pi \parallel \pi^*_{s, t}) - \epsilon^{\rho^s(x, g)}(\pi \parallel \pi^*_{s, t}) \right| \\
    \leq &~\epsilon^{\rho^s(x, g)}(\pi \parallel \pi_G) + \epsilon^{\rho^s(x, g)} (\pi^*_{s, t} \parallel \pi_G) + \epsilon^{\rho^t(x, g)} (\pi^*_{s, t} \parallel \pi_G) \\
    &~ + \left|\epsilon^{\rho^t(x, g)}(\pi \parallel \pi^*_{s, t}) - \epsilon^{\rho^s(x, g)}(\pi \parallel \pi^*_{s, t}) \right| \\
    \leq &~\epsilon^{\rho^s(x, g)}(\pi \parallel \pi_G) + d_{\Pi \Delta \Pi}(\rho^s(x, g), \rho^t(x, g)) + \lambda_{s, t}
\end{align*}
\end{proof}

\Lemref{lemma:da_boud} resembles the seminal bound in domain adaptation theory \cite{ben2010datheory}. Then, we extend the generalization bound to the \mdp~ setting. We consider the joint distributions $\rho^{E} = \{\rho^{e_i}(x, g)\}_{i=1}^N$ and define the characteristic set as follows. \footnote{$e_i$ is only used as an index here.}
\begin{definition}
The characteristic set $B(\delta, E|\Pi)$ is a set of joint distributions $\rho(x, g)$ which $\forall e_i \in E$ with $\rho^{e_i}(x, g)$, 
\begin{align*}
    d_{\Pi \Delta \Pi}(\rho(x, g), \rho^{e_i}(x, g)) \leq \delta
\end{align*}
\end{definition}
In other words, the characteristic set is a set of joint distributions that is close to the training environments' distributions in terms of the $d_{\Pi \Delta \Pi}$ divergence. Notice that if we define $\delta_{E} = \max_{e,e' \in E} d_{\Pi \Delta \Pi}(\rho^{e}(x, g), \rho^{e'}(x, g))$, then we have that the convex hull $\Lambda(\{\rho^{e_i}(x, g)\}_{i=1}^N) \subset B(\delta_{E}, E|\Pi)$. Namely, the characteristic set contains all the distributions of convex combinations of training environments' distributions \cite{sicilia2021dgtheory}.

\begin{prop} \label{prop:gbdg_a}
For any policy class $\Pi$, a set of source joint distributions $\rho^E = \{\rho^{e_i}(x, g) \}_{i=1}^N$ and the target distribution $\rho^t(x, g)$, suppose for any unit sum weights $\{\alpha_i\}_{i=1}^N$, i.e., $0 \leq \alpha_i \leq 1, \sum_i \alpha_i = 1$.
\begin{align*}
    \lambda_{\alpha} = \sum_{i=1}^N \alpha_i \epsilon^{e_i}(\pi_{\alpha}^* \parallel \pi_G) + \epsilon^{t}(\pi_{\alpha}^* \parallel \pi_G), ~~~~~ \pi^*_{\alpha} = \argmin_{\pi' \in \Pi} \sum_{i=1}^N \alpha_i \epsilon^{e_i}(\pi' \parallel \pi_G) + \epsilon^{t}(\pi' \parallel \pi_G)
\end{align*}
where $\epsilon^{e_i}$ and $\epsilon^t$ are short of $\epsilon^{\rho^{e_i}}$ and $\epsilon^{\rho^t}$ respectively.
Let 
\begin{align*}
\tilde{\rho}(x, g) = \argmin_{\rho \in B(\delta_E, E|\Pi)} d_{\Pi \Delta \Pi}(\rho(x, g), \rho^t(x, g))
\end{align*}
Then, $\forall \pi \in \Pi$
\begin{align*}
    \epsilon^t(\pi \parallel \pi_G) \leq & \min_{\alpha} \sum_{i=1}^N \alpha_i \epsilon^{e_i}(\pi \parallel \pi_G) +\lambda_{\alpha} + \underbrace{\max_{e, e' \in E} d_{\Pi \Delta \Pi}(\rho^e(x, g), \rho^{e'}(x, g))}_{\delta_E} \\
    & ~~~~~~~~~~~~~~~~~~~~~~~~~~~~~~~~~~~~~~~~~~~~~~~~~~~~ + d_{\Pi \Delta \Pi}(\tilde{\rho}(x, g), \rho^t(x, g)) 
\end{align*}
\end{prop}

\begin{proof}
By \Lemref{lemma:da_boud}, we have $\forall e_i \in E$
\begin{align*}
\epsilon^t(\pi \parallel \pi_G) \leq \epsilon^{e_i}(\pi \parallel \pi_G) + d_{\Pi \Delta \Pi}(\rho^{e_i}(x, g), \rho^t(x, g)) + \lambda_{e_i, t}
\end{align*}
Then, we have for any unit sum weights $\alpha$
\begin{align*}
    \epsilon^t(\pi \parallel \pi_G) & \leq \sum_{i=1}^N \alpha_i \epsilon^{e_i}(\pi \parallel \pi_G) + \sum_{i=1}^N \alpha_i d_{\Pi \Delta \Pi}(\rho^{e_i}(x, g), \rho^t(x, g)) + \sum_{i=1}^N \alpha_i \lambda_{e_i, t} \nonumber \\
\end{align*}
Noticing that
\begin{align*}
    \sum_{i=1}^N \alpha_i \lambda_{e_i, t} & = \sum_{i=1}^N \alpha_i \left( \min_{\pi' \in \Pi} \epsilon^{e_i}(\pi' \parallel \pi_G) + \epsilon^t(\pi' \parallel \pi_G) \right) \\
    & \leq \min_{\pi' \in \Pi} \sum_{i=1}^N \alpha_i \epsilon^{e_i}(\pi' \parallel \pi_G) + \epsilon^t(\pi' \parallel \pi_G) \\
    & = \lambda_{\alpha}
\end{align*}
Thus, we have
\begin{align*}
      \epsilon^t(\pi \parallel \pi_G) \leq \sum_{i=1}^N \alpha_i \epsilon^{e_i}(\pi \parallel \pi_G) + \sum_{i=1}^N \alpha_i d_{\Pi \Delta \Pi}(\rho^{e_i}(x, g), \rho^t(x, g)) + \lambda_{\alpha}
\end{align*}
Since $d_{\Pi \Delta \Pi}$ divergence also follows the triangular inequality \cite{ben2010datheory}, we have
\begin{align*}
    \epsilon^t(\pi \parallel \pi_G) & \leq \sum_{i=1}^N \alpha_i \epsilon^{e_i}(\pi \parallel \pi_G) + \sum_{i=1}^N \alpha_i d_{\Pi \Delta \Pi}(\rho^{e_i}(x, g), \tilde{\rho}(x, g)) + d_{\Pi \Delta \Pi}(\tilde{\rho}(x, g), \rho^t(x, g)) + \lambda_{\alpha} \\
    & \leq \sum_{i=1}^N \alpha_i \epsilon^{e_i}(\pi \parallel \pi_G) + \lambda_{\alpha} + \delta_E + d_{\Pi \Delta \Pi}(\tilde{\rho}(x, g), \rho^t(x, g))
\end{align*}
The proposition holds by taking the minimum over $\alpha$.
\end{proof}

Finally, we are able to provide the formal statements and proofs for \Propref{prop:gbdg} as follows.

\begin{customprop}{1}[Formal] \label{prop:gbdg_f}
For any $\pi \in \Pi$, we consider the occupancy measure $\rho^{\gE_{\text{train}}} = \{\rho_{\pi}^{e_i}(x^{e_i}, g) \}_{i=1}^N$ for training environments and $\rho_{\pi_G}^t(x^t, g)$ for the target environment. For simplicity, we use $\epsilon^{e_i}$ and $\epsilon^t$ as the abbreviations. Considering
\begin{align*}
    \lambda = \frac{1}{N} \sum_{i=1}^N \epsilon^{e_i}(\pi^* \parallel \pi_G) + \epsilon^{t}(\pi^* \parallel \pi_G), ~~~~ \pi^* = \argmin_{\pi' \in \Pi} \sum_{i=1}^N \epsilon^{e_i}(\pi' \parallel \pi_G) 
\end{align*}
and $\delta = \max_{e, e' \in \gE_{\text{train}}} d_{\Pi \Delta \Pi}(\rho^e_{\pi}(x^e, g), \rho^{e'}_{\pi}(x^{e'}, g))$, the characteristic set $B(\delta, \gE_{\text{train}}|\Pi)$. Define
\begin{align*}
    \tilde{\rho}(x, g) = \argmin_{\rho \in B(\delta, \gE_{\text{train}}|\Pi)} d_{\Pi \Delta \Pi}(\rho(x, g), \rho^t_{\pi_G}(x^t, g))
\end{align*}
Then, we have
\begin{align*}
     J^t(\pi_G) - J^t(\pi)\leq & \frac{2 \gamma}{1 - \gamma} \left[\frac{1}{N} \sum_{i=1}^N  \epsilon^{e_i}(\pi \parallel \pi_G) + \lambda + \delta + d_{\Pi \Delta \Pi}(\tilde{\rho}(x, g), \rho_{\pi_G}^t(x^t, g)) \right]
\end{align*}
\end{customprop}

\begin{proof}
By \Lemref{lemma:obj_bound} and \Propref{prop:gbdg_a}, we have
\begin{align*}
    J^{t}(\pi_G) - J^t(\pi) & \leq \frac{2\gamma}{1 - \gamma} \epsilon^t(\pi \parallel \pi_G) \\
    & \leq \frac{2 \gamma}{1 - \gamma} \left[\min_{\alpha} \sum_{i=1}^N \alpha_i \epsilon^{e_i}(\pi \parallel \pi_G) + \lambda_{\alpha} + \delta + d_{\Pi \Delta \Pi}(\tilde{\rho}(x, g), \rho^t_{\pi_G}(x, g)) \right] \\
    & \leq \frac{2 \gamma}{1 - \gamma} \left[\frac{1}{N} \sum_{i=1}^N \epsilon^{e_i}(\pi \parallel \pi_G) + \lambda_{\alpha=\frac{1}{N}} + \delta + d_{\Pi \Delta \Pi}(\tilde{\rho}(x, g), \rho^t_{\pi_G}(x, g)) \right]
\end{align*}
Noticing that $\lambda$ and $\lambda_{\alpha=\frac{1}{N}}$ have different definitions. But, we have
\begin{align*}
    \lambda_{\alpha=\frac{1}{N}} & = \min_{\pi' \in \Pi} \frac{1}{N} \sum_{i=1}^N \epsilon^{e_i}(\pi' \parallel \pi_G) + \epsilon^{t}(\pi' \parallel \pi_G) \\
    & \leq \frac{1}{N} \sum_{i=1}^N \epsilon^{e_i}(\pi^* \parallel \pi_G) + \epsilon^{t}(\pi^* \parallel \pi_G) \\
    & = \lambda
\end{align*}
where $\pi^* = \argmin_{\pi' \in \Pi} \sum_{i=1}^N \epsilon^{e_i}(\pi' \parallel \pi_G)$.
Thus, we have
\begin{align*}
    J^t(\pi_G) - J^t(\pi)\leq & \frac{2 \gamma}{1 - \gamma} \left[\frac{1}{N} \sum_{i=1}^N  \epsilon^{e_i}(\pi \parallel \pi_G) + \lambda + \delta + d_{\Pi \Delta \Pi}(\tilde{\rho}(x, g), \rho_{\pi_G}^t(x^t, g)) \right]
\end{align*}
This completes the proof.
\end{proof}

\begin{remark}
The informal version \Propref{prop:gbdg} omits unessential constant part and the definition of the characteristic set.
\end{remark}

\subsection{Proof of \Propref{prop:padg}} \label{apx:padg_proof}

Here, we provide the formal proof and statement of \Propref{prop:padg}. To begin with, we prove the following Lemma.

\begin{lemma} \label{lemma:occupancy}
For two goal-conditioned policies $\pi, \pi'$ of the Goal-conditioned MDP $\langle  \gS, \gG, \gA, p, \gamma \rangle$, suppose that $\max_{s, g} D_{\text{TV}}(\pi(\cdot|s, g) \parallel \pi'(\cdot|s, g)) \leq \epsilon$, then we have $D_{\text{TV}}(\rho^{\pi}(s, g) \parallel \rho^{\pi'}(s, g)) \leq \frac{\gamma \epsilon}{1 - \gamma} $.
\end{lemma}

\begin{proof}
The proof follows the perturbation theory in Appendix B of \cite{schulman2015trpo}. By the definition of total variation distance, it suffices to prove that $\forall g$, $D_{\text{TV}}(\rho^{\pi}(s|g) \parallel \rho^{\pi'}(s|g)) \leq \frac{\gamma \epsilon}{1 - \gamma}$, where $\rho^{\pi}(s|g)$ denotes the discounted occupancy measure of $s$ under policy $\pi(\cdot|s, g)$. Consequently, in the following notations, we may omit specifying $g$ if unambiguous.

First, we refer $P_{\pi}$ as the state transition matrix under policy $\pi(\cdot|s, g)$, i.e., $(P_{\pi})_{xy} = \int_{a} p(s'=x|s=y, a)\pi(a|s=y, g) d a$ and subsequently, $G_{\pi} = I + \gamma P_{\pi} + \gamma^2 P_{\pi}^2 + \cdots = (I - \gamma P_{\pi})^{-1}$. Then, the transition discrepancy matrix is defined as $\Delta = P_{\pi'} - P_{\pi}$. Observing that
\begin{align*}
    G^{-1}_{\pi} - G^{-1}_{\pi'} & = \gamma (P_{\pi'} - P_{\pi}) \\
    & = \gamma \Delta \\
    \Rightarrow G_{\pi'} - G_{\pi} & = \gamma G_{\pi'} \Delta G_{\pi}
\end{align*}
Thus, for any initial state distribution $\rho_0$, we have
\begin{align*}
D_{\text{TV}}(\rho^{\pi}(s|g) \parallel \rho^{\pi'}(s|g)) & = \frac{1}{2} \sum_{s} |\rho^{\pi}(s|g) - \rho^{\pi'}(s|g)| \\
& = \frac{1 - \gamma}{2} \parallel (G_{\pi} - G_{\pi'}) \rho_0 \parallel_1 \\
& = \frac{(1 - \gamma) \gamma}{2} \parallel G_{\pi'} \Delta G_{\pi} \rho_0 \parallel_1 \\
& \leq \frac{(1 - \gamma) \gamma}{2} \parallel G_{\pi'} \parallel_1 \parallel \Delta \parallel_1 \parallel G_{\pi} \parallel_1 \parallel \rho_0 \parallel_1
\end{align*}
Noticing that $P_{\pi}, P_{\pi'}$ are matrices whose columns have sum $1$. Thus, $\parallel G_{\pi} \parallel_1 = \parallel G_{\pi'} \parallel_1 = \frac{1}{1 - \gamma}$. Furthermore, 
\begin{align*}
    \parallel \Delta \parallel_1 & = \max_y \int_x \left| \int_a p(s'=x|s=y, a) (\pi'(a|s=y, g) - \pi(a|s=y, g)) d a \right| dx \\ 
    & = \max_y \int_a \int_x p(s' = x | s = y, a) \left| \pi'(a|s=y, g) - \pi(a|s=y, g) \right| dx da \\
    & = \max_y 2 D_{\text{TV}}(\pi(\cdot|s = y, g) \parallel \pi'(\cdot|s = y, g)) \\
    & \leq 2 \epsilon
\end{align*}
In all, we have
\begin{align*}
    D_{\text{TV}}(\rho^{\pi}(s|g) \parallel \rho^{\pi'}(s|g)) \leq \frac{\gamma \epsilon}{1 - \gamma} 
\end{align*}
\end{proof}



To state \Propref{prop:padg} formally, we define the $L$-lipschitz policy class $\Pi^{E}_{\Phi, L}$, whose $\Phi: \gX^E \to \gZ$ maps the input $x^e$s to latent vector $z$s. \footnote{$\gX^E = \cup_{e \in E} \gX^e$}
\begin{align*}
    \Pi_{\Phi, L}^{E} = \{w (\Phi(x^e), g), \forall w|\forall g \in \gG, z, z' \in \Phi(\gX^E), ~ D_{\text{TV}}(w (z, g) \parallel w (z', g)) \leq L \parallel z - z' \parallel_2 \}
\end{align*}
Namely, the nonlinear function $w$ is $L$-smooth over the latent space $\Phi(\gX^E)$ for each $g$. Furthermore, we extend the definition of perfect alignment encoder to the \emph{$(\eta, \psi)$-perfect alignment} as follows.

\begin{definition}[$(\eta, \psi)$-\textbf{Perfect Alignment}]
An encoder $\Phi$ is a $(\eta, \psi)$-perfect alignment encoder over the environments $E$, if it satisfies the following two properties:
\begin{enumerate}[leftmargin=0.2in, itemsep=-0.5pt, topsep=-0.5pt, partopsep=-5pt]
    \item $\forall s \in \gS, \forall e, e' \in E$, $\parallel \Phi(x^e(s)) - \Phi(x^{e'}(s)) \parallel_2 \leq \eta$.
    \item $\forall s, s' \in \gS, \forall e, e' \in E$, $\parallel \Phi(x^e(s)) - \Phi(x^{e'}(s')) \parallel_2 \geq \psi \parallel s - s' \parallel_2$
\end{enumerate}
\end{definition}
Essentially, $\eta$ quantifies the \emph{if} condition of perfect alignment (\Defref{def:pa}), i.e., how aligned the representation $\Phi(x^e(s)), e \in E$s are. Moreover, $\psi$ quantifies the \emph{only if} condition, i.e., how the representations of different states $s$ are separated. Based on the definition of $(\eta, \psi)$-perfect alignment, we prove the formal statement of \Propref{prop:padg} as follows.

\begin{customprop}{2}[Formal] \label{prop:padg_f}
$\forall \pi \in \Pi_{\Phi, L}^{\gE_{\text{train}}}$ and occupancy measure $\rho^{\gE_{\text{train}}} = \{\rho_{\pi}^{e_i}(x^{e_i}, g) \}_{i=1}^N$ for training environments and $\rho_{\pi_G}^t(x^t, g)$ for the target environment. For simplicity, we use $\epsilon^{e_i}$ and $\epsilon^t$ as the abbreviations. Considering $\pi^* = \argmin_{\pi' \in \Pi_{\Phi, L}^{\gE_{\text{train}}}} \sum_{i=1}^N \epsilon^{e_i}(\pi' \parallel \pi_G)$ and 
\begin{align*}
    \tilde{\rho}(x, g) = \argmin_{\rho \in B(\delta, \gE_{\text{train}}|\Pi_{\Phi, L}^{\gE_{\text{train}}})} d_{\Pi_{\Phi, L}^{\gE_{\text{train}}} \Delta \Pi_{\Phi, L}^{\gE_{\text{train}}}}(\rho(x, g), \rho^t_{\pi_G}(x^t, g))
\end{align*}
Then, if the encoder $\Phi$ is a $(\eta, \psi)$-perfect alignment over $\gE_{\text{train}}$ and $\pi_G$ is a $u$-smooth invariant policy with $u = L \psi$, i.e., $\forall s, s', \forall g, D_{\text{TV}}(\pi_G(\cdot|x^e(s), g) \parallel \pi_G(\cdot|x^e(s'), g)) \leq u \parallel s - s' \parallel_2$, we have
\begin{align*} 
    J^t(\pi_G) - J^t(\pi) & \leq \frac{2\gamma}{1 - \gamma} \left[ \underbrace{\frac{1}{N} \sum_{i=1}^N \epsilon^{e_i}(\pi \parallel \pi_G)}_{(E)} +  (3 + \frac{\gamma}{1 - \gamma}) \eta L \right] \\
    & ~~~ + \frac{2 \gamma}{1 - \gamma}\left[
    \underbrace{\epsilon^t(\pi^* \parallel \pi_G) + d_{\Pi_{\Phi, L}^{\gE_{\text{train}}}\Delta\Pi_{\Phi, L}^{\gE_{\text{train}}}}(\tilde{\rho}(x, g), \rho^t_{\pi_G}(x^t, g))}_{(t)} \right].
\end{align*}
\end{customprop}

\begin{proof}
It suffices to prove the following two statements with \Propref{prop:gbdg_f}.
\begin{enumerate}[leftmargin=0.2in, itemsep=-0.5pt, topsep=-2.5pt]
    \item[(1)] $\max_{e_i, e_i' \in \gE_{\text{train}}} d_{\Pi_{\Phi, L}^{\gE_{\text{train}}} \Delta \Pi_{\Phi, L}^{\gE_{\text{train}}}}(\rho^{e_i}_{\pi}(x^{e_i}, g), \rho^{e_i'}_{\pi}(x^{e_i'}, g)) \leq (2 + \frac{\gamma}{1 - \gamma}) \eta L$. 
    \item[(2)] $\frac{1}{N} \sum_{i=1}^N \epsilon^{e_i}(\pi^* \parallel \pi_G) \leq \eta L$.
\end{enumerate}

Proof of (1): By the definition of $\Pi_{\Phi, L}^{\gE_{\text{train}}}$, we have $\forall s, g \text{ and } e, e' \in \gE_{\text{train}}, D_{\text{TV}}(\pi(\cdot|x^e(s), g) \parallel \pi(\cdot|x^{e'}(s), g)) \leq \eta L$. Without loss of generality, we have $\forall e, e' \in \gE_{\text{train}}$,
\begin{align*}
    d_{\Pi_{\Phi, L}^{\gE_{\text{train}}} \Delta \Pi_{\Phi, L}^{\gE_{\text{train}}}}(\rho^{e}_{\pi}(x^{e}, g), \rho^{e'}_{\pi}(x^{e'}, g)) & = \sup_{\pi, \pi' \in \Pi_{\Phi, L}^{\gE_{\text{train}}}} |\E_{\rho^e_{\pi}(s, g)}[D_{\text{TV}}(\pi(\cdot|x^e(s), g) \parallel \pi'(\cdot|x^e(s), g))] \\ 
    & ~~~~ - \E_{\rho^{e'}_{\pi}(s, g)}[D_{\text{TV}}(\pi(\cdot|x^{e'}(s), g) \parallel \pi'(\cdot|x^{e'}(s), g))] | \\
    & \leq \sup_{\pi, \pi' \in \Pi_{\Phi, L}^{\gE_{\text{train}}}} \E_{\rho^e_{\pi}(s, g)}[D_{\text{TV}}(\pi(\cdot|x^e(s), g) \parallel \pi(\cdot|x^{e'}(s), g))] \\
    & ~~~~ + \sup_{\pi, \pi' \in \Pi_{\Phi, L}^{\gE_{\text{train}}}} \E_{\rho^e_{\pi}(s, g)}[D_{\text{TV}}(\pi'(\cdot|x^e(s), g) \parallel \pi'(\cdot|x^{e'}(s), g))] \\
    & ~~~~ + \sup_{\pi, \pi' \in \Pi_{\Phi, L}^{\gE_{\text{train}}}} |\E_{\rho^e_{\pi}(s, g)}[D_{\text{TV}}(\pi(\cdot|x^{e'}(s), g) \parallel \pi'(\cdot|x^{e'}(s), g))] \\
    & ~~~~ - \E_{\rho^{e'}_{\pi}(s, g)}[D_{\text{TV}}(\pi(\cdot|x^{e'}(s), g) \parallel \pi'(\cdot|x^{e'}(s), g))]| \\
    & \leq 2 \eta L + \sup_{A \in \sigma(\gS, \gG)} |\rho^e_{\pi}(A) - \rho^{e'}_{\pi}(A) | \\
    & = 2 \eta L + D_{\text{TV}}(\rho^e_{\pi}(s, g) \parallel \rho^{e'}_{\pi}(s, g)) 
\end{align*}
 Then, by \Lemref{lemma:occupancy}, we have $D_{\text{TV}}(\rho^e_{\pi}(s, g) \parallel \rho^{e'}_{\pi}(s, g)) \leq \frac{\gamma \eta L}{1 - \gamma}$. Consequently, we have $\forall e, e' \in \gE_{\text{train}}$,
\begin{align*}
    d_{\Pi_{\Phi, L}^{\gE_{\text{train}}} \Delta \Pi_{\Phi, L}^{\gE_{\text{train}}}}(\rho^{e}_{\pi}(x^{e}, g), \rho^{e'}_{\pi}(x^{e'}, g)) & \leq (2 + \frac{\gamma }{1 - \gamma}) \eta L
\end{align*}

Proof of (2): First, for each $z \in \Phi(\gX^{\gE_\text{train}})$, we assign one $s(z)$ such that $\exists e(z) \in \gE_{\text{train}}, \text{ s.t. } \Phi(x^{e(z)}(s(z))) = z$. Then, we choose the $\tilde{w}$ that $\tilde{w}(z, g) = \pi_G(\cdot|s(z), g), \forall z \in \Phi(\gX^{\gE_\text{train}})$. Clearly, $\tilde{w}$ is a mapping of $\Phi(\gX^{\gE_\text{train}}) \to \Pi_{\Phi, \infty}^{\gE_{\text{train}}}$. Besides, $\forall z_1, z_2 \in \Phi(\gX^{\gE_\text{train}}), g \in \gG$, we have
\begin{align*}
    \parallel \tilde{w}(z_1, g) - \tilde{w}(z_2, g) \parallel_2 & \leq u \parallel s(z_1) - s(z_2) \parallel_2 \\
    & \leq \frac{u}{\psi} \psi \parallel s(z_1) - s(z_2) \parallel_2 \\
    & \leq \frac{u}{\psi} \parallel \Phi(x^{e(z_1)}(s(z_1))) - \Phi(x^{e(z_2)}(s(z_2))) \parallel_2 \\
    & \leq L \parallel z_1 - z_2 \parallel_2
\end{align*}
Thus, the policy $\tilde{\pi} = \tilde{w}(\Phi(x^e), g) \in \Pi_{\Phi, L}^{\gE_{\text{train}}}$. Furthermore, by the definition of $(\eta, \psi)$-perfect alignment, we have
\begin{align*}
    \frac{1}{N} \sum_{i=1}^N \epsilon^{e_i}(\tilde{\pi} \parallel \pi_G) & = \frac{1}{N} \sum_{i=1}^N \E_{\rho^{e_i}_{\pi}(s, g)} [D_{\text{TV}}(\tilde{\pi}(\cdot|x^{e_i}(s), g) \parallel \pi_G(\cdot|s, g)))] \\
    & \leq \frac{1}{N} \sum_{i=1}^N  \E_{\rho^{e_i}_{\pi}(s, g)} [D_{\text{TV}}(\tilde{\pi}(\cdot|x^{e_i}(s), g) \parallel \tilde{w}(\Phi(x^{e(z)}(s)), g))] \\ 
    & \leq \eta L
\end{align*}
This completes the proof.
\end{proof}





\begin{remark}
If we choose $\psi = \frac{1}{\sqrt{L}}, \eta = \frac{1}{L^2}$ and $u = \sqrt{L}$, the informal version of \Propref{prop:padg_f} describes the case when $L \to \infty$ and omits unessential constants.
\end{remark}

\subsection{Discussions on \Eqref{equ:padg}} \label{apx:remain_discuss}

Here, we discuss the remaining terms $(E), (t)$ in the R.H.S of \Eqref{equ:padg}, i.e., upper bound of $J^t(\pi_G) - J^t(\pi)$. Together with the empirical analysis, we show how these terms are reduced by our perfect alignment criterion.

\textbf{Discussion on $(E)$.} Theoretically speaking, for a $(\eta, \psi)$-perfect alignment encoder, we have $\min (E) \leq \eta L$, as proved in \Apxref{apx:padg_proof}. Thus, when $\Phi$ is an ideal perfect alignment over $\gE_{\text{train}}$, i.e., $\eta=0, \psi > 0$ and $L \to \infty$, the optimal invariant policy $\pi_G \in \Pi_{\Phi, \infty}^{\gE_{\text{train}}}$. In \secref{sec:ablate}, empirical results demonstrate that $\eta$ is minimized to almost $0$ and the reconstruction (\Figref{fig:additional_recon}) demonstrates that the \emph{only if} condition is also satisfied.

Moreover, in GBMDPs with finite states, the perfect alignment encoder $\Phi$ over $\gE_{\text{train}}$ maps all training environments to the same goal-conditioned MDP (\Figref{fig:goalmdp}) with state space $\{z(s), s \sim \gS\}$. Noticing that the mapping is bijective and share the same actions and rewards with the original problem. Thus, a RL algorithm (e.g. Q-learning \cite{watkins1992qlearning}) on the equivalent MDP will converge to an optimal policy $\pi_G$ in the original MDP, i.e., invariant and maximize $J^e, e \in \gE_{\text{train}}$.

Empirically, in \Tableref{tab:comparative_evaluation}, we find that ours PA-SF achieves the near-optimal performance on all $\gE_{\text{train}}$ s, i.e., the same performance as a policy trained on a single environment. Therefore, we conclude that $(E)$ term is reduced to almost zero.

\textbf{Discussion on $(t)$.} In general, it is hard to conduct task-agnostic analysis on $(t)$ term, as discussed in \cite{sicilia2021dgtheory}. Moreover, owing to the intractable $\sup$ operators, it is almost impossible to measure the $(t)$ term directly in the experiments. Here, we analyze the generalization error term $(t)$ with an upper bound and we find evidence that this upper bound is reduced significantly under our perfect alignment criteria.

Here, we analyze the generalization performance when the policy converges to the optimal policy over $\gE_{\text{train}}$, which is the case empirically. Furthermore, we assume that $\forall s \in \gS, e \in \gE_{\text{train}}, p^e_{\pi}(s) \geq \epsilon$. Since it has been proven that the relabeled goal distribution of Skew-Fit will converge to $U(\gS)$ (uniform over the bounded state space) \cite{pong2020skew}, it is reasonable to assume that each state has non-zero occupancy measure under the well-trained policy $\pi(\cdot|s, g)$. We derive the following proposition with the same notation as in \Propref{prop:padg_f} except that $\Pi_{\Phi, L}^{\gE_{\text{train}}}$ is replaced by $\Pi_{\Phi, L}^{\gE}$. 

\begin{prop}
Suppose that $\Phi$ is a $(\eta_t, \psi_t)$-perfect alignment encoder over $\gE$ and $\pi_G, \pi \in \Pi^{\gE}_{\Phi, L}$. Besides, $\forall s \in \gS, e \in \gE_{\text{train}}, p^e_{\pi}(s) \geq \epsilon$. Then, $\forall t \in \gE/\gE_{\text{train}}$, we have
\begin{align*}
    (t) \leq \frac{2 \gamma}{1 - \gamma} \frac{N (E)}{\epsilon} + 4 \eta_t L + (E)
\end{align*}
Consequently, we have
\begin{align*}
J^{t}(\pi_G) - J^t(\pi) \leq \frac{4 \gamma}{1 - \gamma} (1 + \frac{\gamma N}{(1 - \gamma)\epsilon})(E) + \frac{(14 - 12 \gamma) \gamma}{(1 - \gamma)^2} \eta_t L
\end{align*}
\end{prop}

\begin{proof}
It is the straight forward to check that the statement and the proofs in \Propref{prop:padg_f} also hold under the policy class $\Pi_{\Phi, L}^{\gE} \subset \Pi_{\Phi, L}^{\gE_{\text{train}}}$ when $\Phi$ is $(\eta_t, \psi_t)$-perfect alignment over $\gE$. Namely, $\forall \pi, \pi_G \in \Pi_{\Phi, L}^{\gE}$, training environment set $\gE_{\text{train}}$ and the target environment $t \in \gE / \gE_{\text{train}}$, we have
\begin{align*} 
    J^t(\pi_G) - J^t(\pi) & \leq \frac{2\gamma}{1 - \gamma} \left[ \underbrace{\frac{1}{N} \sum_{i=1}^N \epsilon^{e_i}(\pi \parallel \pi_G)}_{(E)} +  (3 + \frac{\gamma}{1 - \gamma}) \eta_t L \right] \\
    & ~~~ + \frac{2 \gamma}{1 - \gamma}\left[
    \underbrace{\epsilon^t(\pi^* \parallel \pi_G) + d_{\Pi_{\Phi, L}^{\gE}\Delta\Pi_{\Phi, L}^{\gE}}(\tilde{\rho}(x, g), \rho^t_{\pi_G}(x^t, g))}_{(t)} \right].
\end{align*}

By definition, $\forall e \in \gE_{\text{train}}$, we have $\rho^e_{\pi} \in B(\delta, \gE_{\text{train}}|\Pi^{\gE}_{\Phi, L})$. As a consequence, we have
\begin{align*}
    (t) & = \frac{1}{N} \sum_{i=1}^N (\epsilon^t(\pi^* \parallel \pi_G) - \epsilon^{e_i}(\pi^* \parallel \pi_G)) + \frac{1}{N} \sum_{i=1}^N \epsilon^{e_i}(\pi^* \parallel \pi_G)   \\
    & ~~~~~ + d_{\Pi^{\gE}_{\Phi, L}\Delta \Pi^{\gE}_{\Phi, L}}(\tilde{\rho}(x, g), \rho^t_{\pi_G}(x^t, g)) \\
    & \leq \max_{e \in \gE_{\text{train}}} \epsilon^t(\pi^* \parallel \pi_G) - \epsilon^{e}(\pi^* \parallel \pi_G) + d_{\Pi^{\gE}_{\Phi, L} \Delta \Pi^{\gE}_{\Phi, L}}(\rho_{\pi}^e(x, g), \rho^t_{\pi_G}(x^t, g)) + \lambda\\
    & \leq 2\max_{e \in \gE_{\text{train}}} d_{\Pi^{\gE}_{\Phi, L} \Delta \Pi^{\gE}_{\Phi, L}}(\rho^e_{\pi}(x^e, g), \rho^t_{\pi_G}(x^t, g)) + \lambda \\
    & \leq \lambda + 2\max_{e \in \gE_{\text{train}}} d_{\Pi^{\gE}_{\Phi, L} \Delta \Pi^{\gE}_{\Phi, L}}(\rho^e_{\pi}(x^e, g), \rho^e_{\pi_G}(x^e, g)) \\
    & ~~~~ + 2\max_{e \in \gE_{\text{train}}} d_{\Pi^{\gE}_{\Phi, L} \Delta \Pi^{\gE}_{\Phi, L}}(\rho^e_{\pi_G}(x^e, g), \rho^t_{\pi_G}(x^t, g)) 
\end{align*}
Without loss of generality, we have
\begin{align*}
     d_{\Pi^{\gE}_{\Phi, L} \Delta \Pi^{\gE}_{\Phi, L}}(\rho^e_{\pi_G}(x^e, g), \rho^t_{\pi_G}(x^t, g)) & = \sup_{\pi, \pi' \in \Pi^{\gE}_{\Phi, L}} |\E_{\rho_{\pi_G}^e(s, g)}[D_{\text{TV}}(\pi(\cdot|x^e(s), g) \parallel \pi'(\cdot|x^e(s), g))] \\
     & ~~~~~ - \E_{\rho_{\pi_G}^t(s, g)}[D_{\text{TV}}(\pi(\cdot|x^t(s), g) \parallel \pi'(\cdot|x^t(s), g))]| \\
     & \leq \E_{\rho_{\pi_G}(s, g)} [\sup_{\pi, \pi' \in \Pi^{\gE}_{\Phi, L}} |D_{\text{TV}}(\pi(\cdot|x^e(s), g) \parallel \pi’(\cdot|x^e(s), g)) \\
     & ~~~~~ - D_{\text{TV}}(\pi(\cdot|x^t(s), g) \parallel \pi'(\cdot|x^t(s), g))|] \\
     & \leq \E_{\rho_{\pi_G}(s, g)} [\sup_{\pi, \pi' \in \Pi^{\gE}_{\Phi, L}} |D_{\text{TV}}(\pi(\cdot|x^e(s), g) \parallel \pi(\cdot|x^t(s), g)) \\
     & ~~~~~ + D_{\text{TV}}(\pi(\cdot|x^t(s), g) \parallel \pi'(\cdot|x^t(s), g)) \\
     & ~~~~~ + D_{\text{TV}}(\pi'(\cdot|x^t(s), g) \parallel \pi'(\cdot|x^e(s), g)) \\
     & ~~~~~ - D_{\text{TV}}(\pi(\cdot|x^t(s), g) \parallel \pi'(\cdot|x^t(s), g))|] \\
     & \leq \E_{\rho_{\pi_G}(s, g)} [\sup_{\pi, \pi' \in \Pi^{\gE}_{\Phi, L}} |D_{\text{TV}}(\pi(\cdot|x^e(s), g) \parallel \pi(\cdot|x^t(s), g)) \\
     & ~~~~~ + D_{\text{TV}}(\pi'(\cdot|x^t(s), g) \parallel \pi'(\cdot|x^e(s), g))|] 
\end{align*}

Clearly, we have $\forall e, e' \in \gE, D_{\text{TV}}(\pi(\cdot|x^e(s), g) \parallel \pi(\cdot|x^{e'}(s), g)) \leq \eta_t L$ and $\forall e, e' \in \gE_{\text{train}}, D_{\text{TV}}(\pi(\cdot|x^e(s), g) \parallel \pi_G(\cdot|x^{e}(s), g)) \leq \frac{N (E)}{\epsilon}$. Thus, by \Lemref{lemma:occupancy}, we have $\forall e \in \gE_{\text{train}}, D_{\text{TV}}(\rho^e_{\pi}(s, g) \parallel \rho^e_{\pi_G}(s, g)) \leq \frac{\gamma}{1 - \gamma} \frac{N (E)}{\epsilon}$. By the fact that $\lambda \leq (E)$, we have
\begin{align*}
    (t) & \leq \frac{2 \gamma}{1 - \gamma} \frac{N (E)}{\epsilon} + 4 \eta_t L + (E)
\end{align*}
Consequently, we have
\begin{align*}
    J^t(\pi_G) - J^t(\pi) \leq & \frac{4 \gamma}{1 - \gamma} (1 + \frac{\gamma N}{(1 - \gamma)\epsilon})(E) + \frac{(14 - 12 \gamma) \gamma}{(1 - \gamma)^2} \eta_t L
\end{align*}
\end{proof}
Intuitively speaking, when the RL policy converges to nearly optimal on $\gE_{\text{train}}$, the generalization regret over target environment can be bounded by the sum of two components: training environment performance regret and the level of invariant over the target environment. The former is relative small as the policy is near optimal. The latter can be reduced by learning perfect alignment encoder on the training environments. As shown in \Figref{fig:ablation}, the LER of different ablations over test environments suggests that $\eta_t$ is reduced empirically in PA-SF. Moreover, both losses $\gL_{\text{MMD}}$ and $\gL_{\text{DIFF}}$ contribute to the reduction. Furthermore, we notice that PA-SF $<$ PA-SF (w/o D) $<$ PA-SF (w/o MD) $<$ Skew-Fit in both LER and generalization performance, which coincides with our theoretical analysis.

\subsection{Discussions on $L^{\text{MMD}}$} \label{apx:mmd}

Noticing that the $(\eta, \psi)$-perfect alignment requires the $l_2$ distance of $z^e(s) - z^{e'}(s)$ is bounded for any state $s$. In practice, we adopt the convention \cite{schulman2015trpo} to minimize the average $l_2$ distance over the state space as a surrogate objective, i.e., $\E_{\rho(s)}[\parallel \Phi(x^e(s)) - \Phi(x^{e'}(s)) \parallel_2^2]$. It prevents the unstable and intractable training of robust optimization in our setting and we find it works well empirically. Here, we introduce a theoretical property of $L^{\text{MMD}}$, which justifies its validity to ensure the \emph{if} condition of perfect alignment. To begin with, we prove the following Lemma.
\begin{lemma} \label{lemma:mono}
For some distribution $P(x)$ and two functions $f, g: \gX \to \mathbb{R}^d$, we have
\begin{align*}
    \E_{x, x' \sim P(x)} [\parallel f(x) - g(x') \parallel_2^2] \geq \frac{1}{2} \E_{x \sim P(x)} [\parallel f(x) - g(x) \parallel_2^2]
\end{align*}
\end{lemma}

\begin{proof}
Without loss of generality, we assume $\E_{x \sim P(x)} [g(x)] = 0$. Then, we have
\begin{align*}
    \E_{x, x' \sim P(x)} [\parallel f(x) - g(x') \parallel_2^2] & = \E_{x \sim P(x)} [\parallel f(x) \parallel_2^2]  + \E_{x' \sim P(x')} [\parallel g(x') \parallel_2^2] \\
    & ~~~~ - 2 \E_{x, x' \sim P(x)} [\langle f(x), g(x') \rangle]  \\
    & =  \E_{x \sim P(x)} [\parallel f(x) \parallel_2^2] + \E_{x' \sim P(x')} [\parallel g(x') \parallel_2^2] \\
    & ~~~~ - 2 \langle \E_{x \sim P(x)} [f(x)], \E_{x' \sim P(x')}[g(x')] \rangle \\
    & =  \E_{x \sim P(x)} [\parallel f(x) \parallel_2^2] + \E_{x' \sim P(x')} [\parallel g(x') \parallel_2^2] \\
    & \geq \frac{1}{2} \E_{x \sim P(x)} [\parallel f(x) - g(x) \parallel_2^2]
\end{align*}
\end{proof}

\begin{prop} \label{prop:paloss}
Under mild assumptions, if $\E_{e, e' \sim \gE_{\text{train}}, \gB_{\text{align}} \sim \gR_{\text{align}}} [\parallel \Phi(x^e(s)) - \Phi(x^{e'}(s)) \parallel_2^2 ] \geq \delta$, then $L^{\text{MMD}}(\Phi) \geq O(\delta)$.
\end{prop}

\begin{proof}
We assume that $\forall z, z', \parallel \psi(z) - \psi(z') \parallel_2 \geq u \parallel z - z' \parallel_2$ for some $u > 0$. For brevity, we denote $z^e(s) = \Phi(x^e(s))$ and $s^e_b = s^e_t(a_{0:t}^b)$ (the $b^{\text{th}}$ sample: state $s$ sampled in environment $e$ after executing some action $a_{0:t}^b$). 
\begin{align*}
    L^{\text{MMD}}(\Phi) & = \E_{e, e' \sim \gE_{\text{train}}, \gB_{\text{align}} \sim \gR_{\text{align}}} [\parallel \frac{1}{B} \sum_{b=1}^B \psi(z^e(s_b^e)) - \frac{1}{B} \sum_{b=1}^B \psi(z^{e'}(s_b^{e'})) \parallel_2^2 ] \\
    & \geq \frac{1}{B^2}\E_{e, e' \sim \gE_{\text{train}}, \{s_b^e, s_b^{e'}\}_{b=1}^B \sim \gR_{\text{align}}} [\sum_{b=1}^B \parallel \psi(z^e(s^e_b)) -  \psi(z^{e'}(s^{e'}_b)) \parallel_2^2 ] + \frac{1}{B^2} \\
    & ~~\E_{e, e' \sim \gE_{\text{train}}, \{s^{e}_b, s^{e'}_b\}_{b=1}^B \sim \gR_{\text{align}}} [\sum_{i\neq j} \langle \psi(z^e(s_i^e)) - \psi(z^{e'}(s_i^{e'})), \psi(z^{e}(s_j^{e})) - \psi(z^{e'}(s_j^{e'})) \rangle] 
\end{align*}
Since $s_b^e$s are sampled from $\gR_{\text{align}}$ independently, we have
\begin{align*}
    L^{\text{MMD}}(\Phi) & \geq \frac{1}{B} \E_{e, e' \sim \gE_{\text{train}}, s^e_b, s^{e'}_b \sim \gR_{\text{align}}}[\parallel \psi(z^e(s^e_b)) - \psi(z^{e'}(s^{e'}_b)) \parallel^2_2] \\
    & ~~~ + \frac{B - 1}{B} \parallel \E_{e, e' \sim \gE_{\text{train}}, s^e_b, s^{e'}_b \sim \gR_{\text{align}}} [\psi(z^e(s^e_b)) - \psi(z^{e'}(s^{e'}_b))] \parallel_2^2 \\
    & \geq \frac{1}{B} \E_{e, e' \sim \gE_{\text{train}}, s^e_b, s^{e'}_b \sim \gR_{\text{align}}}[\parallel \psi(z^e(s^e_b)) - \psi(z^{e'}(s^{e'}_b)) \parallel^2_2] \\
    & \geq \frac{u}{B} \E_{e, e' \sim \gE_{\text{train}}, s^e_b, s^{e'}_b \sim \gR_{\text{align}}}[\parallel z^e(s^e_b) - z^{e'}(s^{e'}_b) \parallel^2_2] \\
    & = \frac{u}{B} \E_{e, e' \sim \gE_{\text{train}}, a_{0:t}^b \sim \gR_{\text{align}}, s_b^e \sim \rho(s^e_t(a_{0:t}^b)), s_b^{e'} \sim \rho(s^{e'}_t(a_{0:t}^b))}[\parallel z^e(s^e_b) - z^{e'}(s^{e'}_b) \parallel^2_2] \\
    & \geq \frac{u}{2B} \E_{e, e' \sim \gE_{\text{train}}, a_{0:t}^b \sim \gR_{\text{align}}, s_b \sim \rho(s_t(a_{0:t}^b))}[\parallel z^e(s_b) - z^{e'}(s_b) \parallel^2_2] ~~~~~~~~~~~ (\text{\Lemref{lemma:mono}}) \\
    & = \frac{u}{2B} \E_{e, e' \sim \gE_{\text{train}}, s_b \sim \gR_{\text{align}}}[\parallel z^e(s_b) - z^{e'}(s_b) \parallel^2_2]
\end{align*}
This completes the proof.
\end{proof}

\section{Additional Results} \label{apx:additional}
Here, we show the ablation study of PA-SF on other tasks: \textit{Reach}, \textit{Push}, and \textit{Pickup}. Figure \ref{fig:additional_ablation} demonstrates that similar results in \secref{sec:ablate} also holds on \textit{Reach} and \textit{Push} while on \textit{Pickup}, PA-SF (w/o D) outperforms PA-SF on test environments marginally. We also find that the LER is relatively large in \textit{Pickup}, perhaps owing to the relative large stochasity in this environment. However, the training and test performance are still satisfactory.

\begin{figure}[h]
    \centering
    \subfigure[Reach]{ 
    \label{fig:ablation_door}
        \includegraphics[width=\linewidth]{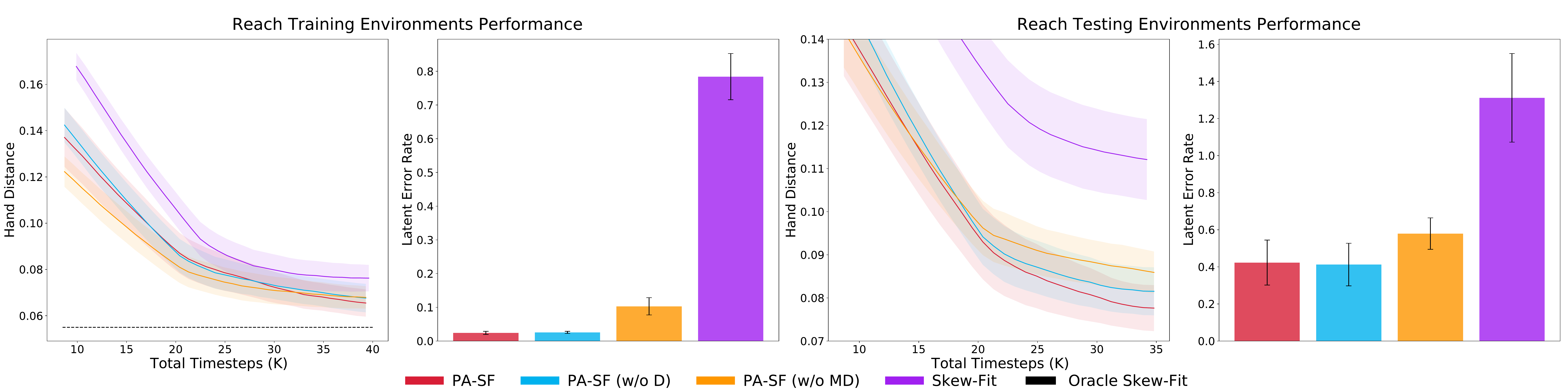}
    }
    \subfigure[Push]{ 
    \label{fig:ablation_push}
        \includegraphics[width=\linewidth]{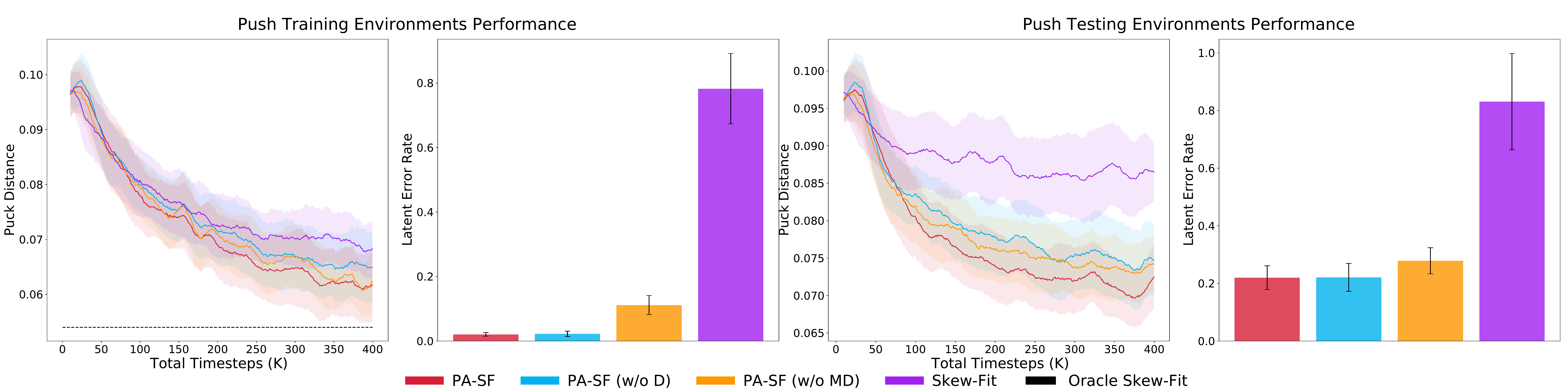}
    }
    \subfigure[Pickup]{ 
    \label{fig:ablation_pickup}
        \includegraphics[width=\linewidth]{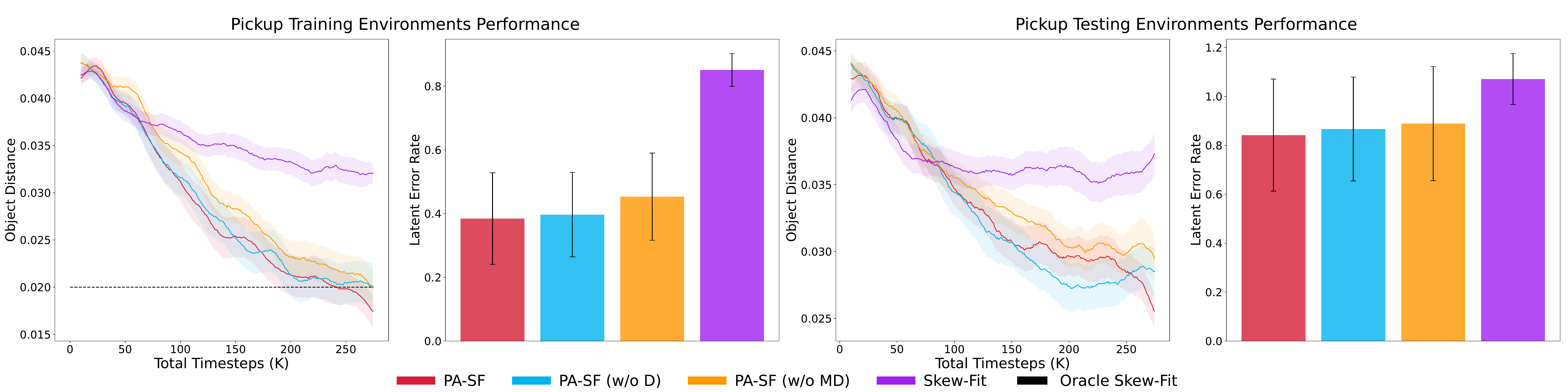}
    }
    \caption{Ablations of our algorithm PA-SF and visualization of the latent representation via LER metric on \textit{Reach}, \textit{Push}, and \textit{Pickup}. All curves represent the mean and one standard deviation (except Pickup with half standard deviation) across 7 seeds.}
    \label{fig:additional_ablation}
    \vskip -0.2in
\end{figure}

Additionally, we compare t-SNE plots of PA-SF with that of vanilla Skew-Fit in Figure \ref{fig:additional_tsne}. Clearly, the t-SNE plot generated from PA-SF are more aligned than that in Skew-Fit. Noticing that Skew-Fit also encodes the irrelevant environmental factors into the latent embedding.

Finally, we visualize how well the VAE trained with $L^{\text{PA}}$ satisfies the perfect alignment condition. In \Figref{fig:additional_recon}, we visualize the reconstruction (middle line) of the original input (bottom line) as well as the shuffled reconstruction (top line), i.e., reconstruction with  the same latent representation $z$ but a different environment index $e$ (\Figref{fig:vae_struct}). Clearly, in all tasks, the VAE successfully acquires the perfectly aligned latent space w.r.t the training environments as the shuffled reconstruction shares the almost same ground truth state $s$ with the reconstruction and the original input. Noticing that the original input images are sampled uniformly from the state space $\gS$.

\begin{figure}[h]
    \centering
    \subfigure[Reach]{ 
    \label{fig:tsne_reach}
        \includegraphics[width=3.85in]{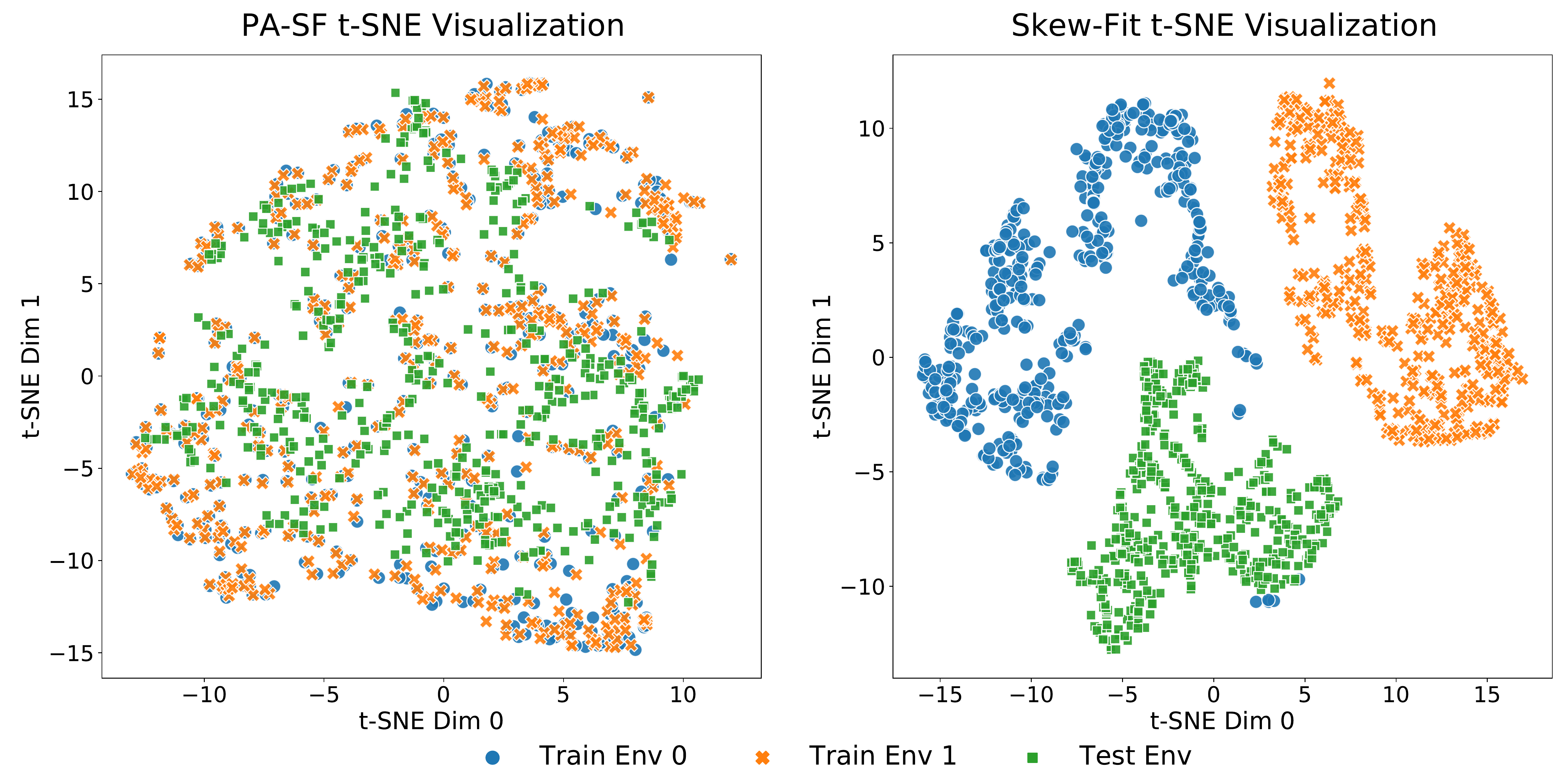}
    }
    \subfigure[Door]{ 
    \label{fig:tsne_door}
        \includegraphics[width=3.85in]{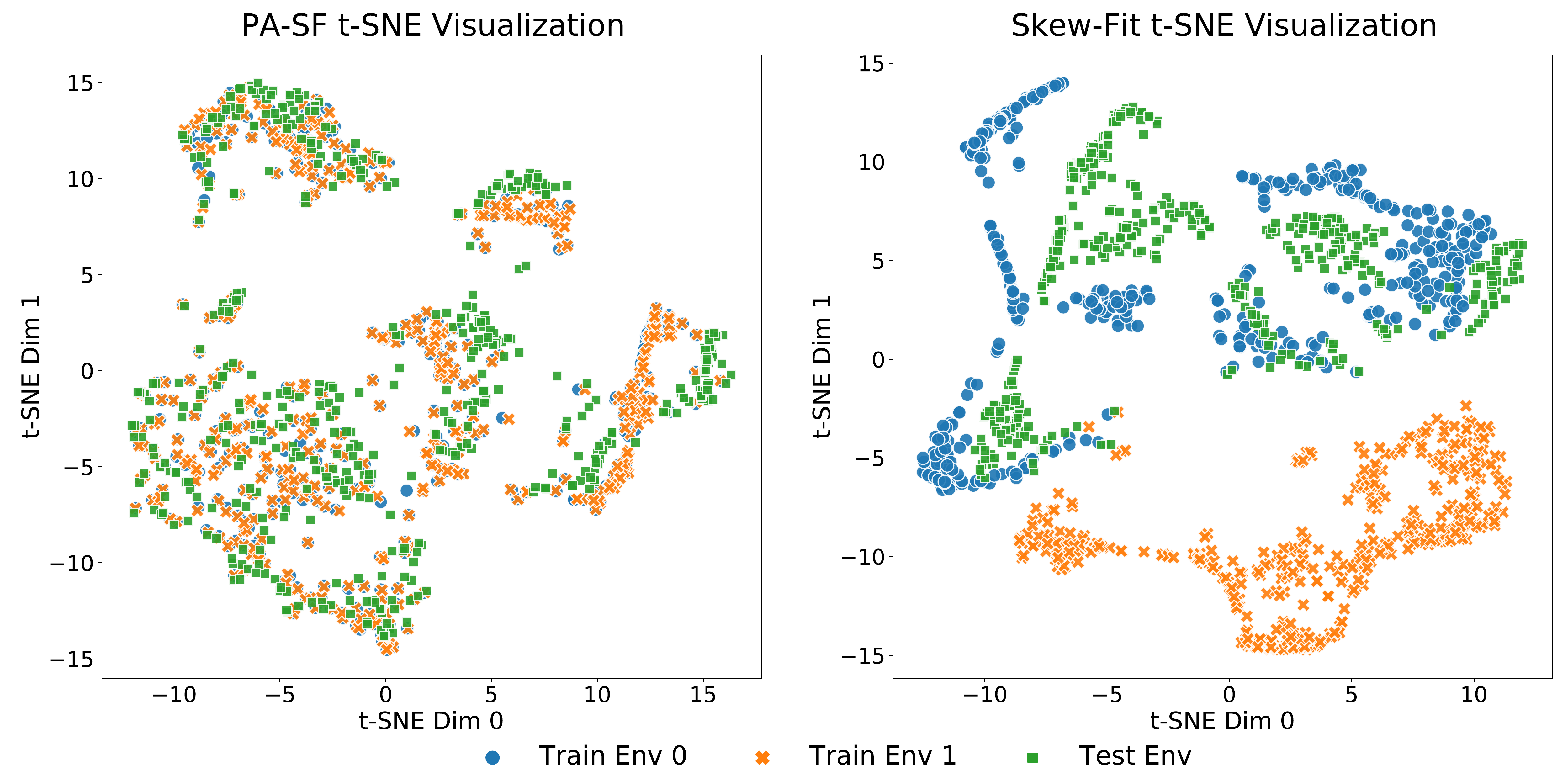}
    }
    \subfigure[Push]{ 
    \label{fig:tsne_push}
        \includegraphics[width=3.85in]{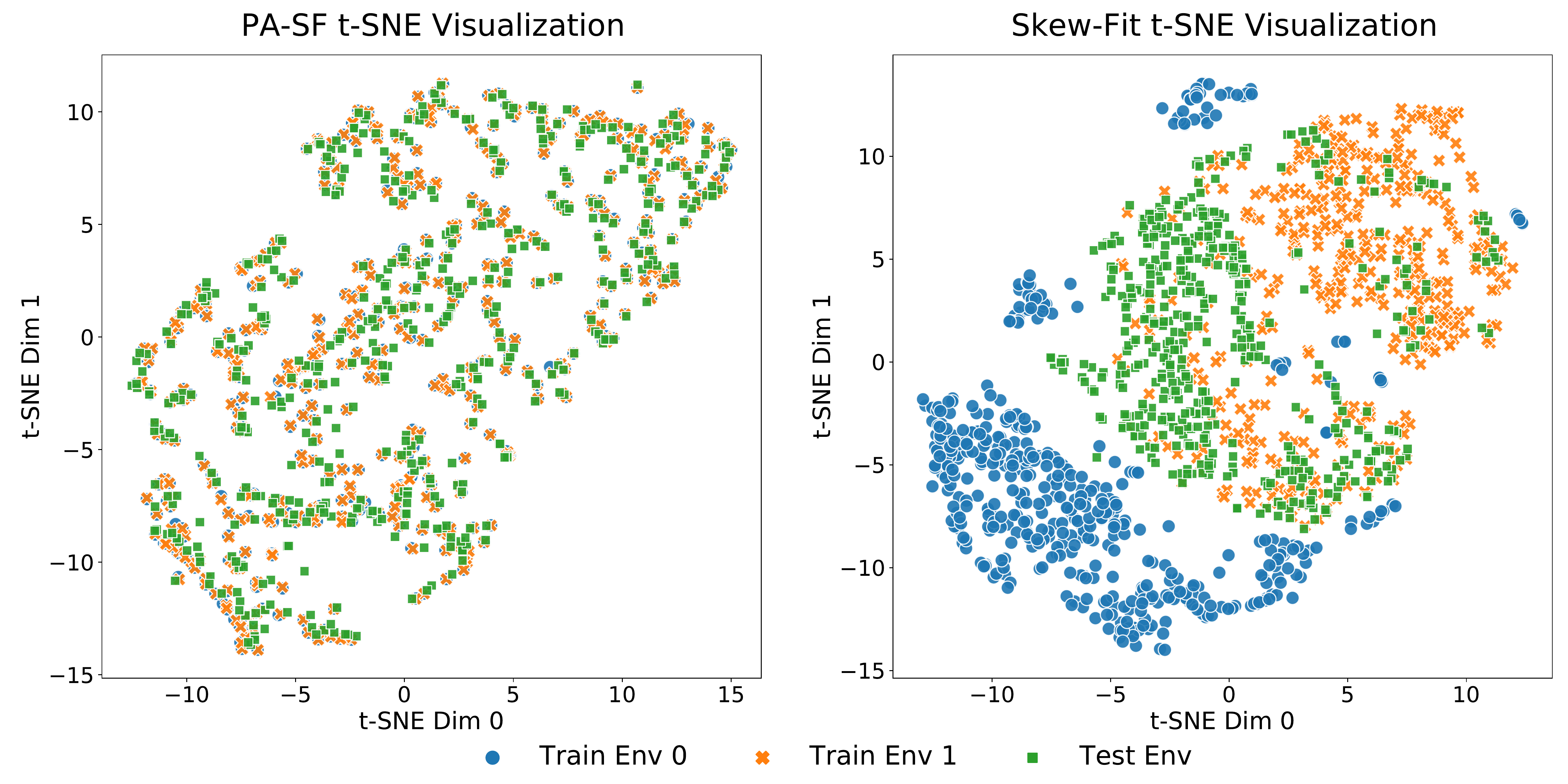}
    }
    \subfigure[Pickup]{
    \label{fig:tsne_pickup}
        \includegraphics[width=3.85in]{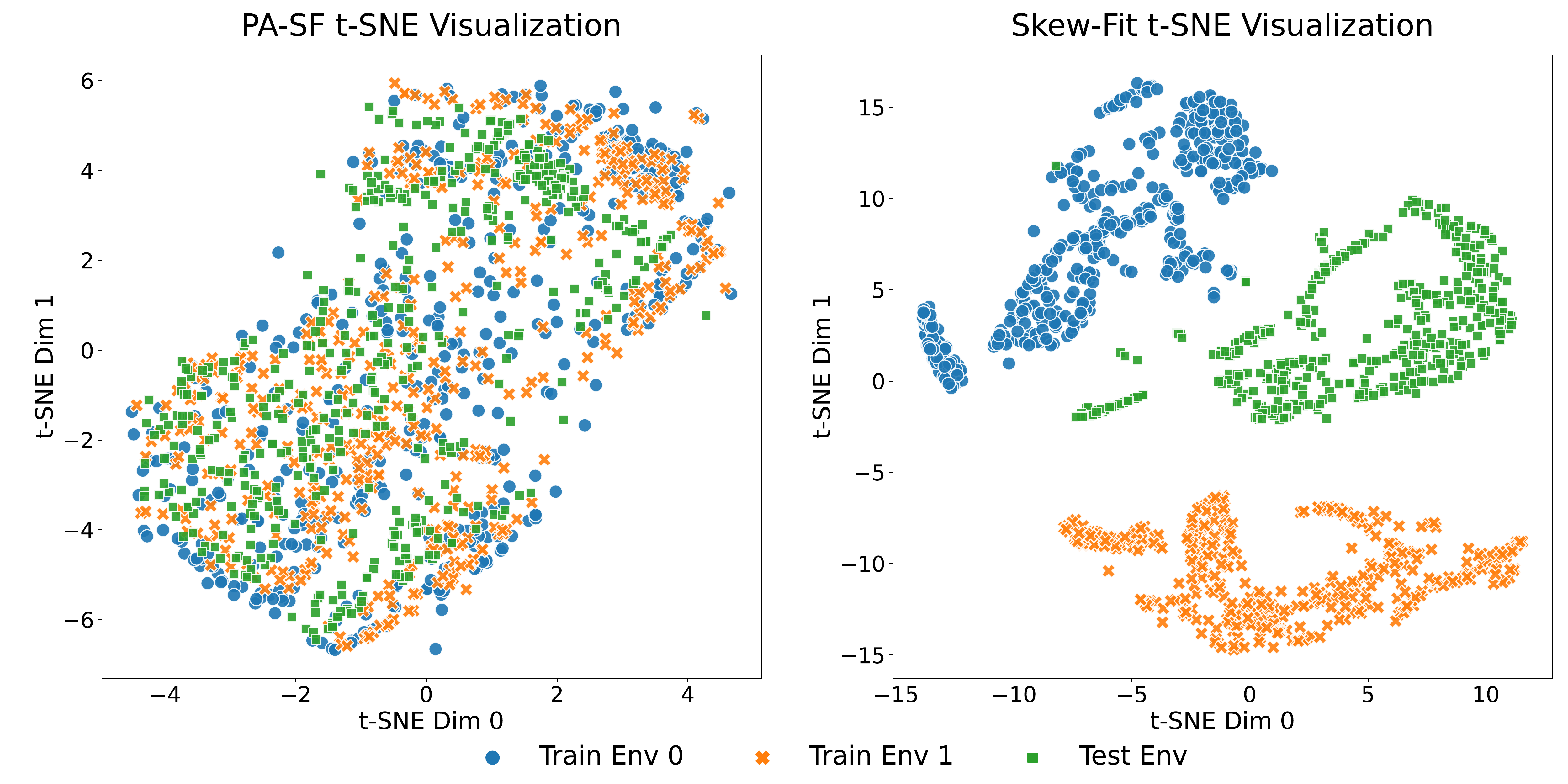}
    }
    \caption{t-SNE visualization of the latent space $\Phi(x^e)$ trained with PA-SF and Skew-Fit for three environments, i.e., 2 training and 1 testing on different tasks.}
    \label{fig:additional_tsne}
    \vskip -0.2in
\end{figure}

\begin{figure}[h]
    \centering
    \subfigure[Reach]{ 
    \label{fig:recon_reach}
        \includegraphics[width=5.5in]{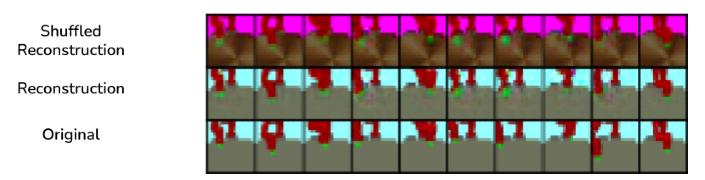}
    }
    \subfigure[Door]{ 
    \label{fig:recon_door}
        \includegraphics[width=5.5in]{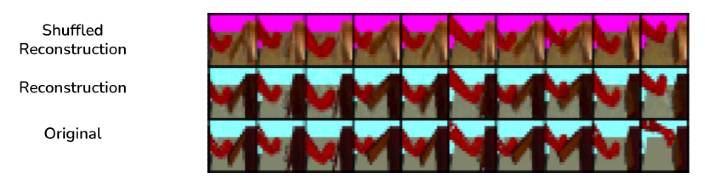}
    }
    \subfigure[Push]{ 
    \label{fig:recon_push}
        \includegraphics[width=5.5in]{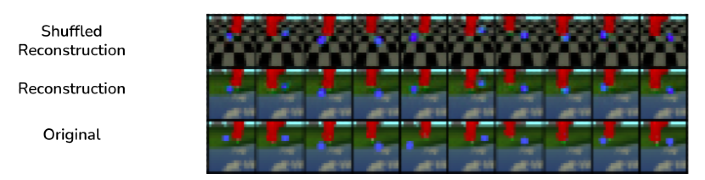}
    }
    \subfigure[Pickup]{
    \label{fig:recon_pickup}
        \includegraphics[width=5.5in]{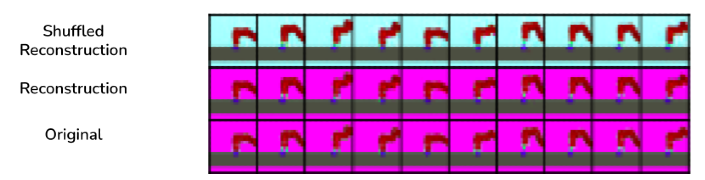}
    }
    \caption{Visualization of the VAE on all four tasks. For each figure, the bottom line shows the original input images (sampled uniformly from the state space). The middle line is the reconstruction of the input image. The top line show the shuffled reconstruction image, i.e., reconstruction with the same latent space $z$ but a shuffled environment index $e$.}
    \label{fig:additional_recon}
    \vskip -0.2in
\end{figure}

\clearpage
\section{Experiment Details} \label{apx:imp}

\subsection{Task Setups}

\begin{figure}[h]
  \centering
  \subfigure[Reach]{
    \label{fig:reach_setup}
    \includegraphics[width=5.32in]{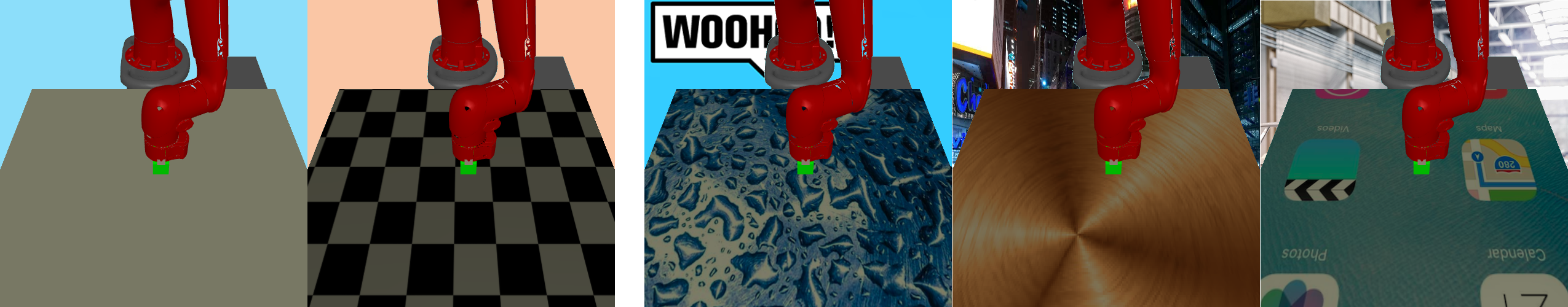}
  }
  \subfigure[Door]{
    \label{fig:door_setup}
    \includegraphics[width=5.32in]{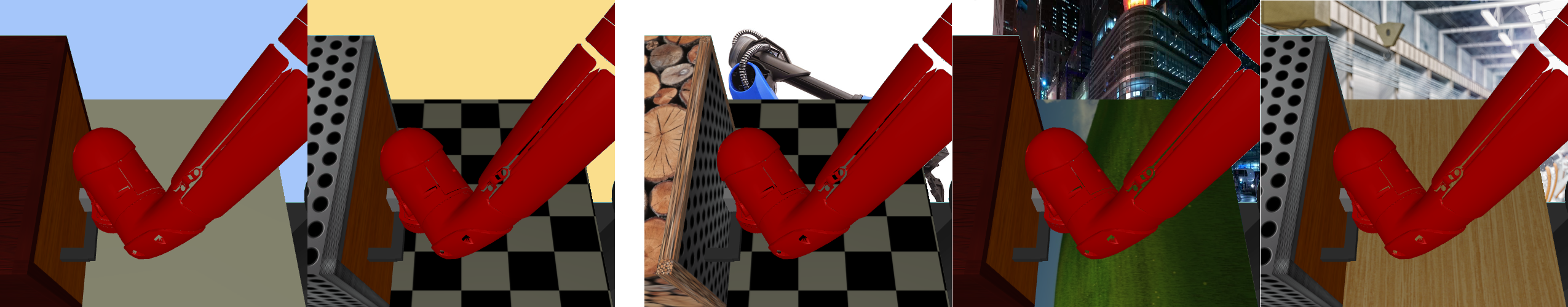}
  }
  \subfigure[Push]{
    \label{fig:push_setup}
    \includegraphics[width=5.32in]{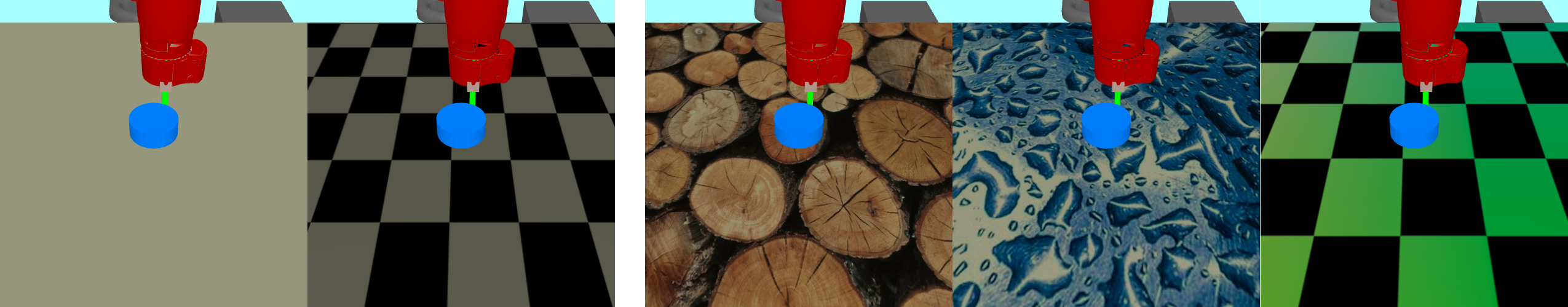}
  }
  \subfigure[Pickup]{
    \label{fig:pickup_setup}
    \includegraphics[width=5.32in]{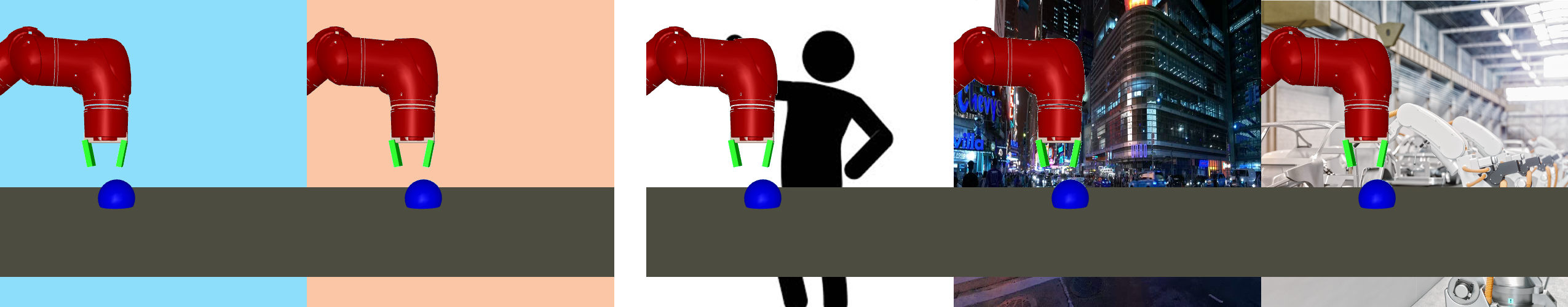}
  }
  \caption{Task setups for training (left) and test (right) environments}
  \label{fig:task_setups}
\end{figure}

Our base environments are first used in \cite{nair2018rig}. \Figref{fig:task_setups} illustrates some of the environments and we provide brief descriptions as follows.

\textit{Reach}: A 7-DoF sawyer arm task in which the goal is to reach a desired target position. We construct multiple training and test environments by altering the backgrounds with various images and dynamic videos and the foregrounds with diverse textures.

\textit{Door}: A 7-DoF sawyer arm task with a box on the table. The goal is to open the door to a target angle. We construct multiple environments in the same way as \textit{Reach} but with different ingredients. Additionally, we use the task with reset at the end of each episode.

\textit{Push}: A  7-DoF sawyer arm task and a small puck on the table. The goal is to push the puck to a target position. We only change table textures as the camera has almost no background as input. 

\textit{Pickup}: The task setting is the same as \textit{Push}. The goal is to pick up the object and place it in the desired position. We construct different environments with different backgrounds. 

\subsection{Implementation Details of PA-SF}

In our experiments, we use the same VAE architecture as Skew-Fit except for the environment index as an extra input for the decoder. Most of the hyper-parameters in VAE including the training schedule is the same as that in Skew-Fit except for the components we added in our algorithm. In $L^{\text{MMD}}$, we use the same random expansion function as in \cite{louizos2016fairvae} as it works well in practice. Namely,

\begin{align*}
    \psi( z ) = \sqrt{\frac{2}{D_{\psi}}} \cos \left( \sqrt{\frac{2}{\gamma_{\psi}}} Wz + b \right)
\end{align*}

where $z \in \R^d$ denotes latent embedding of observation, and $W \in \R^{D_{\psi} \times d}$ and $b \in \R^{D_{\psi}}$ are random weight and bias. $L_{\text{MMD}}$ and $L_{\text{DIFF}}$ are computed with respect to samples from the replay buffer and the aligned buffer. Table \ref{tab:pasf_gen_hyperparams} lists the hyper-parameters that are shared across four tasks. Table \ref{tab:pasf_spec_hyperparams} lists hyper-parameters specified to each task. Notice that we fine-tune the hyper-parameters on some validation environments and test on other environments.

\begin{table}[h]
  \caption{Shared hyper-parameters for PA-SF}
  \label{tab:pasf_gen_hyperparams}
  \centering
  {\setlength{\tabcolsep}{0.4em} 
  \begin{tabular}{ll}
    \toprule
    Hyper-parameter & Value \\
    \midrule
    Aligned Path Length & $50$ \\
    VAE Relay Buffer Batch Size & $32$ \\
    VAE Aligned Buffer Batch Size & $32$ \\
    Random Expansion Function Dimension $D_{\psi}$ & $1024$ \\
    Random Expansion Function Scalar $\gamma_{\psi}$ & $1.0$ \\
    Number of Training per Train Loop & $1500$ \\
    Number of Total Exploration Steps per Epoch & $900$ \\
    \bottomrule
  \end{tabular}
  }
\end{table}

\begin{table}[h]
  \caption{Task specific hyper-parameters for PA-SF}
  \label{tab:pasf_spec_hyperparams}
  \centering
  {\setlength{\tabcolsep}{0.4em} 
  \begin{tabular}{lllll}
    \toprule
    Hyper-parameter & Reach & Door & Push & Pickup \\
    \midrule
    MMD Coefficient $\alpha_{\text{MMD}}$ & 1000 & 1000 & 200 & 100 \\
    Difference Coefficient $\alpha_{\text{DIFF}}$ & 0.1 & 1.0 & 0.04 & 0.04 \\
    $\beta$ for $\beta$-VAE & 20 & 20 & 20 & 10 \\
    Skew Coefficient $\alpha$ & -0.1 & -0.5 & -1 & -1 \\
    Proportion of Aligned Sampling in Exploration & $\frac{1}{6}$ & $\frac{1}{3}$ & $\frac{1}{6}$ & $\frac{1}{6}$ \\
    \bottomrule
  \end{tabular}
  }
\end{table}

\subsection{Implementation of Baselines}
\label{subsec:impl_baselines}

\textbf{Skew-Fit} \cite{pong2020skew}: Skew-Fit is designed to learn a goal-conditioned policy by self-learning in Goal-conditioned MDPs. We extend it to Goal-conditioned Block MDPs by training the $\beta$-VAE with observations sampled from replay buffers of each training environments and constructing skewed distribution respectively. We modified some of the hyper-parameters of Skew-Fit in Table \ref{tab:skewfit_gen_hyperparams}, did a grid search over latent dimension size, $\beta$, VAE training schedule, and number of training per train loop (Table \ref{tab:skewfit_spec_hyperparams}). 

\begin{table}[h]
  \caption{Modified general hyper-parameters for Skew-Fit}
  \label{tab:skewfit_gen_hyperparams}
  \centering
  {\setlength{\tabcolsep}{0.4em} 
  \begin{tabular}{ll}
    \toprule
    Hyper-parameter & Value \\
    \midrule
    Exploration Noise & None \\
    RL Batch Size & 1200 \\
    VAE Batch Size & 96 \\
    Replay Buffer Size for each $e \in \mathcal{E}_{\text{train}}$ & 50000 \\
    \bottomrule
  \end{tabular}
  }
\end{table}

\begin{table}[h]
  \caption{Task specific hyper-parameters for Skew-Fit}
  \label{tab:skewfit_spec_hyperparams}
  \centering
  {\setlength{\tabcolsep}{0.4em} 
  \begin{tabular}{lllll}
    \toprule
    Hyper-parameter & Reach & Door & Push & Pickup \\
    \midrule
    Path Length & $50$ & $100$ & $50$ & $50$ \\
    $\beta$ for $\beta$-VAE & $20$ & $20$ & $40$ & $15$ \\
    Latent Dimension Size & $8$ & $25$ & $15$ & $20$ \\
    $\alpha$ for Skew-Fit & $0.1$ & $0.5$ & $1.0$ & $1.0$ \\
    VAE Training Schedule & $2$ & $1$ & $2$ & $1$ \\
    Sample Goals From & $q_{\phi}^{G}$ & $p_{\text{skewed}}$ & $p_{\text{skewed}}$ & $p_{\text{skewed}}$ \\
    Number of Training per Train Loop & $1200$ & $2000$ & $2000$ & $2000$ \\
    \bottomrule
  \end{tabular}
  }
\end{table}

\textbf{Skew-Fit + RAD} \cite{laskin2020rad}: Previous work \cite{kostrikov2020image} and \cite{laskin2020rad} have found that data augmentation is a simple yet powerful technique to enhance performance for visual input agents. We compare to the most well known method, i.e., Reinforcement Learning with Augmented Data (RAD) and re-implement \footnote{\url{https://github.com/MishaLaskin/rad}} it upon Skew-Fit for Goal-conditioned Block MDPs. Note that RAD is originally designed to augment states for agents directly. In our setting, the augmentation is added to $\beta$-VAE training phase, which increases the robustness of latent space against irrelevant noise. Specifically, we augmented observations for training $\beta$-VAE as well as constructing skewed data distributions. At the beginning of each episode, we sample goals from the skewed distribution and encode augmented goal as latent goal. The training of SAC algorithm depends on the latent code of augmented current and next state. We also incorporate data augmentation with the skewed distribution and hindsight relabeling steps. To investigate the performance of different augmentation methods, we chose to experiment with \textit{crop}, \textit{cutout-color}, \textit{color-jitter}, and \textit{grayscale} and found that \textit{crop} worked best among the four augmentations as reported in the RAD paper. Other augmentation methods such as \textit{cutout-color} and \textit{color-jitter} either replace a patch of images with single color, which may include the end-effector of Sawyer arm, or alter the color of the whole image, which may include target object (i.e., puck in \textit{Push}) and thus hurt performance. We use the same hyper-parameters as in Skew-Fit.

\textbf{MISA} \cite{zhang2020inblock} and \textbf{DBC} \cite{zhang2020inrepr}: Bisimulation metrics have been used to learn minimal yet sufficient representations in Block MDPs. We compare with two SOTA methods: Model-irrelevance State Abstraction (MISA) and Deep Bisimulation for Control (DBC), and modify the code for Goal-conditioned Block MDPs. In particular, we add goals' inputs into the reward predictor and use oracle ground truth distance between the current state and goal state (i.e., end-effector's position and object's position) as rewards. Our code is built upon the publicly available codes  \footnote{\url{https://github.com/facebookresearch/icp-block-mdp}} \footnote{\url{https://github.com/facebookresearch/deep_bisim4control}}.For fair comparison, we also fine-tuned some hyper-parameters on each task respectively. For MISA, we did a grid search over the encoder and decoder learning rates $\in \{ 10^{-3}, 10^{-5} \}$ and reward predictor coefficient $\in \{ 0.5, 1.0, 2.0 \}$. For DBC, we did a grid search over the encoder and decoder learning rates $\in \{ 10^{-3}, 10^{-4} \}$ and bisimulation coefficients $\in \{0.25, 0.5, 1\}$. Please refer to Table \ref{tab:misa_hyperparams} and Table \ref{tab:dbc_hyperparams} for our final choices.

\begin{table}[h]
  \caption{Task specific hyper-parameters for MISA}
  \label{tab:misa_hyperparams}
  \centering
  {\setlength{\tabcolsep}{0.4em} 
  \begin{tabular}{lllll}
    \toprule
    Hyper-parameter & Reach & Door & Push & Pickup \\
    \midrule
    Encoder and Decoder Learning Rate & $10^{-5}$ & $10^{-3}$ & $10^{-3}$ & $10^{-5}$ \\
    Reward Predictor Coefficient & $0.5$ & $1.0$ & $2.0$ & $0.5$ \\
    \bottomrule
  \end{tabular}
  }
\end{table}

\begin{table}[H]
  \caption{Task specific hyper-parameters for DBC}
  \label{tab:dbc_hyperparams}
  \centering
  {\setlength{\tabcolsep}{0.4em} 
  \begin{tabular}{lllll}
    \toprule
    Hyper-parameter & Reach & Door & Push & Pickup \\
    \midrule
    Encoder and Decoder Learning Rate & $10^{-4}$ & $10^{-4}$ & $10^{-4}$ & $10^{-3}$ \\
    Bisimulation Coefficient & $0.5$ & $1.0$ & $0.25$ & $1.0$ \\
    \bottomrule
  \end{tabular}
  }
\end{table}


\end{document}